\pdfoutput=1
\pdfpagewidth=\paperwidth
\pdfpageheight=\paperheight
\documentclass[11pt, letterpaper]{article}

\DeclareFixedFont{\MyTitleFont}{OT1}{\familydefault}{m}{n}{18pt}
\DeclareFixedFont{\MyAuthorFont}{OT1}{\familydefault}{m}{n}{12pt}
\DeclareFixedFont{\MyAbstractTitleFont}{OT1}{\familydefault}{m}{n}{12pt}
\DeclareFixedFont{\MyAbstractFont}{OT1}{\familydefault}{m}{it}{11pt}
\DeclareFixedFont{\MySubtitleFont}{OT1}{\familydefault}{m}{n}{13pt}
\DeclareFixedFont{\MySubSubtitleFont}{OT1}{\familydefault}{m}{n}{12pt}
\DeclareFixedFont{\MySubSubSubtitleFont}{OT1}{\familydefault}{m}{it}{11pt}
\DeclareFixedFont{\MyTextFont}{OT1}{\familydefault}{m}{n}{11pt}

\usepackage{geometry}
\usepackage{titling}

\title{\MyTitleFont Distributed Algorithms for Linearly-Solvable Optimal Control in Networked Multi-Agent Systems \vspace{0em}}
\author{\MyAuthorFont Neng Wan$^{1}$, Aditya Gahlawat$^{1}$, Naira Hovakimyan$^{1}$, \\
	Evangelos A. Theodorou$^{2}$, and Petros G. Voulgaris$^{3}$}
\date{}

\usepackage{titlesec}

\usepackage[marginal]{footmisc}
\usepackage{abstract}

\usepackage[sort]{cite}

\usepackage{indentfirst}
\usepackage{amsthm}
\usepackage{amsfonts}
\usepackage{amsmath}
\usepackage{booktabs}
\usepackage{bm}
\usepackage{subfigure}
\usepackage{enumerate}
\usepackage{float}
\usepackage[margin = 2em]{caption}
\usepackage{algorithm}
\usepackage{algcompatible}

\usepackage[pdftex]{color,graphicx}
\usepackage[dvipsnames]{xcolor}

\usepackage[pdftex, bookmarks]{hyperref}
\hypersetup{colorlinks = true, citecolor = red, linkcolor = blue, urlcolor = black}

\theoremstyle{remark}

\newtheorem{remark}{Remark}
\theoremstyle{definition}
\newtheorem{lemma}{Lemma}

\newtheorem{theorem}{Theorem}
\newtheorem{proposition}[theorem]{Proposition}

\newcommand{\neweq}[1]{\stackrel{\smash{\footnotesize\mathrm{#1}}}{=}}
\newcommand{\newapprox}[1]{\stackrel{\smash{\footnotesize\mathrm{#1}}}{\approx}}

\usepackage{tikz}
\newcommand*\circled[1]{\tikz[baseline=(char.base)]{
		\node[shape=circle,draw,inner sep=1.5pt] (char) {#1};}}

\newcommand*\squared[1]{\tikz[baseline=(char.base)]{
		\node[shape=rectangle,draw,inner sep=2.5pt] (char) {#1};}}

\makeatletter
\newenvironment{breakablealgorithm}
{
	\begin{center}
		\refstepcounter{algorithm}
		\hrule height.8pt depth0pt \kern2pt
		\renewcommand{\caption}[2][\relax]{
			{\raggedright\textbf{\ALG@name~\thealgorithm} ##2\par}%
			\ifx\relax##1\relax 
			\addcontentsline{loa}{algorithm}{\protect\numberline{\thealgorithm}##2}%
			\else 
			\addcontentsline{loa}{algorithm}{\protect\numberline{\thealgorithm}##1}%
			\fi
			\kern2pt\hrule\kern2pt
		}
	}{
		\kern2pt\hrule\relax
	\end{center}
}
\makeatother

\begin{document}
	
	\maketitle
	
	\footnotetext[0]{$^{1}$Neng Wan, Aditya Gahlawat, and Naira Hovakimyan are with the Department of Mechanical Science and Engineering, University of Illinois at Urbana-Champaign, Urbana, IL 61801. \{nengwan2, gahlawat, nhovakim\}@illinois.edu.\\
	$^{2}$Evangelos A. Theodorou is with the Department of Aerospace Engineering, Georgia Institute of Technology, Atlanta, GA 30332. evangelos.theodorou@gatech.edu.\\
	$^{3}$Petros G. Voulgaris is with the Department of Mechanical Engineering, University of Nevada, Reno, NV 89557. pvoulgaris@unr.edu.}
	\vspace{-4em}
	
	\begin{abstract}
		{Distributed algorithms for both discrete-time and continuous-time linearly solvable optimal control (LSOC) problems of networked multi-agent systems (MASs) are investigated in this paper. A distributed framework is proposed to partition the optimal control problem of a networked MAS into several local optimal control problems in factorial subsystems, such that each (central) agent behaves optimally to minimize the joint cost function of a subsystem that comprises a central agent and its neighboring agents, and the local control actions (policies) only rely on the knowledge of local observations. Under this framework, we not only preserve the correlations between neighboring agents, but moderate the communication and computational complexities by decentralizing the sampling and computational processes over the network. For discrete-time systems modeled by Markov decision processes, the joint Bellman equation of each subsystem is transformed into a system of linear equations and solved using parallel programming. For continuous-time systems modeled by It\^o diffusion processes, the joint optimality equation of each subsystem is converted into a linear partial differential equation, whose solution is approximated by a path integral formulation and a sample-efficient relative entropy policy search algorithm, respectively. The learned control policies are generalized to solve the unlearned tasks by resorting to the compositionality principle, and illustrative examples of cooperative UAV teams are provided to verify the effectiveness and advantages of these algorithms.}
	\end{abstract}
	\vspace{-0.5em}
	
	\titleformat*{\section}{\centering\MySubtitleFont}
	\titlespacing*{\section}{0em}{1.25em}{1.25em}[0em]
	\section{Introduction}
	
    The research of control and planning in multi-agent systems (MASs) has been developing rapidly during the past decade with the growing demands from areas, such as cooperative vehicles~\cite{Mahony_RAM_2012, Cichella_CSM_2016}, Internet of Things~\cite{Ota_TSG_2012}, intelligent infrastructures~\cite{Blaabjerg_TIE_2006, Guerrero_TIE_2013, Dorfler_TCNS_2016}, and smart manufacturing~\cite{Leitao_EAAI_2009}. Distinct from other control problems, control of MASs is characterized by the issues and challenges, which include, but are not limited to, a great diversity of possible planning and execution schemes, limited information and resources of the local agents, constraints and randomness of communication networks, optimality and robustness of joint performance. A good summary of recent progress in multi-agent control can be found in~\cite{Amato_CDC_2013, Bensoussan_2013, Cao_TII_2013, Frank_2013, Oh_Auto_2015, Qin_TIE_2017, Zhang_arxiv_2019}. Building upon these results and challenges, this paper puts forward a distributed optimal control scheme for stochastic MASs by extending the linearly-solvable optimal control algorithms to MASs subject to stochastic dynamics in the presence of an explicit communication network and limited feedback information.
    
    Linearly-solvable optimal control (LSOC) generally refers to the model-based stochastic optimal control (SOC) problems that can be linearized and solved with the facilitation of Cole-Hopf transformation, \textit{i.e.} exponential transformation of value function~\cite{Fleming_AMO_1977, Todorov_NIPS_2007}. Compared with other model-based SOC techniques, since LSOC formulates the optimality equations in linear form, it enjoys the superiority of analytical solution~\cite{Pan_NIPS_2015} and superposition principle~\cite{Todorov_NIPS_2009}, which makes LSOC a popular control scheme for robotics~\cite{Kupcsik_AI_2017, Williams_TRO_2018}. LSOC technique was first introduced to linearize and solve the Hamilton–Jacobi–Bellman (HJB) equation for continuous-time SOC problems~\cite{Fleming_AMO_1977}, and the application to discrete-time SOC, also known as the linearly-solvable Markov decision process (LSMDP), was initially studied in~\cite{Todorov_NIPS_2007}. More recent progress on single-agent LSOC problems can be found in~\cite{Peters_CAI_2010, Theodorou_JMLR_2010, Gomez_KDD_2014, Guan_TAC_2014, Pan_NIPS_2015, Williams_JGCD_2017}. 
    
    Different from many prevailing distributed control algorithms~\cite{Cao_TII_2013, Frank_2013, Oh_Auto_2015, Qin_TIE_2017}, such as consensus and synchronization that usually assume a given behavior, multi-agent SOC allows agents to have different objectives and optimizes the action choices for more general scenarios~\cite{Amato_CDC_2013}. Nevertheless, it is not straightforward to extend the single-agent SOC methods to multi-agent problems. The exponential growth of dimensionality in MASs and the consequent surges in computation and data storage demand more sophisticated and preferably distributed planning and execution algorithms. The involvement of communication networks (and constraints) requires the multi-agent SOC algorithms to achieve stability and optimality subject to local observation and more involved cost function. While the multi-agent Markov decision process (MDP) problem has received plenty of attention from both the fields of computer science and control engineering \cite{Guestrin_2002_NIPS, Becker_JAIR_2004, Amato_CDC_2013, Zhang_ICML_2018, Zhang_arxiv_2019, Zhang_arxiv_2019b}, there are relatively fewer results focused on multi-agent LSMDP. A recent result on multi-agent LSMDP represented the MAS problem as a single-agent problem by stacking the states of all agents into a joint state vector, and the scalability of the problem was addressed by parameterizing the value function~\cite{Daniel_2017}; however, since the planning and execution of the control action demand the knowledge of global states as well as a centralized coordination, the parallelization scheme of the algorithm was postponed in that paper. While there are more existing results focused on the multi-agent LSOC in continuous-time setting, most of these algorithms still depend on the knowledge of the global states, \textit{i.e.} a fully connected communication network, which may not be feasible or affordable to attain in practice. Some multi-agent LSOC algorithms also assume that the joint cost function can be factorized over agents, which basically simplifies the multi-agent control problem into multiple single-agent problems, and some features and advantages of MASs are therefore forfeited. Broek \textit{et al.} investigated the multi-agent LSOC problem for continuous-time systems governed by It\^o diffusion process~\cite{Broek_JAIR_2008}; a path integral formula was put forward to approximate the optimal control actions, and a graphical model inference approach was adopted to predict the optimal path distribution; nonetheless, the optimal control law assumed an accurate and complete knowledge of global states, and the inference was performed on the basis of mean-field approximation, which assumes that the cost function can be disjointly factorized over agents and ignores the correlations between agents. A distributed LSOC algorithm with infinite-horizon and discounted cost was studied in~\cite{Anderson_Robotica_2014} for solving a distance-based formation problem of nonholonomic vehicular network without explicit communication topology. The multi-agent LSOC problem was also recently discussed in~\cite{Williams_JGCD_2017} as an accessory result for a novel single-agent LSOC algorithm; an augmented dynamics was built by piling up the dynamics of all agents, and a single-agent LSOC algorithm was then applied to the augmented system. Similar to the discrete-time result in~\cite{Daniel_2017}, the continuous-time result resorting to augmented dynamics also presumes the fully connected network and faces the challenge that the computation and sampling schemes that originated from single-agent problem may become inefficient and possibly fail as the dimensions of augmented state and control grow exponentially in the number of agents.
    
    To address the aforementioned challenges, this paper investigates the distributed LSOC algorithms for discrete-time and continuous-time MASs with consideration of local observation, correlations between neighboring agents, efficient sampling and parallel computing. A distributed framework is put forward to partition the connected network into multiple factorial subsystems, each of which comprises a (central) agent and its neighboring agents, such that the local control action of each agent, depending on the local observation, optimizes the joint cost function of a factorial subsystem, and the sampling and computational complexities of each agent are related to the size of the factorial subsystem instead of the entire network. Sampling and computation are parallelized to expedite the algorithms and exploit the resource in network, with state measurements, intermediate solutions, and sampled data exchanged over the communication network. For discrete-time multi-agent LSMDP problem, we linearize the joint Bellman equation of each factorial subsystem into a system of linear equations, which can be solved with parallel programming, making both the planning and execution phases fully decentralized. For continuous-time multi-agent LSOC problem, instead of adopting the mean-field assumption and ignoring the correlations between neighboring agents, joint cost functions are permitted in the subsystems; the joint optimality equation of each subsystem is first cast into a joint stochastic HJB equation, and then solved with a distributed path integral control method and a sample-efficient relative entropy policy search (REPS) method, respectively. The compositionality of LSOC is utilized to efficiently generate a composite controller for unlearned task from the existing controllers for learned tasks. Illustrative examples of coordinated UAV teams are presented to verify the effectiveness and advantages of multi-agent LSOC algorithms. Building upon our preliminary work on distributed path integral control for continuous-time MASs~\cite{Wan_arXiv_2020}, this paper not only integrates the distributed LSOC algorithms for both discrete-time and continuous-time MASs, but supplements the previous result with a distributed LSMDP algorithm for discrete-time MASs, a distributed REPS algorithm for continuous-time MASs, a compositionality algorithm for task generalization, and more illustrative examples.

    The paper is organized as follows: \hyperref[sec2]{Section~2} introduces the preliminaries and formulations of multi-agent LSOC problems; \hyperref[sec3]{Section~3} presents the distributed LSOC algorithms for discrete-time and continuous-time MASs, respectively; \hyperref[sec4]{Section~4} shows the numerical examples, and \hyperref[sec5]{Section~5} draws the conclusions. Some notations used in this paper are defined as follows: For a set $\mathcal{S}$, $|\mathcal{S}|$ represents the cardinality of the set $\mathcal{S}$; for a matrix $X$ and a vector $v$, $\det X$ denotes the determinant of matrix $X$, and weighted square norm $\|v\|^2_X := v^\top X v$.

	\section{Preliminaries and Problem Formulation}\label{sec2}
	
	Preliminaries on MASs and LSOC are introduced in this section. The communication networks underlying MASs are represented by graphs, and the discrete-time and continuous-time LSOC problems are extended from single-agent scenario to MASs under a distributed planning and execution framework.

	\titleformat*{\subsection}{\MySubSubtitleFont}
	\titlespacing*{\subsection}{0em}{0.75em}{0.75em}[0em]
	\subsection{Multi-Agent Systems and Distributed Framework}\label{sec2.1}
	For a MAS consisting of $N \in \mathbb{N}$ agents, $\mathcal{D} = \{1, 2, \cdots, N\}$ denotes the index set of agents, and the communication network among agents is described by an undirected graph $\mathcal{G} = \{\mathcal{V}, \mathcal{E}\}$, which implies that the communication channel between any two agents is bilateral. The communication network $\mathcal{G}$ is assumed to be connected. An agent $i \in \mathcal{D}$ in the network is denoted as a vertex $v_i \in \mathcal{V} = \{v_1, v_2, \cdots, v_N\}$, and an undirected edge $(v_i, v_j) \in \mathcal{E} \subset \mathcal{V} \times \mathcal{V}$ in graph $\mathcal{G}$ implies that the agents $i$ and $j$ can measure the states of each other, which is denoted by $x_i = [x_{i(1)}, x_{i(2)}, \cdots, x_{i(M)}]^\top \in \mathbb{R}^{M}$ for agent $i \in \mathcal{D}$. Agents $i$ and $j$ are neighboring or adjacent agents, if there exists a communication channel between them, and the index set of agents neighboring to agent $i$ is denoted by $\mathcal{N}_i$ with $\bar{\mathcal{N}}_i = \mathcal{N}_i \cup \{ i \}$. We use column vectors $\bar{x}_i = [x_i^\top, x^\top_{j \in \mathcal{N}_i}]^\top \in \mathbb{R}^{M \cdot |\mathcal{\bar N}_i|   }$ to denote the joint state of agent $i$ and its adjacent agents $\mathcal{N}_i$, which together group up the factorial subsystem $\bar{\mathcal{N}}_i$ of agent $i$, and $x = [x_1^\top, x_2^\top, \cdots, x^\top_n]^\top \in \mathbb{R}^{M \cdot N}$  denotes the global states of MAS. \hyperref[fig1]{Figure~1} shows a MAS and all its factorial subsystems.
	
	To optimize the correlations between neighboring agents while not intensifying the communication and computational complexities of the network or each agent, our paper proposes a distributed planning and execution framework, as a trade-off scheme between the easiness of implementation and optimality of performance, for multi-agent LSOC problems. Under this distributed framework, the local control action $u_i$ of agent $i \in N$ is computed by solving the local LSOC problem defined in factorial subsystem $\bar{\mathcal{N}}_i$. Instead of requiring the cost functions fully factorized over agents \cite{Broek_JAIR_2008} or the knowledge of global states \cite{Daniel_2017, Williams_JGCD_2017}, joint cost functions are permitted in every factorial subsystem, which captures the correlations and cooperation between neighboring agents, and the local control action $u_i$ only relies on the local observation $\bar{x}_i$ of agent $i$, which also simplifies the structure of communication network. Meanwhile, the global computational complexity is no longer exponential with respect to the total amount of agents in network $|\mathcal{D}|$, but becomes linear with respect to it in general, and the computational complexity of local agent $i$ is only related to the number of agents in factorial subsystem $\bar{\mathcal{N}}_i$. However, since the local control actions are computed from the local observations with only partial information, the distributed LSOC law obtained under this framework is usually a sub-optimal solution, despite that it coincides with the global optimal solution when the communication network is fully connected. More explanations and discussions on this will be given at the end of this section. Before that, we first reformulate the discrete-time and continuous-time LSOC problems from the single-agent scenario to a multi-agent setting that is compatible with our distributed framework.
	\begin{figure}[htpb]
		\centering 
		\includegraphics[width=0.75\textwidth]{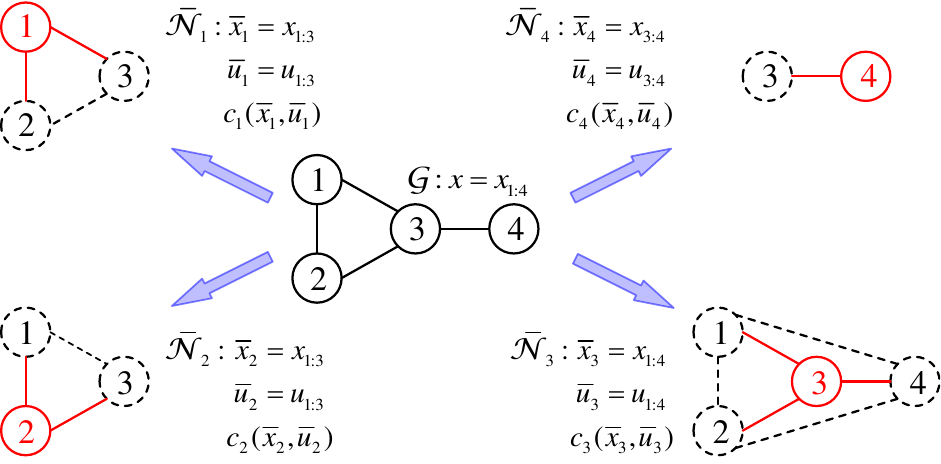}
		\caption{An example of MAS and factorial subsystems. MAS $\mathcal{G}$ with four agents can be partitioned into four factorial subsystems $\bar{\mathcal{N}}_1, \bar{\mathcal{N}}_2, \bar{\mathcal{N}}_3$, and $\bar{\mathcal{N}}_4$, and each subsystem is assumed to be fully connected.}\label{fig1}
	\end{figure}

\subsection{Discrete-Time Dynamics}\label{Sec2_2}
Discrete-time SOC for single-agent systems, also known as the single-agent MDP, is briefly reviewed and then generalized to the networked MAS scenario. We consider the single-agent MDPs with finite state space and continuous control space in the first-exit setting. For a single agent $i \in \mathcal{D}$, the state variable $x_i$ belongs to a finite set $\mathcal{S}_i = \{s^i_1, s^i_2, \cdots\} = \mathcal{I}_i \cup \mathcal{B}_i$, which may be generated from an infinite-dimensional state problem by an appropriate coding scheme~\cite{Sutton_2018}. $\mathcal{I}_i \subset \mathcal{S}_i$ denotes the set of interior states of agent $i$, and $\mathcal{B}_i \subset \mathcal{S}_i$ denotes the set of boundary states. Without communication and interference from other agents, the passive dynamics of agent $i$ follows the probability distribution
\begin{equation*}
	x'_i \sim p_i( \cdot  |  x_i),
\end{equation*} 
where $x_i, x'_i \in \mathcal{S}_i$, and $p_i(x'_i  |  x_i)$ denotes the transition probability from state $x_i$ to $x'_i$. When taking control action $u_i$ at state $x_i$, the controlled dynamics of agent $i$ is described by the distribution mapping
\begin{equation}\label{single_agent_controlled}
	x_i' \sim  u_i (  \cdot | x_i) = p_i(  \cdot  |  x_i, u_i),
\end{equation}
where $u_i(x'_i |  x_i)$ or $p_i(x'_i  |  x_i, u_i)$ denotes the transition probability from state $x_i$ to state $x'_i$ subject to control $u_i$ that belongs to a continuous space. We require that $u_i(x'_i  |  x_i) = 0$, whenever $p_i(x'_i  |  x_i) = 0$, to prevent the direct transitions to the goal states. When $x_i \in \mathcal{I}_i \subset \mathcal{S}_i$, the immediate or running cost function of LSMDPs is designed as:
\begin{equation}\label{eq3}
	c_i(x_i, u_i) = q_i(x_i) + \textrm{KL}( u_i(\cdot  | x_i) \parallel p_i( \cdot  |  x_i)),
\end{equation}
where the state cost $q_i(x_i)$ can be an arbitrary function encoding how (un)desirable different states are, and the KL-divergence\footnote{The KL-divergence (relative entropy) between two discrete probability mass functions $p(x)$ and $q(x)$ is defined as\begin{equation*}\label{KLD}
		\textrm{KL}(p \parallel q) = \sum_{x \in \mathcal{X}} p(x) \log[{p(x)} / {q(x)}],
	\end{equation*}which has an absolute minimum $0$ when $p(x) = q(x), \forall x \in \mathcal{X}$. For two continuous probability density functions $p(x)$ and $q(x)$, the KL-divergence is defined as 
	\begin{equation*}
		\textrm{KL}(p  \parallel  q) = \int_{x\in \chi}  p(x) \log[p(x) / q(x)] dx.
\end{equation*}} measures the cost of control actions. When $x_i \in \mathcal{B}_i \subset \mathcal{S}_i$, the final cost function is defined as $\phi_i(x_i) \geq 0$.  The cost-to-go function of first-exit problem starting at state-time pair $(x_i^{t_0}, t_0)$ is defined as
\begin{equation}\label{eq7}
	J^{u_i}_i(x^{t_0}_i, t_0) = \mathbb{E}^{u_i} \bigg[ \phi_i(x_{i}^{t_f}) + \sum_{\tau = t_0}^{t_f - 1} c_i(x_i^\tau, u^\tau_i)  \bigg], 
\end{equation}
where $(x^\tau_i, u_i^\tau)$ is the state-action pair of agent $i$ at time step $\tau$, $x^{t_f}_i$ is the terminal or exit state, and the expectation $\mathbb{E}^{u_i}$ is taken with respect to the probability measure under which $x_i$ satisfies~\eqref{single_agent_controlled} given the control law $u_i = (u^{t_0}_i, u^{t_1}_i, \cdots, u^{t_f - 1}_i)$ and initial condition $x_i^{t_0}$. The objective of discrete-time stochastic optimal control problem is to find the optimal policy $u_i^*$ and value functions $V_i(x_i)$ by solving the Bellman equation
\begin{equation}\label{eq8}
	V_i(x_i) = \min_{u_i} \left\{ c_i(x_i, u_i) + \mathbb{E}_{x'_i \sim u_i(\cdot | x_i)}[V_i(x'_i)]  \right\},
\end{equation}
where the value function $V_i(x_i)$ is defined as the expected cumulative cost for starting at state $x_i$ and acting optimally thereafter, \textit{i.e.} $V_i(x_i) = \min_{u_i} J^{u_i}_i(x^{t_0}_i, t_0)$.

Based on the formulations of single-agent LSMDP and augmented dynamics, we introduce a multi-agent LSMDP formulation subject to the distributed framework of the factorial subsystem. For simplicity, we assume that the passive dynamics of agents in MAS are homogeneous and mutually independent, \textit{i.e.} agents without control are governed by identical dynamics and do not interfere or collide with each other. This assumption is also posited in many previous papers on multi-agent LSOC or distributed control~\cite{Olfati-Saber_PIEEE_2007, Broek_JAIR_2008, Frank_2013, Williams_JGCD_2017}. Since agent $i \in \mathcal{D}$ can only observe the states of neighboring agents $\mathcal{N}_i$, we are interested in the subsystem $\bar{\mathcal{N}}_i = \mathcal{N}_i \cup \{i\}$ when computing the control law of agent $i$. Hence, the autonomous dynamics of subsystem $\bar{\mathcal{N}}_i$ follow the distribution mapping
\begin{equation}\label{eq1}
	\bar{x}'_i \sim \bar{p}_i(  \cdot | \bar{x}_i) =\prod_{j \in {\mathcal{\bar{N}}}_i } p_j(  \cdot  |  x_j),
\end{equation}
where the joint state $\bar{x}_i, \bar{x}'_i \in \prod_{j\in\mathcal{\bar{N}}_i} S_j = \bar{\mathcal{S}}_i = \{\bar{s}^i_1, \bar{s}^i_2, \cdots \}= \bar{\mathcal{I}}_i \cup \bar{\mathcal{B}}_i$, and distribution information $\bar{p}_i( \cdot  |  \bar{x}_i)$ are generally only accessible to agent $i$. Similarly, the global state of all agents $x' \sim p( \cdot  |  x) = \prod_{i=1}^{N} p_i(  \cdot |  x_i)$, which is usually not available to a local agent $i$ unless it can receive the information from all other agents $\mathcal{D} \backslash \{i\}$, \textit{e.g.} agent 3 in~\hyperref[fig1]{Figure~1}. Since the local control action $u_i$ only relies on the local observation of agent $i$, we assume that the joint posterior state $\bar{x}'_i$ of subsystem $\bar{\mathcal{N}}_i$ is exclusively determined by the joint prior state $\bar{x}_i$ and joint control $\bar{u}_i$ when computing the optimal control actions in subsystem $\bar{\mathcal{N}}_i$. More intuitively, under this assumption, the local LSOC algorithm in subsystem $\bar{\mathcal{N}}_i$ only requires the measurement of joint state $\bar{x}_i$ and treats subsystem $\bar{\mathcal{N}}_i$ as a complete connected network, as shown in~\hyperref[fig1]{Figure~1}. When each agent in $\mathcal{\bar{N}}_i$ samples their control action independently, the joint controlled dynamics of the factorial subsystem $\bar{\mathcal{N}}_i$ satisfies
\begin{equation}\label{eq2}
	\bar{x}'_i \sim \bar{u}_i( \cdot | \bar{x}_i) =  \prod_{j \in \mathcal{\bar{N}}_i} u_j ( \cdot  |  \bar{x}_i) = \prod_{j\in\mathcal{\bar{N}}_i} p_j(  \cdot  |  x_i, x_{j\in\mathcal{N}_i}, u_j),
\end{equation}
where the joint state $\bar{x}_i$ and joint distribution $\bar{u}_i(\cdot | \bar{x}_i)$ are only accessible to agent $i$ in general. Once we figure out the joint control distribution $\bar{u}_i(\cdot | \bar{x}_i)$ for subsystem $\bar{\mathcal{N}}_i$, the local control distribution $u_i(\cdot | \bar{x}_i)$ of agent $i$ can be retrieved by calculating the marginal distribution. The joint immediate cost function for subsystem $\bar{\mathcal{N}}_i$ when $\bar{x}_i \in \mathcal{\bar{I}}_i$ is defined as follows 
\begin{equation}\label{eq4}
	c_i(\bar{x}_i, \bar{u}_i) =q_i(\bar{x}_i)  + \textrm{KL}(\bar{u}_i(  \cdot  | \bar{x}_i) \parallel  \bar{p}_i(  \cdot  | \bar{x}_i)) = q_i(\bar{x}_i)  + \sum_{j \in \bar{\mathcal{N}}_i}\textrm{KL}(u_j(  \cdot  | \bar{x}_i) \parallel  p_j(  \cdot  | x_j)) , 
\end{equation}
where the state cost $q_i(\bar{x}_i)$ can be an arbitrary function of joint state $\bar{x}_i$, \textit{i.e.} a constant or the norm of disagreement vector, and the second equality follows from \eqref{eq1} and \eqref{eq2}, which implies that the joint control cost is the cumulative sum of the local control costs. When $\bar{x}_i \in \mathcal{\bar{B}}_i$, the exit cost function $\phi_i(\bar{x}_i) = \sum_{j\in\mathcal{\bar{N}}_i} \omega^i_{j} \cdot \phi_j(x_j)$, where $\omega^i_{j} > 0$ is a weight measuring the priority of assignment on agent $j$. In order to improve the success rate in application, it is preferable to assign $\omega_i^i$ as the largest weight when computing the control distribution $\bar{u}_i$ in subsystem $\mathcal{\bar{N}}_i$. Subsequently, the joint cost-to-go function of first-exit problem in subsystem $\bar{\mathcal{N}}_i$ becomes
\begin{equation}\label{J_CTG_Discrete_Time}
	J^{\bar{u}_i}_i(\bar{x}^{t_0}_i, t_0) = \mathbb{E}^{\bar{u}_i} \bigg[ \phi_i(\bar{x}_{i}^{t_f}) + \sum_{\tau = t_0}^{t_f - 1} c_i(\bar{x}_i^\tau, \bar{u}^\tau_i)  \bigg].
\end{equation}
Some abuses of notations occur when we define the cost functions $c_i$ and $\phi_i$, cost-to-go function $J_i$, and value function $V_i$ in single-agent setting and the factorial subsystem; one can differentiate the different settings from the arguments of these functions. Derived from the single-agent Bellman equation~\eqref{eq8}, the joint optimal control action $\bar{u}_i^*( \cdot  | \bar{x}_i)$ subject to the joint cost function~\eqref{eq4} can be solved from the following joint Bellman equation in subsystem $\bar{\mathcal{N}}_i$
\begin{equation}\label{CompBellman}
	V_i(\bar{x}_i) = \min_{\bar u_i} \left\{ c_i(\bar{x}_i, \bar{u}_i) + \mathbb{E}_{\bar{x}'_i \sim \bar{u}_i(\cdot | \bar{x}_i)}[V_i(\bar{x}'_i)]  \right\},
\end{equation}
where $V_i(\bar{x}_i)$ is the (joint) value function of joint state $\bar{x}_i$. A linearization method as well as a parallel programming method for solving~\eqref{CompBellman} will be discussed in~\hyperref[sec3]{Section~3}.

\subsection{Continuous-Time Dynamics}
For continuous-time LSOC problems, we first consider the dynamics of single agent $i$ described by the following It\^o diffusion process 
\begin{equation}\label{eq10}
	dx_i = f_i(x_i, t)dt + B_i(x_i) [u_i(x_i, t)dt + \sigma_i d w_i],
\end{equation}
where $x_i\in \mathbb{R}^{M}$ is agent $i$'s state vector from an uncountable state space; $f_i(x_i, t) + B_i(x_i) \cdot u _i(x_i, t) \in \mathbb{R}^{M}$ is the deterministic drift term with passive dynamics $f_i(x_i, t)$, control matrix $B_i(x_i) \in \mathbb{R}^{M\times P}$ and control action $u_i(x_i, t) \in \mathbb{R}^P$; noise $dw_i \in \mathbb{R}^P$ is a vector of possibly correlated\footnote{When the components of $d{\tilde w}_i = [d{\tilde w}_{i, (1)}, \cdots, d{\tilde w}_{i, (P)}]^\top$ are correlated and satisfy a multi-variate normal distribution $N(0, \Sigma_i)$, by using the Cholesky decomposition $\Sigma_i = \sigma_i \sigma^\top_i$, we can rewrite $d\tilde{w}_i = \sigma_i dw_i$, where $dw_i$ is a vector of Brownian components with zero drift and unit-variance rate.} Brownian components with zero mean and unit rate of variance, and the positive semi-definite matrix $\sigma_i \in \mathbb{R}^{P \times P}$ denotes the covariance of noise $dw_i$.  When $x_i \in \mathcal{I}_i$, the running cost function is defined as
\begin{equation}\label{eq11}
	c_i(x_i, u_i) = q_i(x_i) + \dfrac{1}{2}u_i(x_i, t)^\top R_i u_i(x_i, t),
\end{equation}
where $q_i(x_i) \geq 0$ is the state-related cost, and $u_i^\top R_iu_i$ is the control-quadratic term with matrix $R_i \in \mathbb{R}^{P\times P}$ being positive definite. When $x^{t_f}_i \in \mathcal{B}_i$, the terminal cost function is $\phi_i(x^{t_f}_i)$, where $t_f$ is the exit time. Hence, the cost-to-go function of first-exit problem is defined as
\begin{equation}\label{J_CTG_Cont_Time}
	J^{u_i}_i(x_i^t, t) = \mathbb{E}^{u_i}_{x_i^t, t} \left[ \phi_i(x_{i}^{t_f}) + \int_{t}^{t_f}  c_i(x_i(\tau), u_i(\tau)) \ d\tau \right],
\end{equation}
where the expectation is taken with respect to the probability measure under which $x_i$ is the solution to~\eqref{eq10} given the control law $u_i$ and initial condition $x_i(t)$. The value function is defined as the minimal cost-to-go function $V_i(x_i, t) = \min_{u_i} J_i^{u_i}(x_i^t, t)$. Subject to the dynamics~\eqref{eq10} and running cost function~\eqref{eq11}, the optimal control action $u_i^*$ can be solved from the following single-agent stochastic Hamilton–Jacobi–Bellman (HJB) equation:
\begin{align}\label{singleHJB}
	- \partial_t V_i(x_i, t) = \min_{u_i} \Big\{ c_i(x_i, u_i)   +  [f_i(x_i, t)  & + B_i(x_i)u_i(x_i,t)]^\top \cdot \nabla_{x_i}  V_i(x_i, t) \\
	&  + \frac{1}{2} \textrm{tr} \left[B_i(x_i)\sigma_i\sigma^\top_iB_i(x_i)^\top \cdot \nabla^2_{x_ix_i} V_i(x_i, t) \right] \Big\}, \nonumber
\end{align}
where $\nabla_{x_i}$ and $\nabla^2_{x_i x_i}$ respectively refer to the gradient and Hessian matrix with $\nabla_{x_i} V_i = [ \partial V_i / \partial x_{i(1)}, \break \cdots,  \partial V_i / \partial x_{i(M)}]^\top$ and elements $[\nabla^2_{x_ix_i}V_i]_{m,n} = \partial^2 V_i / \partial x_{i(m)} \partial x_{i(n)}$. A few methods have been proposed to solve the stochastic HJB in~\eqref{singleHJB}, such as the approximation methods via discrete-time MDPs or eigenfunction~\cite{Todorov_IEEESADPRL_2009} and path integral approaches~\cite{Kappen_PRL_2005, Broek_JAIR_2008, Theodorou_JMLR_2010}.

Similar to the extension of discrete-time LSMDP from single-agent setting to MAS, the joint continuous-time dynamics for factorial subsystem $\bar{\mathcal{N}}_i$ is described by
\begin{equation}\label{eq13}
	d \bar{x}_i = \bar{f}_i(\bar{x}_i,t) dt + \bar{B}_i(\bar{x}_i)  \left[ \bar{u}_i(\bar{x}_i,t) dt + \bar{\sigma}_i d\bar{w}_i  \right],
\end{equation}
where the joint passive dynamics vector is denoted by $\bar{f}_i(\bar{x}_i,t) = [f_i(x_i, t)^\top, f_{j\in\mathcal{N}_i}(x_j,t)^\top]^\top \in \mathbb{R}^{M \cdot |\mathcal{\bar N}_i|}$, the joint control matrix is denoted by $\bar{B}_i(\bar{x}_i) = \textrm{diag}\{ B_i(x_i), B_{j\in\mathcal{N}_i}(x_{j}) \} \in \break \mathbb{R}^{M \cdot |\bar{\mathcal{N}}_i| \times P \cdot |\bar{\mathcal{N}}_i| }$,  $\bar{u}_i(\bar{x}_i,t) = [u_i(\bar{x}_i,t)^\top,  u_{j \in \mathcal{N}_i}(\bar{x}_i, t)^\top]^\top \in \mathbb{R}^{P \cdot |\bar{\mathcal{N}}_i|}$ is the joint control action,  $d\bar{w}_i = [dw^\top_i,   dw^\top_{j\in\mathcal{N}_i}]^\top  \in  \mathbb{R}^{P\cdot|\bar{\mathcal{N}}_i|}$ is the joint noise vector, and the joint covariance matrix is denoted by $\bar{\sigma}_i = \textrm{diag}\{ \sigma_i, \sigma_{j\in\mathcal{N}_i} \} \in \mathbb{R}^{P \cdot |\bar{\mathcal{N}}_i| \times P \cdot |\bar{\mathcal{N}}_i|}$. Analogous to the discrete-time scenario, we assume that the passive dynamics of agents are homogeneous and mutually independent, and for the local planning algorithm on agent $i$ or subsystem $\bar{\mathcal{N}}_i$, which computes the local control action $u_i(\bar{x}_i, t)$ for agent $i$ and joint control action $\bar{u}_i(\bar{x}_i, t)$ in subsystem $\bar{\mathcal{N}}_i$, the evolution of joint state $\bar{x}_i$ only depends on the current values of $\bar{x}_i$ and joint control $\bar{u}_i(\bar{x}_i, t)$. When $\bar{x}_i \in \mathcal{\bar I}_i$, the joint immediate cost function for subsystem $\bar{\mathcal{N}}_i$ is defined as
\begin{equation}\label{cont_cost}
	c_i(\bar{x}_i, \bar{u}_i) = q_i(\bar{x}_i) + \frac{1}{2}\bar{u}_i(\bar{x}_i, t)^\top \bar{R}_i \bar{u}_i(\bar{x}_i, t),
\end{equation}
where the state-related cost $q_i(\bar{x}_i)$ can be an arbitrary function measuring the (un)desirability of different joint states $\bar{x}_i \in \mathcal{\bar S}_i$, and $\bar{u}_i^\top \bar{R}_i \bar{u}_i$ is the control-quadratic term with matrix $\bar{R}_i \in \mathbb{R}^{P \cdot |\bar{\mathcal{N}}_i| \times P \cdot |\bar{\mathcal{N}}_i|}$ being positive definite. When $\bar{R}_i = \textrm{diag}\{R_i, R_{j \in \mathcal{N}_i} \}$ with $R_i$ and $R_j$ defined in~\eqref{eq11}, the joint control cost term in~\eqref{cont_cost} satisfies $\bar{u}_i^\top \bar{R}_i \bar{u}_i = \sum_{j\in\mathcal{\bar{N}}_i}u_j^\top R_ju_j$, which is symmetric with respect to the relationship of discrete-time control costs in~\eqref{eq4}. When $\bar{x}_i \in \bar{\mathcal{B}}_i$, the terminal cost function is defined as $\phi_i(\bar{x}_i) = \sum_{j\in\mathcal{\bar{N}}_i}\omega_j^i \cdot \phi_j(x_j)$, where $\omega_j^i > 0$ is the weight measuring the priority of assignment on agent $j$, and we let the weight $\omega_i^i$ dominate other weights $\omega_{j \in \mathcal{N}_i}^i$ to improve the success rate. Compared with the cost functions fully factorized over agents, the joint cost functions in~\eqref{cont_cost} can gauge and facilitate the correlation and cooperation between neighboring agents. Subsequently, the joint cost-to-go function of first-exit problem in subsystem $\bar{\mathcal{N}}_i$ is defined as
\begin{equation*}
	J^{\bar{u}_i}_i(\bar{x}_i^t, t) = \mathbb{E}^{\bar{u}_i}_{\bar{x}_i^t, t} \left[ \phi_i(\bar{x}_{i}^{t_f}) + \int_{t}^{t_f} c_i(\bar{x}_i(\tau), \bar{u}_i(\tau)) \ d\tau  \right].
\end{equation*}
Let the (joint) value function $V_i(\bar{x}_i, t)$ be the minimal cost-to-go function, \textit{i.e.} $V_i(\bar{x}_i, t) = \min_{\bar{u}_i} J_i^{\bar{u}_i}(\bar{x}_i^t, t)$. We can then compute the joint optimal control action $\bar{u}_i^*$ of subsystem $\mathcal{\bar{N}}_i$ by solving the following joint optimality equation
\begin{equation}\label{ContBellman}
	V_i(\bar{x}_i, t) = \min_{\bar{u}_i} \mathbb{E}^{\bar{u}_i}_{\bar{x}_i^t, t} \left[ \phi_i(\bar{x}_{i}^{t_f}) + \int_{t}^{t_f} c_i(\bar{x}_i(\tau), \bar{u}_i(\tau)) \ d\tau  \right].
\end{equation}
A distributed path integral control algorithm and a distributed REPS algorithm for solving~\eqref{ContBellman} will be respectively discussed in~\hyperref[sec3]{Section~3}. Discussion on the relationship between discrete-time and continuous-time LSOC dynamics can be found in~\cite{Todorov_NIPS_2007, Todorov_PNAS_2009}.

\begin{remark}
	Although each agent $i \in \mathcal{D}$ under this framework acts optimally to minimize a joint cost-to-go function defined in their subsystem $\mathcal{\bar{N}}_i$, the distributed control law obtained from solving the local problem \eqref{CompBellman} or \eqref{ContBellman} is still a sub-optimal solution unless the communication network $\mathcal{G}$ is fully connected. Two main reasons account for this sub-optimality. First, when solving \eqref{CompBellman} or \eqref{ContBellman} for the joint (or local) optimal control $\bar{u}_i^*$ (or $u_i^*$), we ignore the connections of agents outside the subsystem $\bar{\mathcal{N}}_i$ and assume that the evolution of joint state $\bar{x}_i$ only relies on the current values of $\bar{x}_i$ and joint control $\bar{u}_i$. This simplification is reasonable and almost accurate for the central agent $i$ of subsystem $\bar{\mathcal{N}}_i$, but not for the non-central agents $j \in \mathcal{N}_i$, which are usually adjacent to other agents in $\mathcal{N}_j \backslash \mathcal{N}_i$. Therefore, the local optimal control actions $\bar{u}_j^*$ of other agents $j \in \mathcal{D} \backslash \{i\}$ are respectively computed from their own subsystems $\bar{\mathcal{N}}_j$, which may contradict the joint optimal control $\bar{u}_i$ solved in subsystem $\bar{\mathcal{N}}_i$ and result in a sub-optimal solution. Similar conflicts widely exist in the distributed control and optimization problems subject to limited communication and partial observation, and some serious and heuristic studies on the global- and sub-optimality of distributed subsystems have been conducted in~\cite{Johari_MOR_2004, Nedic_TAC_2010, Frank_2013, Voulgaris_CDC_2017}. We will not dive into those technical details in this paper, as we believe that the executability of a sub-optimal plan with moderate communication and computational complexities should outweigh the performance gain of a global-optimal but computationally intractable plan in practice. In this regard, the distributed framework built upon factorial subsystems, which captures the correlations between neighboring agents while ignoring the further connections outside subsystems, provides a trade-off alternative between the optimality and complexity and is analogous to the structured prediction framework in supervised learning.
\end{remark}

	\section{Distributed Linearly-Solvable Optimal Control}\label{sec3}
	Subject to the multi-agent LSOC problems formulated in~\hyperref[sec2]{Section~2}, the linearization methods and distributed algorithms for solving the joint discrete-time Bellman equation~\eqref{CompBellman} and joint continuous-time optimality equation~\eqref{ContBellman} are discussed in this section.
	
	\subsection{Discrete-Time Systems}\label{Sec3_1}
	We first consider the discrete-time MAS with dynamics~\eqref{eq2} and immediate cost function~\eqref{eq4}, which give the joint Bellman equation~\eqref{CompBellman} in subsystem $\mathcal{\bar{N}}_i$
	\begin{equation*}
		V_i(\bar{x}_i) = \min_{\bar u_i} \left\{ c_i(\bar{x}_i, \bar{u}_i) + \mathbb{E}_{\bar{x}'_i \sim \bar{u}_i(\cdot | \bar{x}_i)}[V_i(\bar{x}'_i)]  \right\}. 
	\end{equation*}
	In order to compute the value function $V_i(\bar{x}_i)$ and optimal control action $\bar{u}_i^*( \cdot  | \bar{x}_i)$ from equation~\eqref{CompBellman}, an exponential or Cole-Hopf transformation is employed to linearize~\eqref{CompBellman} into a system of linear equations, which can be cast into a decentralized programming and solved in parallel. Local optimal control distribution $u_i^*( \cdot  | \bar{x}_i)$ that depends on the local observation of agent $i$ is then derived by marginalizing the joint control distribution $\bar{u}_i^*( \cdot  | \bar{x}_i)$. 
	
	\subsubsection*{A. Linearization of Joint Bellman Equation}
	Motivated by the exponential transformation employed in~\cite{Todorov_PNAS_2009} for single-agent system, we define the desirability function $Z_i(\bar{x}_i)$ for joint state $\bar{x}_i \in \mathcal{\bar I}_i$ in subsystem $\mathcal{\bar{N}}_i$ as
	\begin{equation}\label{ExpTrans}
		Z_i(\bar{x}_i) = \exp[-V_i(\bar{x}_i)], 
	\end{equation} 
	which implies that the desirability function $Z_i(\bar{x}_i)$ is negatively correlated with the value function $V_i(\bar{x}_i)$, and the value function can also be written conversely as a logarithm of the desirability function, $V_i(\bar{x}_i) = \log 1 / Z_i(\bar{x}_i)$. For boundary states $\bar{x}_i \in \mathcal{\bar B}_i$, the desirability functions are defined as $Z_i(\bar{x}_i) = \exp[-\phi_i(\bar{x}_i)]$. Based on the transformation~\eqref{ExpTrans}, a linearized joint Bellman equation~\eqref{CompBellman} along with the joint optimal control distribution is presented in \hyperref[thm1]{Theorem~1}.

	\setcounter{theorem}{0}
	\begin{theorem}\label{thm1}
		With exponential transformation~\eqref{ExpTrans}, the joint Bellman equation~\eqref{CompBellman} for subsystem $\mathcal{\bar{N}}_i$ is equivalent to the following linear equation with respect to the desirability function	
		\begin{equation}\label{eq_prop1}
			Z_i(\bar{x}_i) =  \exp(-q_i(\bar{x}_i)) \cdot \sum_{\bar{x}'_i}\bar{p}_i(\bar{x}'_i|\bar{x}_i)Z_i(\bar{x}'_i),
		\end{equation}
		where $q_i(\bar{x}_i)$ is the state-related cost defined in~\eqref{cont_cost}, and $\bar{p}_i(\bar{x}_i'|\bar{x}_i)$ is the transition probability of passive dynamics in~\eqref{eq1}. The joint optimal control action $\bar{u}_i^*( \cdot | \bar{x}_i)$ solving~\eqref{CompBellman} satisfies
		\begin{equation}\label{OptimalControl}
			\bar{u}_i^*( \cdot | \bar{x}_i) = \frac{\bar{p}_i(  \cdot   | \bar{x}_i)Z_i(\cdot)}{\sum_{\bar{x}'_i}\bar{p}_i(\bar{x}'_i|\bar{x}_i)Z_i(\bar{x}'_i)},
		\end{equation}
		where $\bar{u}_i^*(\bar{x}_i' | \bar{x}_i)$ is the transition probability from $\bar{x}_i$ to $\bar{x}_i'$ in controlled dynamics~\eqref{eq2}.
	\end{theorem}
	\begin{proof}
		See \hyperref[appA]{Appendix~A} for the proof.
	\end{proof}

	Joint optimal control action $\bar{u}^*_i(\bar{x}_i' | \bar{x}_i)$ in~\eqref{OptimalControl} only relies on the joint state $\bar{x}_i$ of subsystem $\mathcal{\bar{N}}_i$, \textit{i.e.} local observation of agent $i$. To execute this control action~\eqref{OptimalControl}, we still need to figure out the values of desirability functions or value functions, which can be solved from~\eqref{eq_prop1}. A conventional approach for solving~\eqref{eq_prop1}, which was adopted in~\cite{Todorov_NIPS_2007, Todorov_NIPS_2009, Todorov_PNAS_2009} for single-agent LSMDP, is to rewrite~\eqref{eq_prop1} as a recursive formula and approximate the solution by iterations. We can approximate the desirability functions of interior states $\bar{x}_i \in \mathcal{\bar{I}}_i$ by recursively executing the following update law
	\begin{equation}\label{Z_update}
		Z_{\mathcal{I}} =  \Theta Z_{\mathcal{I}} + \Omega Z_\mathcal{B},
	\end{equation} 
	where $Z_\mathcal{I}$ and $Z_\mathcal{B}$ are respectively the desirability vectors of interior states and boundary states; the diagonal matrix $\Theta = \textrm{diag}\{\exp(-q_{\mathcal{I}})\} \cdot  P_{\mathcal{II}}$ with $q_{\mathcal{I}}$ denoting the state-related cost of interior states, $P_{\mathcal{II}} = [p_{mn}]$ denoting the transition probability matrix between interior states, and $p_{mn} = \bar{p}_i(\bar{x}_i' = \bar{s}^i_n \ | \ \bar{x}_i = \bar{s}^i_m)$ for $\bar{s}^i_n, \bar{s}^i_m \in \mathcal{\bar I}_i$; the matrix $\Omega = \textrm{diag}\{ \exp(-q_{\mathcal{I}}) \} \cdot P_{\mathcal{IB}}$ with $P_{\mathcal{IB}} = [p_{mn}]$ denoting the transition probability matrix from interior states to boundary states, and $p_{mn} = \bar{p}_i(\bar{x}_i' = \bar{s}^i_n \ | \ \bar{x}_i = \bar{s}^i_m)$ for $\bar{s}^i_m \in \mathcal{\bar I}_i$ and $\bar{s}^i_n \in \mathcal{\bar{B}}_i$. Assigning an initial value to the desirability vector $Z_{\mathcal{I}}$, the recursive formula~\eqref{Z_update} is guaranteed to converge to a unique solution, since the spectral radius of matrix $\Theta$ is less than $1$. More detailed convergence analysis on this iterative solver of LSMDPs has been given in~\cite{Todorov_NIPS_2007}. However, this centralized solver is inefficient when dealing with the MAS problems with high-dimensional state space, which requires the development of a distributed solver that exploits the resource of network and expedites the computation.

	\subsubsection*{B. Distributed Planning Algorithm} 
	While most of the distributed SOC algorithms are executed by local agents in a decentralized approach, a great number of these algorithms still demand a centralized solver in planning phase, which becomes a bottleneck for their implementations when the amount of agents scales up~\cite{Amato_CDC_2013}. For a fully connected MAS with $N$ agents, when each agent has $|\mathcal{I}|$ interior states, the dimension of the vector $Z_\mathcal{I}$ in \eqref{Z_update} is $|\mathcal{I}|^N$, and as the number of agents $N$ grows up, it will become more intractable for a central computation unit to store all the data and execute all the computation required by~\eqref{Z_update} due to the \textit{curse of dimensionality}. Although the subsystem-based distributed framework can alleviate this problem by demanding a less complex network and making the dimension and computational complexity only related to the sizes of factorial subsystems, it is still preferable to utilize the resources of MAS by distributing the data and computational task of~\eqref{Z_update} to each local agent in $\bar{\mathcal{N}}_i$, instead of relying on a central planning agent. Hence, we rewrite the linear equation~\eqref{Z_update} in the following form
	\begin{equation}\label{eq22}
		(I - \Theta) Z_{\mathcal{I}} = \Omega Z_{\mathcal{B}},
	\end{equation}
	and formulate~\eqref{eq22} into a parallel programming problem. In order to solve the desirability vector $Z_{\mathcal{I}}$ from~\eqref{eq22} via a distributed approach, each agent in subsystem $\mathcal{\bar{N}}_i$ only needs to know (store) a subset (rows) of the partitioned matrix $\left[ I - \Theta, \hspace{5pt} \Omega Z_\mathcal{B} \right]$. Subject to the equality constraints laid by its portion of coefficients, agent $j \in \mathcal{\bar{N}}_i$ first initializes its own version of solution $Z_{\mathcal{I}, j}^{(0)}$ to~\eqref{eq22}. Every agent has access to the solutions of its neighboring agents, and the central agent $i$ can access the solutions of agents in $\mathcal{N}_i$. A consensus of $Z_{\mathcal{I}}$ in \eqref{eq22} can be reached among $Z_{\mathcal{I}, j\in\mathcal{\bar{N}}_i}$, when implementing the following synchronous distributed algorithm on each computational agent in $\mathcal{\bar{N}}_i$~\cite{Mou_TAC_2015}:
	\begin{equation}\label{dist_alge}
		Z^{(n+1)}_{\mathcal{I}, j} = Z^{(n)}_{\mathcal{I}, j} - P_j\left( Z^{(n)}_{\mathcal{I}, j} - \dfrac{1}{d_j} \sum_{k \in \mathcal{N}_j \cap \mathcal{\bar{N}}_i } Z^{(n)}_{\mathcal{I}, k} \right),
	\end{equation}
	where $P_j$ is the orthogonal projection matrix on the kernel of $[I- \Theta]_j$, rows of the matrix $I-\Theta$ stored in agent $j$; $d_j = |\mathcal{N}_j \cap \mathcal{\bar{N}}_i|$ is the amount of neighboring agents of agent $j$ in subsystem $\mathcal{\bar{N}}_i$; and $n$ is the index of update iteration. An asynchronous distributed algorithm~\cite{Liu_TAC_2018}, which does not require agents to concurrently update their solution $Z_{\mathcal{I}, i}$, can also be invoked to solve~\eqref{eq22}. Under these distributed algorithms, different versions of solution $Z_{\mathcal{I}, j}$ are exchanged across the network, and the requirements of data storage and computation can be allocated evenly to each agent in network, which improve the overall efficiency of algorithms. Meanwhile, these fully parallelized planning algorithms can be optimized and boosted further by naturally incorporating some parallel computing and Internet of Things (IoT) techniques, such as edge computing~\cite{Shi_IoT_2016}. Nonetheless, for MASs with massive population, \textit{i.e.} $N \rightarrow \infty$, most of the control schemes and algorithms introduced in this paper will become fruitless, and we may resort to the mean-field theory~\cite{Bensoussan_2013, Bakshi_arxiv_2018, Bakshi_TCNS_2019}, which describes MASs by probability density model rather than connected graph model.

	\subsubsection*{C. Local Control Action}	
	After we figure out the desirability functions $Z_i(\bar{x}_i)$ for all the joint states $\bar{x}_i \in \bar{\mathcal{I}}_i \cup \bar{\mathcal{B}}_i$ in subsystem $\bar{\mathcal{N}}_i$, the local optimal control distribution $u_i^*(x'_i|\bar{x}_i)$ for central agent $i$ is derived by calculating the marginal distribution of $\bar{u}^*_i(\bar{x}'_i|\bar{x}_i)$
	\begin{equation*}
		u_i^*(x'_i|\bar{x}_i) = \sum_{j \in \mathcal{N}_i} \bar{u}_i^{*}(x'_i, x'_{j\in\mathcal{N}_i} | \bar{x}_i),
	\end{equation*}
	where both the joint distribution $\bar{u}^*_i(\bar{x}'_i|\bar{x}_i)$ and local distribution $u_i^*(x'_i|\bar{x}_i)$ rely on the local observation of central agent $i$. By sampling control action from marginal distribution $u_i^*(x'_i|\bar{x}_i)$, agent $i$ behaves optimally to minimize the joint cost-to-go function~\eqref{J_CTG_Discrete_Time} defined in subsystem $\bar{\mathcal{N}}_i$, and the local optimal control distribution $u_k^*(x'_k|\bar{x}_k)$ of other agents $k \in \mathcal{D} \backslash \{i\}$ in network $\mathcal{G}$ can be derived by repeating the preceding procedures in subsystems $\bar{\mathcal{N}}_{k}$. The procedures of multi-agent LSMDP algorithm are summarized as \hyperref[alg1]{Algorithm~1} in~\hyperref[appE]{Appendix~E}.

	\subsection{Continuous-Time Systems}\label{Sec3_2}
	We now consider the LSOC for continuous-time MASs subject to joint dynamics~\eqref{eq13} and joint immediate cost function~\eqref{cont_cost}, which can be formulated into the joint optimality equation~\eqref{ContBellman} as follows 
	\begin{equation*}
		V_i(\bar{x}_i, t) = \min_{\bar{u}_i} \mathbb{E}^{\bar{u}_i}_{\bar{x}_i^t, t} \left[ \phi_i(\bar{x}_{i}^{t_f}) + \int_{t}^{t_f} c_i(\bar{x}_i(\tau), \bar{u}_i(\tau)) \ d\tau  \right].
	\end{equation*}
	In order to solve~\eqref{ContBellman} and derive the local optimal control action $u^*_i(\bar{x}_i, t)$ for agent $i$, we first cast the joint optimality equation~\eqref{ContBellman} into a joint stochastic HJB equation that gives an analytic form for joint optimal control action $\bar{u}^*_i(\bar{x}_i, t)$. To solve for the value function, by resorting to the Cole-Hopf transformation, the stochastic HJB equation is linearized into a  partial differential equation (PDE) with respect to desirability function. Feynman-Kac formula is then invoked to formulate the solution of the linearized PDE and joint optimal control action as the path integral formulae forward in time, which are later approximated respectively by a distributed Monte Carlo (MC) sampling method and a sample-efficient distributed REPS algorithm.

	\subsubsection*{A. Linearization of Joint Optimality Equation}
	Similar to the transformation~\eqref{ExpTrans} for discrete-time systems, we adopt the following Cole-Hopf or exponential transformation in continuous-time systems
	\begin{equation}\label{ContTrans}
		Z(\bar{x}_i, t) = \exp[-V_i(\bar{x}_i, t) / \lambda_i],
	\end{equation}
	where  $\lambda_i \in \mathbb{R}$ is a  scalar, and $Z(\bar{x}_i, t)$ is the desirability function of joint state $\bar{x}_i$ at time $t$. Conversely, we also have $V_i(\bar{x}_i, t) = \lambda_i  \log Z(\bar{x}_i, t)$ from~\eqref{ContTrans}. In the following theorem, we convert the optimality equation~\eqref{ContBellman} into a joint stochastic HJB equation, which reveals an analytic form of joint optimal control action $\bar{u}^*_i(\bar{x}_i, t)$, and linearize the HJB equation into a linear PDE that has a closed-form solution for the desirability function.
	
	\begin{theorem}\label{thm2}
		Subject to the joint dynamics~\eqref{eq13} and immediate cost function~\eqref{cont_cost}, the joint optimality equation~\eqref{ContBellman} in subsystem $\mathcal{\bar{N}}_i$ is equivalent to the joint stochastic HJB equation
		\begin{align}\label{sto_HJB}
			-\partial_t V_i(\bar{x}_i, t) =&  \min_{\bar{u}_i}  \mathbb{E}_{\bar{x}_i,t}^{\bar{u}_i} \bigg[  \sum_{j \in \bar{\mathcal{N}}_i} [f_j(x_j,t)  + B_j(x_j) u_j(\bar{x}_i, t)]^\top \cdot \nabla_{x_j} V_i(\bar{x}_i,t)  +  q_i(\bar{x}_i, t)   \\
			& + \frac{1}{2} \bar{u}_i(\bar{x}_i, t)^\top\bar{R}_i\bar{u}_i(\bar{x}_i, t)  + \frac{1}{2} \sum_{j \in \bar{\mathcal{N}}_i}  \textrm{tr} \left( B_j(x_j)\sigma_j \sigma_j^\top B_j(x_j)^\top \cdot \nabla_{x_jx_j}V_i(\bar{x}_i, t) \right) \bigg], \nonumber
		\end{align}
		with boundary condition $V_i(\bar{x}_i, t_f) = \phi_i(\bar{x}_i)$, and the optimum of~\eqref{sto_HJB} can be attained with the joint optimal control action
		\begin{equation}\label{OptimalAction}
			\bar{u}^*_i(\bar{x}_i, t) = -\bar{R}_i^{-1}\bar{B}_i(\bar{x}_i)^\top \nabla_{\bar{x}_i}V_i(\bar{x}_i, t).
		\end{equation}
		Subject to the transformation~\eqref{ContTrans}, control action~\eqref{OptimalAction} and condition $\bar{R}_i = (\bar{\sigma}_i\bar{\sigma}_i^\top  /  \lambda_i)^{-1}$, the joint stochastic HJB equation~\eqref{sto_HJB} can be linearized as 
		\begin{equation}\label{Z_Function}
			\partial_{t} Z_i(\bar{x}_i, t) = \bigg[\frac{q_i(\bar{x}_i, t)}{\lambda_i} - \sum_{j\in\bar{\mathcal{N}}_i} f_j(x_j, t)^\top  \nabla_{x_j} - \frac{1}{2}\sum_{j \in \bar{\mathcal{N}}_i} \textrm{tr}\left( B_j(x_j)\sigma_j \sigma_j^\top B_j(x_j)^\top  \nabla_{x_jx_j} \right)  \bigg]Z_i(\bar{x}_i,t)
		\end{equation}
		with boundary condition $Z_i(\bar{x}_i, t_f) = \exp[- \phi_i(\bar{x}_i) / \lambda_i]$, which has a solution
		\begin{equation}\label{Z_Solution}
			Z_i(\bar{x}_i, t) = \mathbb{E}_{\bar{x}_i,t}\left[ \exp\left( -\frac{1}{\lambda_i} \phi_i(\bar{y}^{t_f}_i) -\frac{1}{\lambda_i} \int_{t}^{t_f}q_i(\bar{y}_i, \tau) \ d\tau \right) \right],
		\end{equation}
		where the diffusion process $\bar{y}(t)$ satisfies the uncontrolled dynamics $d \bar{y}_i(\tau) = \bar{f}_i(\bar{y}_i,\tau) d\tau + \bar{B}_i(\bar{y}_i)   \bar{\sigma}_i  \cdot d\bar{w}_i(\tau)$ with initial condition $\bar{y}_i(t) = \bar{x}_i(t)$.
	\end{theorem}
	\begin{proof}
		See \hyperref[appB]{Appendix~B} for the proof.
	\end{proof}
	
	Based on the transformation~\eqref{ContTrans}, the gradient of value function satisfies $\nabla_{\bar{x}_i} V_i(\bar{x}_i, t) = - \lambda_i \cdot \nabla_{\bar{x}_i} Z_i(\bar{x}_i,t) / Z_i(\bar{x}_i,t)$. Hence, the joint optimal control action~\eqref{OptimalAction} can be rewritten as
	\begin{equation}\label{OptimalAction2}
		\bar{u}^*_i(\bar{x}_i, t) = \lambda_i \bar{R}_i^{-1}\bar{B}^\top_i(\bar{x}_i) \cdot \frac{\nabla_{\bar{x}_i} Z_i(\bar{x}_i,t)}{Z_i(\bar{x}_i,t)}  = \bar{\sigma}_i\bar{\sigma}_i^\top \bar{B}_i(\bar{x}_i)^\top \cdot \frac{\nabla_{\bar{x}_i} Z_i(\bar{x}_i,t)}{Z_i(\bar{x}_i,t)},
	\end{equation}
	where the second equality follows from the condition $\bar{R}_i = (\bar{\sigma}_i\bar{\sigma}_i^\top / \lambda_i)^{-1}$. Meanwhile, by virtue of Feynman-Kac formula, the stochastic HJB equation~\eqref{sto_HJB} that must be solved backward in time can now be solved by an expectation of diffusion process evolving forward in time. While a closed-form solution of the desirability function $Z_i(\bar{x}_i,t)$ is given in~\eqref{Z_Solution}, the expectation $\mathbb{E}_{\bar{x}_i, t}(\cdot)$ is defined on the sample space consisting of all possible uncontrolled trajectories initialized at $(\bar{x}_i, t)$, which makes this expectation intractable to compute. A common approach in statistical physics and quantum mechanics is to first formulate this expectation as a path integral~\cite{Moral_2004, Kappen_PRL_2005, Theodorou_Entropy_2015}, and then approximate the integral or optimal control action with various techniques, such as MC sampling~\cite{Kappen_PRL_2005} and policy improvement~\cite{Theodorou_JMLR_2010}. In the following subsections, we first formulate the desirability function~\eqref{Z_Function} and control action~\eqref{OptimalAction2} as path integrals and then approximate them with a distributed MC sampling algorithm and a sample-efficient distributed REPS algorithm, respectively.

	\subsubsection*{B. Path Integral Formulation}
	Before we show a path integral formula for the desirability function $Z_i(\bar{x}_i, t)$ in \eqref{Z_Solution}, some manipulations on joint dynamics~\eqref{eq13} are performed to avoid singularity problems \cite{Theodorou_JMLR_2010}. By rearranging the components of joint states $\bar{x}_i$ in \eqref{eq13}, the joint state vector $\bar{x}_i$ in subsystem $\mathcal{\bar{N}}_i$ can be partitioned as $[\bar{x}^\top_{i(n)}, \bar{x}^\top_{i(d)}]^\top$, where $\bar{x}_{i(n)} \in \mathbb{R}^{U\cdot |\bar{\mathcal{N}}_i|}$ and $\bar{x}_{i(d)} \in \bar{\mathbb{R}}^{D \cdot |\bar{\mathcal{N}}_i|}$ respectively indicate the joint non-directly actuated states and joint directly actuated states of subsystem $\mathcal{\bar{N}}_i$; $U$ and $D$ denote the dimensions of non-directly actuated states and directly actuated states for a single agent. Consequently, the joint passive dynamics term $\bar{f}_i(\bar{x}_i, t)$ and the joint control transition matrix $\bar{B}_i(\bar{x}_i)$ in \eqref{eq13} are partitioned as $[\bar{f}^{ \tiny\ \top}_{i(n)}, \bar{f}^{ \tiny\ \top}_{i(d)}]^\top$ and $[0, \bar{B}^\top_{i(d)}(\bar{x}_i)]^\top$, respectively. Hence, the joint dynamics~\eqref{eq13} can be rewritten in a partitioned vector form as follows
	\begin{equation}\label{Partitioned_dynamics}
		\left(\begin{matrix}
			d\bar{x}_{i(n)}\\
			d\bar{x}_{i(d)}
		\end{matrix}\right) = \left(\begin{matrix}
			\bar{f}_{i(n)}(\bar{x}_i,t)\\
			\bar{f}_{i(d)}(\bar{x}_i,t)
		\end{matrix}\right)dt + \left(\begin{matrix}
			0\\
			\bar{B}_{i(d)}(\bar{x}_i)
		\end{matrix}\right)\left[ \bar{u}_i(\bar{x}_i,t) dt + \bar{\sigma}_i d\bar{w}_i  \right].
	\end{equation}
	With the partitioned dynamics~\eqref{Partitioned_dynamics}, the path integral formulae for the desirability function~\eqref{Z_Solution} and joint optimal control action~\eqref{OptimalAction2} are given in~\hyperref[prop3]{Proposition~3}.
	
	\begin{proposition}\label{prop3} Partition the time interval from $t$ to $t_f$ into $K$ intervals of equal length $\varepsilon > 0$, $t = t_0 < t_1 < \cdots < t_K = t_f$, and let the trajectory variable $\bar{x}_i^{(k)} = [\bar{x}_{i(n)}^{(k)\top}, \bar{x}_{i(d)}^{(k)\top}]^\top$ denote the segments of joint uncontrolled trajectories on time interval $[t_{k-1}, t_k)$, governed by joint dynamics~\eqref{Partitioned_dynamics} with $\bar{u}_i(\bar{x}_i,t) = 0$ and initial condition $\bar{x}_i(t) = \bar{x}_i^{(0)}$. The desirability function~\eqref{Z_Solution} in subsystem $\bar{\mathcal{N}}_i$ can then be reformulated as a path integral
		\begin{equation}\label{Prop3E1}
			Z_i(\bar{x}_i,t) = \lim_{\varepsilon \downarrow 0}  \int   \exp\left( -\tilde{S}_i^{\varepsilon, \lambda_i}(\bar{x}_i^{(0)}, \bar{\ell}_i, t_0) - K D |\mathcal{\bar N}_i| / 2 \cdot \log (2\pi\lambda_i \varepsilon) \right)  d\bar{\ell}_i,
		\end{equation}		
		where the integral is over path variable $\bar \ell_i = (\bar{x}^{(1)}_i, \cdots, \bar{x}^{(K)}_i)$, \textit{i.e.} set of all uncontrolled trajectories initialized at $(\bar{x}_i, t)$, and the generalized path value 
		\begin{align}\label{Prop3E2}
			\tilde{S}_i^{\varepsilon, \lambda_i}(\bar{x}_i^{(0)}, \bar{\ell}_i, t_0) = \frac{\phi_i(\bar{x}^{(K)}_i)}{\lambda_i} + \frac{\varepsilon}{\lambda_i} & \sum_{k=0}^{K-1}q_i(\bar{x}^{(k)}_i, t_k)   + \frac{1}{2}\sum_{k=0}^{K-1} \log \det(H_i^{(k)}) \allowdisplaybreaks  \\
			&+ \frac{\varepsilon}{2\lambda_i}  \sum_{k=0}^{K-1} \left\|  \frac{ \bar{x}_{i(d)}^{(k+1)} - \bar{x}_{i(d)}^{(k)}}{\varepsilon} -  \bar{f}_{i(d)}(\bar{x}^{(k)}_i, t_k) \right\|^2_{\left(H_i^{(k)}\right)^{-1}}   \nonumber
		\end{align}	
		with $H_i^{(k)} = \bar{B}_{i(d)}(\bar{x}^{(k)}_i)  \bar{\sigma}_i\bar{\sigma}_i^\top \bar{B}_{i(d)}(\bar{x}^{(k)}_i)^\top = \lambda_i\bar{B}_{i(d)}(\bar{x}^{(k)}_i)\bar{R}_i^{-1} \bar{B}_{i(d)}(\bar{x}^{(k)}_i)^\top$. Hence, the joint optimal control action in subsystem $\bar{\mathcal{N}}_i$ can be reformulated as a path integral
		\begin{equation}\label{OptCtrlPath}
			\bar{u}^*_i(\bar{x}_i, t) = \lambda_i \bar{R}_i^{-1} \bar{B}_{i(d)}(\bar{x}_i)^\top \cdot \lim_{\varepsilon \downarrow 0} \int  \tilde{p}^*_i(\bar{\ell}_i | \bar{x}_i^{(0)}, t_0) \cdot \tilde{u}_i(\bar{x}_i^{(0)}, \bar{\ell}_i, t_0) \ d\bar{\ell}_i,
		\end{equation}			
		where 
		\begin{equation}\label{OptPathDist}
			\tilde{p}^*_i(\bar{\ell}_i | \bar{x}_i^{(0)}, t_0) = \frac{\exp( -\tilde{S}_i^{\varepsilon, \lambda_i}(\bar{x}_i^{(0)}, \bar{\ell}_i, t_0) )}{\int \exp( -\tilde{S}_i^{\varepsilon, \lambda_i}(\bar{x}_i^{(0)}, \bar{\ell}_i, t_0) ) \ d\bar{\ell}_i }
		\end{equation}
		is the optimal path distribution, and
		\begin{equation}\label{inictrl}
			\tilde{u}_i(\bar{x}_i^{(0)}, \bar{\ell}_i, t_0)  = -\frac{\varepsilon}{\lambda_i}\nabla_{\bar{x}^{(0)}_{i(d)}}q_i(\bar{x}^{(0)}_i, t_0) + \left(H_i^{(0)}\right)^{-1} \left(\frac{\bar{x}_{i(d)}^{(1)} - \bar{x}_{i(d)}^{(0)}}{\varepsilon} - \bar{f}_{i(d)}(\bar{x}_i^{(0)}, t_0)  \right)
		\end{equation}
		is the initial control variable.
	\end{proposition}
	\begin{proof}
		See \hyperref[appC]{Appendix~C} for the proof.
	\end{proof}
	
	While \hyperref[prop3]{Proposition~3} gives the path integral formulae for desirability (value) function $Z_i(\bar{x}_i,t) $ and joint optimal control action $\bar{u}^*_i(\bar{x}_i, t)$, one of the challenges when implementing these formulae is the approximation of optimal path distribution~\eqref{OptPathDist} and integrals in~\eqref{Prop3E1} and~\eqref{OptCtrlPath}, since these integrals are defined on the set of all possible uncontrolled trajectories initialized at $(\bar{x}_i, t)$ or $(\bar{x}_i^{(0)}, t_0)$, which are intractable to exhaust and might become computationally expensive to sample as the state dimension and the amount of agents scale up. An intuitive solution is to formulate this problem as a statistical inference problem and predict the optimal path distribution $\tilde{p}^*_i(\bar{\ell}_i | \bar{x}_i^{(0)}, t_0)$ via various inference techniques, such as the easiest MC sampling~\cite{Kappen_PRL_2005} or Metropolis-Hastings sampling~\cite{Broek_JAIR_2008}. Provided a batch of uncontrolled trajectories $\mathcal{Y}_i = \{ (\bar{x}_i^{(0)}, \bar{\ell}_i^{[y]})  \}_{y = 1, \cdots, Y}$ and a fixed $\varepsilon$, the Monte Carlo estimation of optimal path distribution in~\eqref{OptPathDist} is
	\begin{equation}\label{MC_Estimator}
		\tilde{p}^*_i(\bar{\ell}_i^{[y]} | \bar{x}_i^{(0)}, t_0) \approx \frac{\exp( -\tilde{S}_i^{\varepsilon, \lambda_i}(\bar{x}_i^{(0)}, \bar{\ell}_i^{[y]}, t_0) )}{\sum_{y=1}^{Y} \exp( -\tilde{S}_i^{\varepsilon, \lambda_i}(\bar{x}_i^{(0)}, \bar{\ell}^{[y]}_i, t_0) ) },
	\end{equation}
	where $\tilde{S}_i^{\varepsilon, \lambda_i}(\bar{x}_i^{(0)}, \bar{\ell}_i^{[y]}, t_0)$ denotes the generalized path value~\eqref{Prop3E2} of trajectory $(\bar{x}_i^{(0)}, \bar{\ell}_i^{[y]})$, and the estimation of joint optimal control action in~\eqref{OptCtrlPath} is
	\begin{equation}\label{MC_Estimato2}
		\bar{u}^*_i(\bar{x}_i, t) =  \lambda_i \bar{R}_i^{-1}  \bar{B}_{i(d)}(x_i)^\top \cdot \sum_{y=1}^{Y} \tilde{p}^*_i(\bar{\ell}^{[y]}_i | \bar{x}_i^{(0)}, t_0) \cdot \tilde{u}_i(\bar{x}_i^{(0)}, \bar{\ell}^{[y]}_i, t_0),
	\end{equation}		
	where $\tilde{u}_i(\bar{x}_i^{(0)}, \bar{\ell}^{[y]}_i, t_0)$ is the initial control of sample trajectory $(\bar{x}_i^{(0)}, \bar{\ell}_i^{[y]})$. To expedite the process of sampling, the sampling tasks of $\mathcal{Y}_{i\in \mathcal{D}}$ can be distributed to different agents in the network, \textit{i.e.} each agent $j$ in MAS $\mathcal{G}$ only samples its local uncontrolled trajectories $\{(x_j^{(0)}, \ell_j^{[y]})\}_{y = 1, \cdots, Y}$, and the optimal path distributions of each subsystem $\mathcal{\bar{N}}_i$ are approximated via~\eqref{MC_Estimator} after the central agent $i$ restores $\mathcal{Y}_i$ by collecting the samples from its neighbors $j \in \mathcal{\bar{N}}_i$. Meanwhile, the local trajectories $\{(x_j^{(0)}, \ell_j^{[y]})\}_{y = 1, \cdots, Y}$ of agent $j$ not only can be utilized by the local control algorithm of agent $j$ in subsystem $\bar{\mathcal{N}}_j$, but also can be used by the local control algorithms of the neighboring agents $k \in \mathcal{N}_j$ in subsystem $\bar{\mathcal{N}}_k$, and the parallel computation of GPUs can further facilitate the sampling processes by allowing each agent to concurrently generate multiple sample trajectories~\cite{Williams_JGCD_2017}. The continuous-time multi-agent LSOC algorithm based on estimator~\eqref{MC_Estimator} is summarized as \hyperref[alg2]{Algorithm~2} in~\hyperref[appE]{Appendix~E}.
	
	However, since the aforementioned auxiliary techniques are still proposed on the basis of pure sampling estimator~\eqref{MC_Estimator}, the total amount of samples required for a good approximation is not reduced, which will hinder the implementations of \hyperref[prop3]{Proposition~3} and estimator~\eqref{MC_Estimator} when the sample trajectories are expensive to generate. Various investigations have been conducted to mitigate this issue. For the sample-efficient approximation of desirability (value) function $Z_i(\bar{x}_i,t)$, parameterized desirability (value) function, Laplace approximation~\cite{Kappen_PRL_2005}, approximation based on mean-field assumption~\cite{Broek_JAIR_2008}, and a forward-backward framework based on Gaussian process can be used~\cite{Pan_NIPS_2015}. Meanwhile, it is more common in practice to only efficiently predict the optimal path distribution $\tilde{p}^*_i(\bar{\ell}_i | \bar{x}_i^{(0)}, t_0)$ and optimal control action $\bar{u}^*_i(\bar{x}_i, t)$, while omitting the value of desirability function. Numerous statistical inference algorithms, such as Bayesian inference and variational inference~\cite{Boutselis_arxiv_2018}, can be adopted to predict the optimal path distribution, and diversified policy improvement and policy search algorithms, such as PI$^2$~\cite{Theodorou_JMLR_2010} and REPS~\cite{Gomez_KDD_2014}, can be employed to update the parameterized optimal control policy.  Hence, in the following subsection, a sample-efficient distributed REPS algorithm is introduced to update the parameterized joint control policy for each subsystem and MAS.

	\subsubsection*{C. Relative Entropy Policy Search in MAS}
	Since the initial control $\tilde{u}_i(\bar{x}_i^{(0)}, \bar{\ell}_i, t_0)$ can be readily calculated from sample trajectories and network communication, we only need to determine the optimal path distribution $\tilde{p}^*_i(\bar{\ell}_i | \bar{x}_i^{(0)}, t_0)$ before evaluating the joint optimal control action from~\eqref{OptCtrlPath}. To approximate this distribution more efficiently, we extend the REPS algorithm~\cite{Peters_CAI_2010, Gomez_KDD_2014} for path integral control in MAS. REPS is a model-based algorithm with competitive convergence rate, whose objective is to search for a parametric policy $\pi^{(k)}_i(\bar{u}_i^{(k)} | \bar{x}_i^{(k)})$ that generates a path distribution $\tilde{p}^\pi_i(\bar{\ell}_i| \bar{x}_i^{(0)}, t_0)$ to approximate the optimal path distribution $\tilde{p}^*_i(\bar{\ell}_i | \bar{x}_i^{(0)}, t_0)$. Compared with reinforcement learning algorithms and model-free policy improvement or search algorithms, since REPS is built on the basis of model, the policy and trajectory generated from REPS can automatically satisfy the constraints of the model. Moreover, for the MAS problems, REPS algorithm allows to consider a more general scenario when the initial condition $\bar{x}_i(t) = \bar{x}^{(0)}_i$ is stochastic and satisfies a distribution $\mu_i(\bar{x}^{(0)}_i, t_0)$. REPS algorithm alternates between two steps: a) \textit{Learning step} that approximates the optimal path distribution from samples generated by the current or initial policy, \textit{i.e.} $\tilde{p}^\pi_i(\bar{\ell}_i| \bar{x}_i^{(0)}, t_0) \rightarrow \tilde{p}^*_i(\bar{\ell}_i | \bar{x}_i^{(0)}, t_0)$, and b) \textit{Updating step} that updates the parametric policy $\pi^{(k)}_i(\bar{u}_i^{(k)} | \bar{x}_i^{(k)})$ to reproduce the desired path distribution $\tilde{p}^\pi_i(\bar{\ell}_i| \bar{x}_i^{(0)}, t_0)$ generated in step a). The algorithm terminates when the policy and approximate path distribution converge. 
	
	During the learning step, we approximate the joint optimal distribution $\tilde{p}_i^*(\bar{x}^{(0)}_i, \bar{\ell}_i) = \tilde{p}_i^*(\bar{\ell}_i | \bar{x}_i^{(0)}, t_0) \cdot \mu_i(\bar{x}^{(0)}_i, t_0)$ by minimizing the relative entropy (KL-divergence) between an approximate distribution $\tilde{p}_i(\bar{x}^{(0)}_i, \bar{\ell}_i)$ and $\tilde{p}_i^*(\bar{x}^{(0)}_i, \bar{\ell}_i)$ subject to a few constraints and a batch of sample trajectories $\mathcal{Y}_i = \{ (\bar{x}_i^{(0)}, \allowbreak \bar{\ell}_i^{[y]})\}_{y = 1, \cdots, Y}$ generated by the current or initial policy related to old distribution $\tilde{q}_i(\bar{x}^{(0)}_i, \bar{\ell}_i)$ from prior iteration. Similar to the distributed MC sampling method, the computation task of sample trajectories $\mathcal{Y}_i$ can either be exclusively assigned to the central agent $i$ or distributed among the agents in subsystem $\mathcal{\bar N}_i$ and then collected by the central agent $i$. However, unlike the local uncontrolled trajectories $\{(x_i^{(0)}, \ell_i^{[y]})\}_{y = 1, \cdots, Y}$ in the MC sampling method, which can be interchangeably used by the control algorithms of all agents in subsystem $\bar{\mathcal{N}}_i$, since the trajectory sets $\mathcal{Y}_{i \in \mathcal{D}}$ of the REPS approach are generated subject to different policies $\pi^{(k)}_i(u^{(k)}_i | \bar{x}^{(k)}_i)$ for $i\in\mathcal{D}$, the sample data in $\mathcal{Y}_{i}$ can only be specifically used by the local REPS algorithm in subsystem $\bar{\mathcal{N}}_i$. We can then update the approximate path distribution $\tilde{p}_i(\bar{x}^{(0)}_i, \bar{\ell}_i)$ by solving the following optimization problem
	\begin{align}\label{LearningStep}
		\arg \max_{\tilde{p}_i(\bar{x}_i^{(0)}, \bar{\ell}_i)} & \int \tilde{p}_i(\bar{x}_i^{(0)}, \bar{\ell}_i) \left[   -\tilde{S}_i^{\varepsilon, \lambda_i}(\bar{x}_i^{(0)}, \bar{\ell}_i, t_0) - \log \tilde{p}_i(\bar{x}_i^{(0)}, \bar{\ell}_i)  \right] \ d\bar{x}_i^{(0)} d\bar{\ell}_i,  \allowdisplaybreaks \\
		\textrm{s.t.} & 	\int \tilde{p}_i(\bar{x}_i^{(0)}, \bar{\ell}_i) \log\frac{\tilde{p}_i(\bar{x}_i^{(0)}, \bar{\ell}_i)}{\tilde{q}_i(\bar{x}_i^{(0)}, \bar{\ell}_i)} \ d\bar{x}_i^{(0)} d\bar{\ell}_i  \leq \delta,  \nonumber \allowdisplaybreaks\\
		& 	\int \tilde{p}_i(\bar{x}_i^{(0)}, \bar{\ell}_i) \psi(\bar{x}_i^{(0)}) \ d\bar{x}_i^{(0)} d\bar{\ell}_i  = \hat{\psi}_i^{(0)},  \nonumber \allowdisplaybreaks \\
		& 	\int \tilde{p}_i(\bar{x}_i^{(0)}, \bar{\ell}_i) \ d\bar{x}_i^{(0)} d\bar{\ell}_i = 1,\nonumber
	\end{align}
	where $\delta > 0$ is a parameter confining the update rate of approximate path distribution; $\psi_i(\bar{x}_i^{(0)})$ is a state feature vector of the initial condition; and $\hat{\psi}^{(0)}_i$ is the expectation of state feature vector subject to initial distribution $\mu_i(\bar{x}^{(0)}_i)$. When the initial state $\bar{x}_i^{(0)}$ is deterministic, \eqref{LearningStep} degenerates to an optimization problem for path distribution $\tilde{p}_i(\bar{\ell}_i | \bar{x}_i^{(0)}, t_0)$, and the second constraint in~\eqref{LearningStep} can be neglected. The optimization problem~\eqref{LearningStep} can be solved analytically with the method of Lagrange multipliers, which gives
	\begin{equation}\label{distribution}
		\tilde{p}_i(\bar{x}_i^{(0)}, \bar{\ell}_i) = \exp\left( -\frac{1+\kappa + \eta}{1 + \kappa} \right) \cdot  \tilde{q}_i(\bar{x}_i^{(0)}, \bar{\ell}_i) ^{\frac{\kappa}{1 + \kappa}}  \cdot \exp\left( - \frac{\tilde{S}_i^{\varepsilon, \lambda_i}(\bar{x}_i^{(0)}, \bar{\ell}_i, t_0)  + \theta^\top \psi_{i}(\bar{x}_i^{(0)}) }{1 + \kappa} \right),
	\end{equation}
	where $\kappa$ and $\theta$ are the Lagrange multipliers that can be solved from the dual problem 
	\begin{equation}\label{DualProb}
		\arg \min_{\kappa, \theta} \ g(\kappa, \theta), \qquad \textrm{s.t. }\kappa > 0,
	\end{equation}
	with objective function
	\begin{align}\label{DualFun}
		g(\kappa, \theta) &= \kappa \delta + \theta^\top \hat\psi_i^{(0)}  +  (1 + \kappa) \cdot \log \int  \tilde{q}_i(\bar{x}_i^{(0)}, \bar{\ell}_i)^{\frac{\kappa}{1 + \kappa}} \\
		&\hspace{116pt} \times \exp\left( - \frac{\tilde{S}_i^{\varepsilon, \lambda_i}(\bar{x}_i^{(0)}, \bar{\ell}_i, t_0)  + \theta^\top \psi_{i}(\bar{x}_i^{(0)}) }{1 + \kappa} \right) \ d\bar{x}_i^{(0)} d\bar{\ell}_i; \nonumber
	\end{align}
	and the dual variable $\eta$ can then be determined by the constraint obtained by substituting~\eqref{DualProb} into the normalization (last) constraint in~\eqref{LearningStep}. The approximation of the dual function $g(\kappa, \theta)$ from sample trajectories $\mathcal{Y}_i$, calculation of the old (or initial) path distribution $\tilde{q}_i(\bar{\ell}_i, \bar{x}_i^{(0)})$ from the current (or initial) policy $\pi^{(k)}_i(\bar{u}_i^{(k)} | \bar{x}_i^{(k)})$, and a brief interpretation on the formulation of optimization problem~\eqref{LearningStep} are explained in \hyperref[appD]{Appendix~D}. 
	
	In the updating step of REPS algorithm, we update the policy $\pi_i^{(k)}$ such that the joint distribution $\tilde{p}_i^\pi(\bar{x}_i^{(k+1)} | \bar{x}_i^{(k)}) = \int p(\bar{x}_i^{(k+1)}|\bar{x}_i^{(k)}, \bar{u}_i^{(k)})  \cdot \pi_i^{(k)}(\bar{u}_i^{(k)} | \bar{x}_i^{(k)}) \ d\bar{u}^{(k)}_i $ generated by policy $\pi^{(k)}_i$ approximates the path distribution $\tilde{p}_i(\bar{\ell}_i | \bar{x}_i^{(0)}, t_0)$ generated in the learning step~\eqref{LearningStep} and ultimately converges to the optimal path distribution $\tilde{p}^{*}_i(\bar{\ell}_i | \bar{x}_i^{(0)}, t_0)$ in~\eqref{OptPathDist}. More explanations on distribution $\tilde{p}_i^\pi$ and old distribution $\tilde{q}_i$ can be found in \eqref{OldDist} of \hyperref[appD]{Appendix~D}. In order to provide a concrete example, we consider a set of parameterized time-dependent Gaussian policies that are linear in states, \textit{i.e.} $\pi_i^{(k)}(\bar{u}_i^{(k)} | \bar{x}_i^{(k)}, \hat{a}_i^{(k)}, \hat{b}_i^{(k)}, \hat\Sigma^{(k)}_{i} ) \sim \mathcal{N}(\bar{u}_i^{(k)}|\hat{a}_i^{(k)}\bar{x}_i^{(k)} + \hat{b}_i^{(k)}, \hat{\Sigma}_i^{(k)})$ at time step $t_k < t_f$, where $\chi_i^{(k)} = (\hat{a}_i^{(k)}, \hat{b}_i^{(k)}, \hat\Sigma_i^{(k)})$ are the policy parameters to be updated. For simplicity, one can also construct a stationary Gaussian policy $\hat{\pi}_i(\bar{u}_i | \bar{x}_i, \hat{a}_i, \hat{b}_i, \hat\Sigma_{i} ) \sim \mathcal{N}(\bar{u}_i|\hat{a}_i\bar{x}_i + \hat{b}_i, \hat{\Sigma}_i)$, which was employed in~\cite{Kupcsik_CAI_2013} and shares the same philosophy as the parameter average techniques discussed in~\cite{Broek_JAIR_2008, Theodorou_JMLR_2010}. The parameters $\chi_i^{*(k)}$ in policy $\pi_i^{(k)}$ can be updated by minimizing the relative entropy between $\tilde{p}_i(\bar{x}_i^{(0)}, \bar{\ell}_i)$ and $\tilde{p}^\pi_i(\bar{x}_i^{(0)}, \bar{\ell}_i)$, \textit{i.e.}
	\begin{equation}\label{eq37}
		\chi_i^{*(k)} =\arg {\textstyle \max_{\chi_i^{(k)}} } \ \int  \tilde{p}_i(\bar{x}_i^{(0)}, \bar{\ell}_i) \cdot \log {\pi}_i^{(k)}(\bar{u}_i^{*(k)} | \bar{x}_i^{(k)}, \chi_i^{(k)}) \ d\bar{x}_i^{(0)} d\bar{\ell}_i,
	\end{equation}
	where $\bar{u}_i^{*(k)}$ is the optimal control action. More detailed interpretation on policy update as well as approximation of~\eqref{eq37} from sample data $\mathcal{Y}_i$ can be found in \hyperref[appD]{Appendix~D}. When implementing the REPS algorithm introduced in this subsection, we initialize each iteration by generating a batch of sample trajectories $\mathcal{Y}_i$ from the current (or initial) policy, which was updated to restore the old approximate path distribution $\tilde{q}_i(\bar{x}_i^{(0)}, \bar{\ell}_i)$ in the prior iteration. With these sample trajectories, we update the path distribution through (\ref{LearningStep})-(\ref{DualFun}) and update the policy again through~\eqref{eq37} till convergence of the algorithm. This REPS-based continuous-time MAS control algorithm is also summarized as \hyperref[alg3]{Algorithm~3} in~\hyperref[appE]{Appendix~E}.

	\subsubsection*{D. Local Control Action}
	The preceding two distributed LSOC algorithms will return either a joint optimal control action $\bar{u}^*_i(\bar{x}_i, t)$ or a joint optimal control policy $\pi_i^{*}(\bar{u}^{*}_i(t)|\bar{x}_i(t), \chi_i^{*})$ for subsystem $\bar{\mathcal{N}}_i$ at joint state $\bar{x}_i(t)$. Similar to the treatment of distributed LSMDP in discrete-time MAS, only the central agent $i$ selects or samples its local control action $u^*_i(\bar{x}_i, t)$ from $\bar{u}^*_i(\bar{x}_i, t)$ or $\pi_i^{*}(\bar{u}^{*}_i(t)|\bar{x}_i(t), \chi_i^{*})$, while other agents $j \in \mathcal{D} \backslash \{i\}$ are guided by the control law, $\bar{u}^*_j(\bar{x}_j, t)$ or $\pi_j^{*}(\bar{u}^{*}_j(t)|\bar{x}_j(t), \chi_j^{*})$, computed in their own subsystems $\bar{\mathcal{N}}_j$. 
	
	For distributed LSOC based on the MC sampling estimator~\eqref{MC_Estimato2}, the local optimal control action $u^*_i(\bar{x}_i, t)$ can be directly selected from the joint control action $\bar{u}^*_i(\bar{x}_i, t)$. For distributed LSOC based on the REPS method, the recursive algorithm generates control policy $\pi^{(k)}_i(\bar{u}^{(k)}_i |  \bar{x}_i^{(k)}, \chi_i^{(k)}) \sim \mathcal{N}(\bar{u}_i^{(k)}|\hat{a}_i^{(k)}\bar{x}_i^{(k)} + \hat{b}_i^{(k)}, \hat{\Sigma}_i^{(k)})$ for each time step $t = t_0 < t_1 < \cdots < t_K = t_f$ per iteration. When generating the trajectory set $\mathcal{Y}_i$ from the current control policy  $\pi_i^{(k)}(\bar{u}^{(k)}_i|\bar{x}_i^{(k)}, \chi_i^{(k)})$, we sample the local control action of agent $i$ from the marginal distribution $\pi_i^{(k)}(u^{(k)}_i|\bar{x}_i^{(k)}, \chi_i^{(k)}) = \sum_{j \in \mathcal{N}_i} \pi^{(k)}_i({u}^{(k)}_i, {u}^{(k)}_{j\in\mathcal{N}_i}  | \bar{x}_i^{(k)}, \chi_i^{(k)})$. After the convergence of local REPS algorithm in subsystem $\bar{\mathcal{N}}_i$, the local optimal control action $u_i^*(\bar{x}_i, t)$ of the central agent $i$ at state $x_i(t)$ can be sampled from the following marginal distribution
	\begin{equation}\label{Marginalize_Policy}
		\pi_i^{*(0)}(u^{*(0)}_i|\bar{x}_i^{(0)}, \chi_i^{*(0)}) = \sum_{j \in \mathcal{N}_i} \pi^{*(0)}_i({u}^{*(0)}_i, {u}^{*(0)}_{j\in\mathcal{N}_i}  | \bar{x}_i^{(0)}, \chi_i^{*(0)}),
	\end{equation}
	which minimizes the joint cost-to-go function~\eqref{J_CTG_Cont_Time} in subsystem $\mathcal{\bar{N}}_i$ and only relies on the local observation of agent $i$.

	\subsection{Generalization with Compositionality}\label{sec3.3}
	While several accessory techniques have been introduced in~\hyperref[Sec3_1]{Section~3.1} and \hyperref[Sec3_2]{Section~3.2} to facilitate the computation of optimal control law, we have to repeat all the aforementioned procedures when a new optimal control task with different terminal cost or preferred exit state is assigned. A possible approach to solve this problem and generalize LSOC is to resort to contextual learning, such as the contextual policy search algorithm~\cite{Kupcsik_CAI_2013}, which adapts to different terminal conditions by introducing hyper-parameters and an additional layer of learning algorithm. However, this supplementary learning algorithm will undoubtedly increase the complexity and computational task for the entire control scheme. Thanks to the linearity of LSOC problems, a task-optimal controller for a new (unlearned) task can also be constructed from existing (learned) controllers by utilizing the compositionality principle~\cite{Broek_JAIR_2008, Todorov_NIPS_2009, daSilva_TOG_2009, Pan_NIPS_2015}.  Suppose that we have $F$ previously learned (component) LSOC problems and a new (composite) LSOC problem with different terminal cost from the previous $F$ problems. Apart from different terminal costs, these $F+1$ multi-agent LSOC problems share the same communication network $\mathcal{G}$, state-related cost $q(\bar{x}_i)$, exit time $t_f$, and dynamics, which means that the state space, control space, and sets of interior and boundary states of these $F+1$ problems are also identical. Let $\phi_i^{\{f\}}(\bar{x}_i)$ with $f \in \{1, 2, \cdots, F\}$ be the terminal costs of $F$ component problems in subsystem $\bar{\mathcal{N}}_i$, and $\phi_i(\bar{x}_i)$ denotes the terminal cost of composite or new problem in subsystem $\bar{\mathcal{N}}_i$. We can efficiently construct a joint optimal control action $\bar{u}^*_i( \bar{x}'_i | \bar{x}_i)$ for the new task from the existing component controllers $\bar{u}^{*\{f\}}_i( \bar{x}'_i | \bar{x}_i)$ by the compositionality principle.

	For multi-agent LSMDP in discrete-time MAS, when there exists a set of weights $\bar{\omega}_i^{\{f\}}$ such that
	\begin{equation*}
		\phi_i(\bar{x}_i) = -\log \bigg[  \sum_{f=1}^{F} \bar{\omega}_i^{\{f\}} \cdot \exp\Big(  -\phi_i^{ \{f\} } (\bar{x}_i) \Big) \bigg],
	\end{equation*}
	by the definition of discrete-time desirability function~\eqref{ExpTrans}, we can imply that 
	\begin{equation}\label{Dis_Des_Composite}
		Z_i(\bar{x}_i) = \sum_{f = 1}^{F} \bar{\omega}_i^{ \{f\} } \cdot Z_i^{\{f\}}(\bar{x}_i)
	\end{equation}
	for all $\bar{x}_i \in \bar{\mathcal{B}}_i$. Due to the linear relation in~\eqref{Z_update}, identity~\eqref{Dis_Des_Composite} should also hold for all interior states $\bar{x}_i \in \bar{\mathcal{I}}_i$. Substituting~\eqref{Dis_Des_Composite} into the optimal control action~\eqref{OptimalControl}, the task-optimal control action for the new task with terminal cost $\phi_i(\bar{x}_i)$ can be immediately generated from the existing controllers
	\begin{equation}\label{Dis_Comp}
		\bar{u}^*_i( \bar{x}'_i | \bar{x}_i) = \sum_{f = 1}^{F} \bar{W}_i^{\{f\}}(\bar{x}_i)   \cdot \bar{u}^{*\{f\}}_i( \bar{x}'_i | \bar{x}_i),
	\end{equation}
	where $\bar{W}_i^{\{f\}}(\bar{x}_i) = \bar{\omega}_i^{\{f\}}  \mathcal{W}^{\{f\}}_i(\bar{x}_i) / (\sum_{e=1}^{F} \bar{\omega}_i^{\{e\}}  \mathcal{W}_i^{\{e\}}(\bar{x}_i))$ and $\mathcal{W}_i^{\{f\}}(\bar{x}_i) = \sum_{\bar{x}'_i}\bar{p}_i(\bar{x}'_i|\bar{x}_i)Z^{\{f\}}_i(\bar{x}'_i)$. For LSOC in MAS, compositionality principle can be used for not only the generalization of controllers for the same subsystem $\bar{\mathcal{N}}_i$, but the generalization of control law across the network. For any two subsystems that satisfy the aforementioned compatible conditions, the task-optimal controller of one subsystem can also be directly constructed from the existing computational result of the other subsystem by resorting to~\eqref{Dis_Comp}.

	The generalization of continuous-time LSOC problem can be readily inferred by symmetry. For $F+1$ continuous-time LSOC problems that satisfy the compatible conditions, when there exist scalars $\lambda_i$, $\lambda_i^{\{f\}}$ and weights $\bar{\omega}_i^{\{f\}}$ such that
	\begin{equation}\label{eq42}
		\phi_i(\bar{x}_i) = -\lambda_i \log \bigg[  \sum_{f = 1}^{F} \bar{\omega}_i^{\{f\}} \cdot \exp \Big( -\frac{1}{\lambda_i^{\{ f \}}} \cdot \phi_i^{\{f\}} (\bar{x}_i)  \Big) \bigg],
	\end{equation}
	by the definition of continuous-time desirability function~\eqref{ContTrans}, we readily have $Z_i(\bar{x}_i, t_f) = \sum_{f = 1}^{F} \bar{\omega}_i^{ \{f\} } \cdot Z_i^{\{f\}}(\bar{x}_i, t_f)$ for all $\bar{x}_i \in \bar{\mathcal{B}}_i$. Since $Z^{\{f\}}_i(\bar{x}_i, \tau)$ is the solution to linearized stochastic HJB equation \eqref{Z_Function}, the linear combination of desirability function
	\begin{equation}\label{Con_Des_Composite}
		Z_i(\bar{x}_i, \tau) = \sum_{f = 1}^{F} \bar{\omega}_i^{ \{f\} } \cdot Z_i^{\{f\}}(\bar{x}_i, \tau)
	\end{equation}
	holds everywhere from $t$ to $t_f$ on condition that \eqref{Con_Des_Composite} holds for all terminal states $\bar{\mathcal{B}}_i$, which is guaranteed by~\eqref{eq42}. Substituting~\eqref{Con_Des_Composite} into the continuous-time optimal controller~\eqref{OptimalAction2}, the task-optimal composite controller $\bar{u}_i^*(\bar{x}_i)$ can be constructed from the component controllers $\bar{u}_i^{\{f\}*}(\bar{x}_i)$ by
	\begin{equation*}
		\bar{u}_i^*(\bar{x}_i, t) = \sum_{f = 1}^{F} \bar{W}_i^{\{f\}}(\bar{x}_i, t)   \cdot \bar{u}^{\{f\}*}_i( \bar{x}_i, t),
	\end{equation*}
	where $\bar{W}_i^{\{f\}}(\bar{x}_i, t) = \bar{\omega}_i^{\{f\}} Z_i^{\{f\}}(\bar{x}_i, t) / (\sum_{f=1}^{F} \bar{\omega}_i^{\{f\}} Z_i^{\{f\}}(\bar{x}_i, t))$. However, this generalization technique rigorously does not apply to the control policies based on the trajectory optimization or policy search methods, such as PI$^2$ \cite{Theodorou_JMLR_2010} and REPS in~\hyperref[Sec3_2]{Section~3.2}, since these policy approximation methods usually only predict the optimal path distribution or optimal control policy, while leaving the value or desirability function unknown.

	\section{Illustrative Examples}\label{sec4}

Distributed LSOC can be deployed on a variety of MASs in reality, such as the distributed HVAC system in smart buildings, cooperative unmanned aerial vehicle (UAV) teams, and synchronous wind farms. Motivated by the experiments in~\cite{Broek_JAIR_2008, Williams_JGCD_2017, Daniel_2017}, we demonstrate the distributed LSOC algorithms with a cooperative UAV team in cluttered environment. Compared with these preceding experiments, the distributed LSOC algorithms in this paper allow to consider an explicit and simpler communication network underlying the UAV team, and the joint cost function between neighboring agents can be constructed and optimized with less computation. Consider a cooperative UAV team consisting of three agents illustrated in~\hyperref[fig2]{Figure~2}. UAV~1 and 2, connected by a solid line, are strongly coupled via their joint cost functions $q_1(\bar{x}_1)$ and $q_2(\bar{x}_2)$, and their local controllers $u_1(\bar{x}_1)$ and $u_2(\bar{x}_2)$ are designed to drive UAV~1 and 2 towards their exit state while minimizing the distance between two UAVs and avoiding obstacles. These setups are useful when we need multiple UAVs to fly closely towards an identical destination, \textit{e.g.} carrying and delivering a heavy package together or maintaining communication channels/networks. By contrast, although UAV~3 is also connected with UAV~1 and 2 via dotted lines, the immediate cost function of UAV~3 is designed to be independent or fully factorized from the states of UAV~1 and 2 in each subsystem, such that UAV 3 is only loosely coupled with UAVs~1 and~2 through their terminal cost functions, which restores the scenario considered in~\cite{Broek_JAIR_2008, Williams_JGCD_2017}. The local controller of UAV~3 is then synthesized to guide the UAV to its exit state with minimal cost. In the following subsections, we will verify both discrete-time and continuous-time distributed LSOC algorithms subject to this cooperative UAV team scenario.

\vspace{1em}

\begin{figure}[H]
	\centering 
	\includegraphics[width=0.32\textwidth]{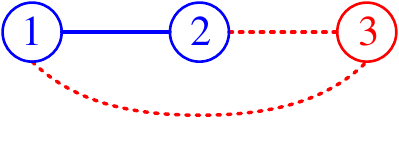}
	\caption{Communication network of UAV team. UAV 1 and 2 are strongly coupled through their running cost functions. UAV 3 is loosely coupled with UAV 1 and 2 through their terminal cost functions.}\label{fig2}
\end{figure}
\noindent 

\subsection{UAVs wtih Discrete-Time Dynamics}
In order to verify the distributed LSMDP algorithm, we consider a cooperative UAV team described by the probability model introduced in \hyperref[Sec2_2]{Section~2.2}. The flight environment is described by a $5\times5$ grid with total $25$ cells to locate the positions of UAVs, and the shaded cells in~\hyperref[fig3]{Figure~3(a)} represent the obstacles. Each UAV is described by a state vector $x_i = [r_i, c_i]^\top$, where $r_i$ and $c_i \in \{1,2,3,4,5\}$ are respectively the row and column indices locating the  $i^{\rm th}$ UAV. The cell with circled index \circled{$i$} denotes the initial state of the  $i^{\rm th}$ UAV, and the one with boxed index \squared{ \hspace{-2pt} $i$ \hspace{-2pt} } indicates the exit state of the  $i^{\rm th}$ UAV. The passive transition probabilities of interior, edge, and corner cells are shown in \hyperref[fig3]{Figure~3(b)}, \hyperref[fig3]{Figure~3(c)}, and \hyperref[fig3]{Figure~3(d)}, respectively, and the passive probability of a UAV transiting to an adjacent cell can be interpreted as the result of random winds. To fulfill the requirements of aforementioned scenario, the controlled transition distributions should i) drive UAV 1 and 2 to their terminal state $(5, 5)$ while shortening the distance between two UAVs, and ii) guide UAV 3 to terminal state $(1, 5)$ with minimum cost, which is related to obstacle avoidance, control cost, length of path, and flying time.

\vspace{1em}

\begin{figure}[H]
	\centering 
	\includegraphics[width=0.7\textwidth]{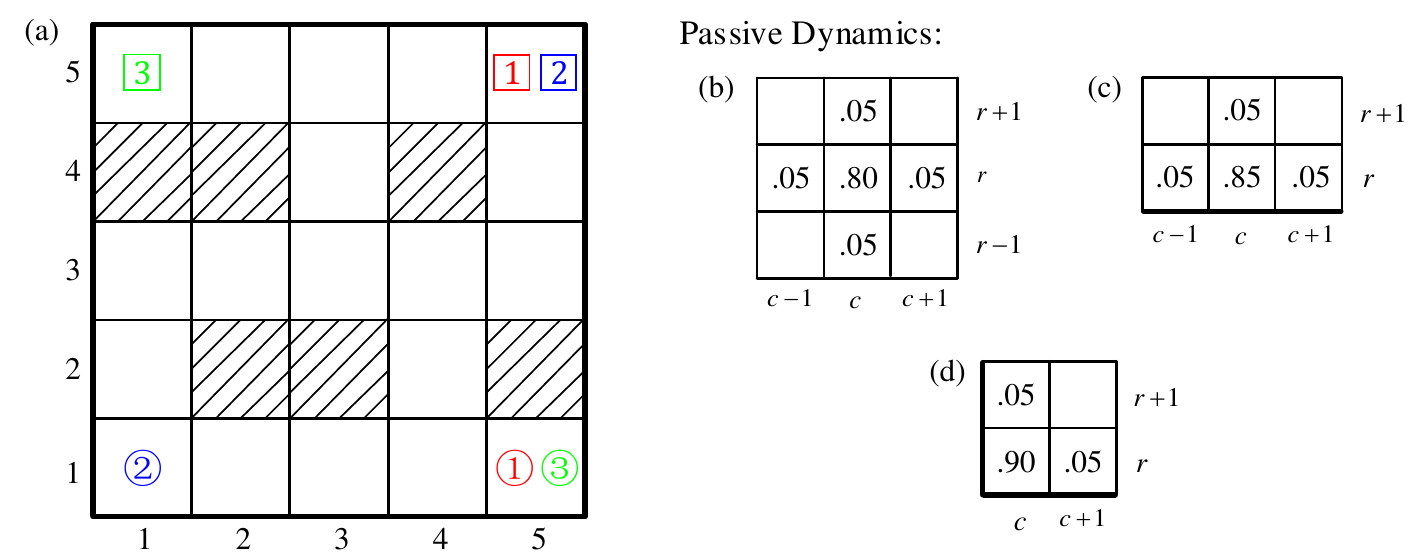}\label{fig3}
	\caption{ Flight environment and UAV's passive transition dynamics. (a) Flight environment, initial states, and exit states of UAVs. (b) The passive transition probability of UAVs in the interior cells. (c) The passive transition probability of UAVs in the edge cells. (d) The passive transition probability of UAVs in the corner cells.}
\end{figure}

\vspace{0.5em}

In order to realize these objectives, we consider the state-related cost functions as follows
\begin{align}\label{DisCostFun}
	q_1(\bar{x}_1) & = w_{12} \cdot (|r_1 - r_2| + |c_1 - c_2|) + o_1(x_1) \cdot o_2(x_2), \nonumber \allowdisplaybreaks\\
	q_2(\bar{x}_2) & = w_{21} \cdot (|r_2 - r_1| + |c_2 - c_1|) + o_1(x_1) \cdot o_2(x_2),\allowdisplaybreaks \\
	q_3(\bar{x}_3) & = o_3(x_3),\nonumber \allowdisplaybreaks
\end{align}
where in the following examples, weights on relative distance $\|x_1 - x_2\|_1$, defined by Manhattan distance, are $w_{12} = w_{21} = 3.5$; state and obstacle cost $o_i(x_i) = 30$ when the $i^{\rm th}$ UAV  is in an obstacle cell, and $o_i(x_i) = 2.2$ when UAV $i$ is in a regular cell; and terminal cost $q_i(\bar{x}_i) = 0$ for $\bar{x}_i \in \bar{\mathcal{B}}_i$. State-related cost $q_1(\bar{x}_1)$ and $q_2(\bar{x}_2)$, which involve both the states of UAV 1 and 2, can measure the joint performance between two UAVs, while cost $q_3(\bar{x}_3)$ only contains the state of UAV 3. Without the costs on relative distance, \textit{i.e.} $w_{12} = w_{21} = 0$, this distributed LSMDP problem will degenerate to three independent shortest path problems with obstacle avoidance as considered in \cite{Broek_JAIR_2008, Williams_JGCD_2017, Daniel_2017}, and the optimal trajectories are straightforward\footnote{The shortest path for UAV 1 is $(1, 5) \rightarrow (1, 4) \rightarrow (2, 4) \rightarrow (3, 4) \rightarrow (3, 5) \rightarrow (4, 5) \rightarrow (5, 5)$. There are three different shortest paths for UAV 2, and the shortest path for UAV 3 is shown in~\hyperref[fig4]{Figure~4 (c)}.}. Desirability function $Z_i(\bar{x}_i)$ and local optimal control distribution $u^*_i(\cdot | \bar{x}_i)$ can then be computed by following~\hyperref[alg1]{Algorithm~1} in~\hyperref[appE]{Appendix~E}.   \hyperref[fig4]{Figure~4} shows the maximum likelihood controlled trajectories of UAV team subject to passive dynamics in~\hyperref[fig3]{Figure~3} and cost functions~\eqref{DisCostFun}. Some curious readers may wonder why UAV~1 in \hyperref[fig4]{Figure~4(a)} decides to stay in $(1, 4)$ cell for two consecutive time steps rather than moving forward to $(1, 3)$ and then flying along with UAV~2, which generates a lower state-related cost. While this alternative corresponds to a lower state-related cost, it may not be the optimal trajectory minimizing the control cost and the overall immediate cost in~\eqref{eq4}. In order to verify this speculation, we increase the passive transition probabilities $p_i(\cdot | \bar{x}_i)$ to surpass certain thresholds in~\hyperref[fig5]{Figure~5(a)-(c)}, which can be interpreted as stronger winds in reality. With this altered passive dynamics and cost functions in~\eqref{DisCostFun}, the maximum likelihood controlled trajectory of UAV~1 is shown in~\hyperref[fig5]{Figure 5(d)}, which verifies our preceding reasoning. Trajectories of UAV 2 and 3 subject to altered passive dynamics are identical to the results in~\hyperref[fig4]{Figure~4}. To provide an intuitive view on the efficiency improvements of our distributed LSMDP algorithm, \hyperref[fig5.5]{Figure~6} presents the average data size and computational complexity on each UAV ($|\mathcal{S}_i| = 25$), gauged by the row number $m$ of matrices and vectors in~\eqref{Z_update} and~\eqref{dist_alge}, subject to the centralized programming~\eqref{Z_update} and parallel programming~\eqref{dist_alge} in different communication network, such as line, ring, complete binary tree, and fully connected topology, which restores the scenario with exponential complexity considered in \cite{Broek_JAIR_2008, Daniel_2017}.

\vspace{3em}

\begin{figure}[h]
	\centering 
	\includegraphics[width=1.0\textwidth]{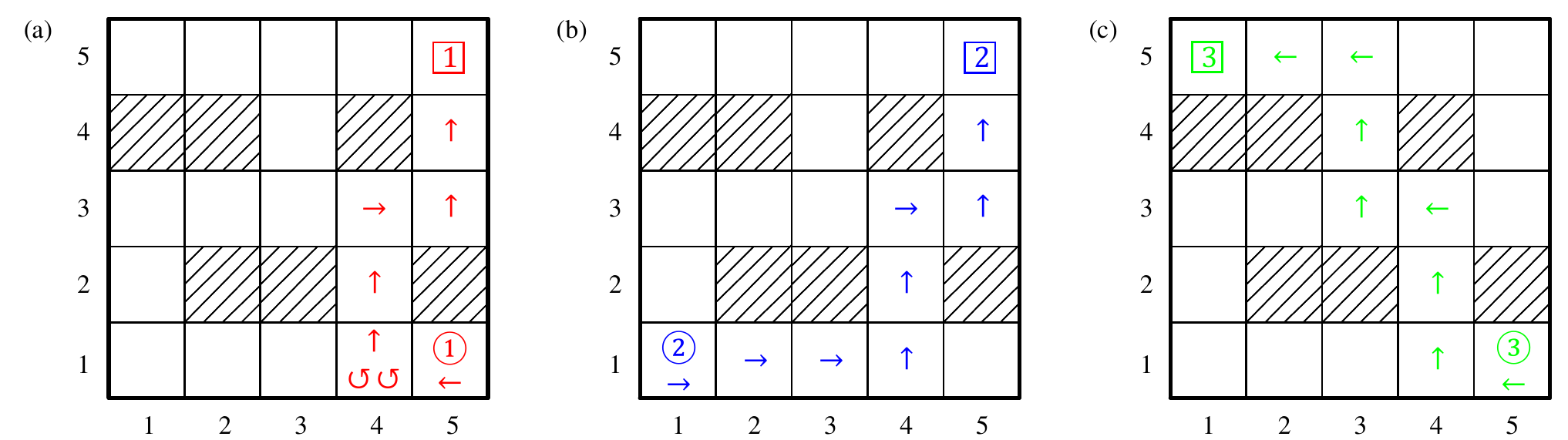}\label{fig4}
	\caption{UAVs' trajectories subject to optimal control policy. (a) Controlled trajectory of UAV~1: $(1, 5) \rightarrow (1, 4) \rightarrow (1, 4) \rightarrow (1, 4) \rightarrow (2, 4) \rightarrow (3, 4) \rightarrow (3, 5) \rightarrow (4, 5) \rightarrow (5, 5)$. (b) Controlled trajectory of UAV 2: $(1, 1) \rightarrow (1, 2) \rightarrow (1, 3) \rightarrow (1, 4) \rightarrow (2, 4) \rightarrow (3, 4) \rightarrow (3, 5) \rightarrow (4, 5) \rightarrow (5, 5)$. (c) Controlled trajectory of UAV~3: $(1, 5) \rightarrow (1, 4) \rightarrow (2, 4) \rightarrow (3, 4) \rightarrow (3, 3) \rightarrow (4, 3) \rightarrow (5, 3) \rightarrow (5, 2) \rightarrow (5, 1)$.}
\end{figure}

\clearpage

\begin{figure}[htpb]
	\centering 
	\includegraphics[width=0.77\textwidth]{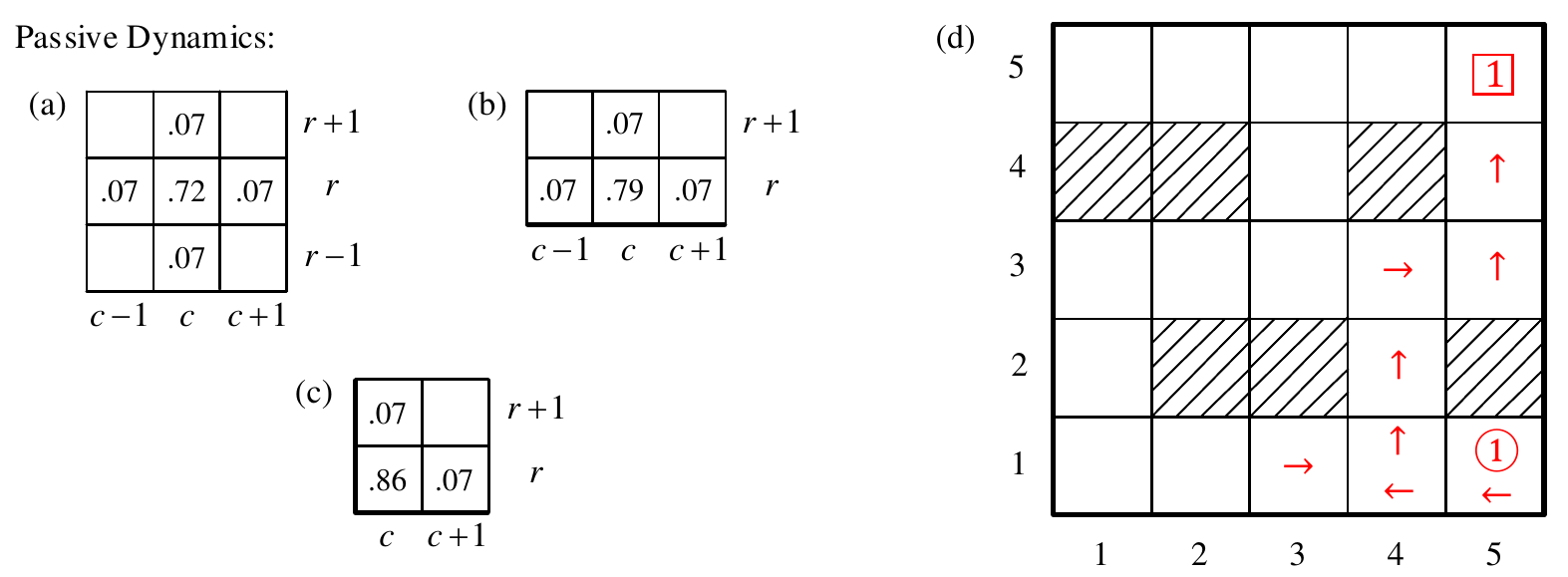}\label{fig5}
	\caption{Altered passive transition dynamics and controlled trajectory of UAV 1 subject to new dynamics. (a-c) are respectively the passive transition probabilities for interior, cell, and corner cells. (d) Controlled trajectory of UAV 1 subject to the altered passive dynamics: $(1, 5) \rightarrow (1, 4) \rightarrow (1, 3) \rightarrow (1, 4) \rightarrow (2, 4) \rightarrow (3, 4) \rightarrow (3, 5) \rightarrow (4, 5) \rightarrow (5, 5)$.}
\end{figure}

\begin{figure}[H]
	\centering 
	\includegraphics[width=0.45\textwidth]{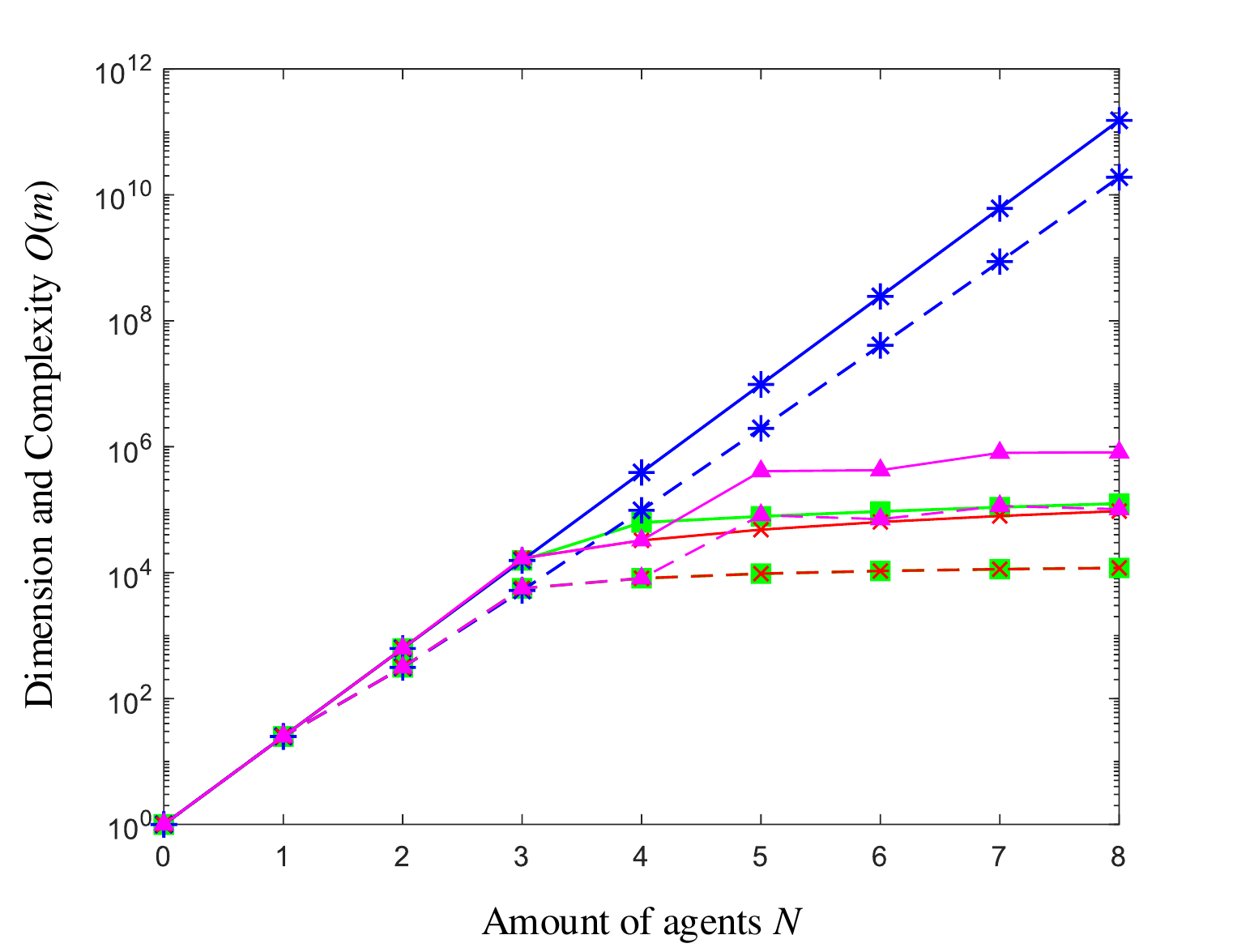}\label{fig5.5}
	\caption{Relationship between amount of agents $N$ and data dimension and computational complexity $O(m)$. Vertical axis is the value of $m$. Solid and dashed lines respectively represent the results attained by centralized algorithm~\eqref{Z_update} and parallel algorithm~\eqref{dist_alge}. Blue asterisk, red cross, green square, and magenta triangle respectively denote the results associated with fully connected, line, ring, and complete binary tree communication networks.}
\end{figure}

\vspace{1em}

\subsection{UAVs wtih Continuous-Time Dynamics}\label{sec4_1}
In order to verify the continuous-time distributed LSOC algorithms, consider the continuous-time UAV dynamics described by the following equation~\cite{Broek_JAIR_2008, Yao_AST_2019}:
\begin{equation}\label{Cont_Model}
	\left(\begin{matrix}
		dx_i\\
		dy_i\\
		dv_i\\
		d\varphi_i
	\end{matrix}\right) = 
	\left(\begin{matrix}
		v_i \cos \varphi_i\\
		v_i \sin \varphi_i\\
		0\\
		0
	\end{matrix}\right) dt + \left(\begin{matrix}
		0 & 0\\
		0 & 0\\
		1 & 0\\
		0 & 1
	\end{matrix}\right) \left[ \left( \begin{matrix}
		u_i\\
		\omega_i
	\end{matrix}  \right) dt + \left(\begin{matrix}
		\sigma_i & 0\\
		0 & \nu_i
	\end{matrix}\right)dw_i
	\right], 
\end{equation}
where $(x_i, y_i)$, $v_i$, and $\varphi_i$ respectively denote the position, forward velocity, and direction angle of the $i^{\rm th}$ UAV; forward acceleration $u_i$ and angular velocity $\omega_i$ are the control inputs, and disturbance $w_i$ is a standard Brownian motion. Control transition matrix $\bar{B}_i(x_i)$ is a constant matrix in~\eqref{Cont_Model}, and we set noise level parameters $\sigma_i = 0.1$ and $\nu_i = 0.05$ in simulation. The communication network underlying the UAV team as well as the control objectives are the same as in the discrete-time scenario. Instead of adopting the Manhattan distance, the distance in continuous-time problem is associated with $L_2$ norm. Hence, we consider the state-related costs defined as follows
\begin{equation}\label{UAV_RunningCost}
	\begin{split}
		q_1(\bar{x}_1)  & = w_{11} \cdot (\| (x_1, y_1) - (x_1^{t_f}, y_1^{t_f}) \|_2 - d^{\max}_1)  + w_{12} \cdot ( \| (x_1, y_1) - (x_2, y_2) \|_2 - d^{\max}_{12}),\\
		q_2(\bar{x}_2) & = w_{22} \cdot (\| (x_2, y_2) - (x_2^{t_f}, y_2^{t_f}) \|_2 - d^{\max}_2)  + w_{21} \cdot ( \| (x_2, y_2) - (x_1, y_1) \|_2 - d^{\max}_{21}),\\
		q_3(\bar{x}_3) & = w_{33}  \cdot (\| (x_3, y_3) - (x_3^{t_f}, y_3^{t_f}) \|_2 - d^{\max}_3),
	\end{split}
\end{equation}
where $w_{ii}$ is the weight on distance between the $i^{\rm th}$ UAV and its exit position; $w_{ij}$ is the weight on distance between the $i^{\rm th}$ and $j^{\rm th}$ UAVs; $d^{\max}_i$ usually denotes the initial distance between the $i^{\rm th}$ UAV and its destination, and $d_{ij}^{\max}$ denotes the initial distance between the $i^{\rm th}$ and $j^{\rm th}$ UAVs. The parameters $d^{\max}_i$ and $d_{ij}^{\max}$ are chosen to regularize the numerical accuracy and stability of algorithms.

To show an intuitive improvement brought by the joint state-related cost, we first verify the continuous-time LSOC algorithm in a simple flight environment without obstacle. Consider three UAVs forming the network in~\hyperref[fig2]{Figure~2} and with initial states $x_1^{0} = (5, 5, 0.5, 0)^\top$, $x_2^{0} = (5, 45, 0.5, 0)^\top$, $x_3^0 = (5, 25, 0.5, 0)^\top$ and identical exit state $x_i^{t_f} = (45, 25, 0, 0)^\top$ for $i = 1, 2, 3$. The exit time is $t_f = 25$, and the length of each control cycle is $0.2$. When sampling trajectory roll-outs $\mathcal{Y}_i$, the time interval from $t$ to $t_f$ is partitioned into $K = 7$ intervals of equal length $\varepsilon$, \textit{i.e.} $\varepsilon K  = t_f - t$, until $\varepsilon$ becomes less than 0.2. Meanwhile, to make the exploration process more aggressive, we increase the noise level parameters $\sigma_i = 0.75$ and $\nu_i = 0.65$ when sampling trajectory roll-outs. The size of data set $\mathcal{Y}_i$ for estimator~\eqref{MC_Estimator} in each control cycle is $400$ sample trajectories, which can be generated concurrently by GPU~\cite{Williams_JGCD_2017}. Control weight matrices $\bar{R}_i$ are selected as identity matrices. Guided by the sampling-based distributed LSOC algorithm, \hyperref[alg2]{Algorithm~2} in \hyperref[appE]{Appendix~E}, UAV trajectories and the relative distance between UAV 1 and 2 in condition of both joint and independent state-related costs are presented in~\hyperref[fig6]{Figure~7}. Letting update rate constraint be $\delta_i = 25$ and the size of trajectory set be $|\mathcal{Y}_i| = 400$ and $150$ for the initial and subsequent policy iterations in every control period, simulation results obtained from the REPS-based distributed LSOC algorithm, \hyperref[alg3]{Algorithm~3} in \hyperref[appE]{Appendix~E}, are given in~\hyperref[fig7]{Figure~8}. \hyperref[fig6]{Figure~7} and \hyperref[fig7]{Figure~8} imply that joint state-related costs  can significantly influence the trajectories and shorten the relative distance between UAV 1 and 2, which fulfills our preceding control requirements. 

\begin{figure}[H]
	\centering
	\vspace{-0.5em}\includegraphics[width=0.90\textwidth]{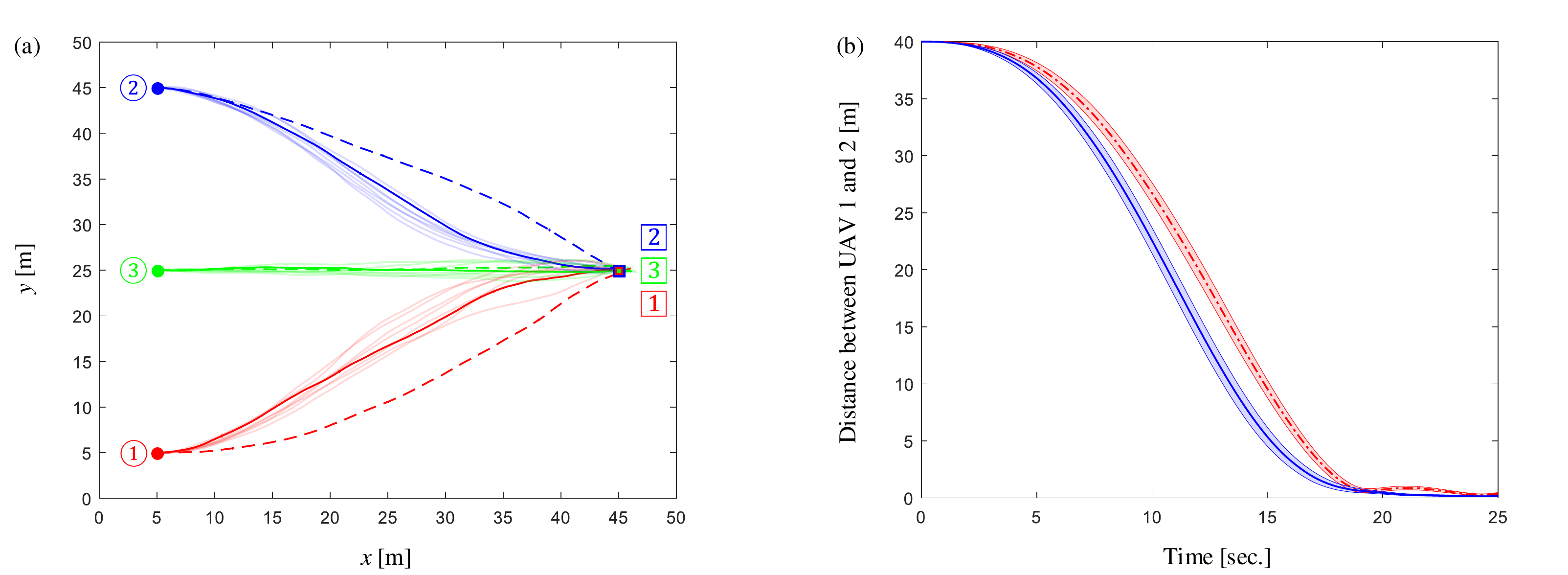}\vspace{-0em}\label{fig6}
	\caption{UAV trajectories and the distance between UAV 1 and 2 from 50 trails subject to the sampling-based distributed LSOC algorithm. (a) Trajectories of UAVs. Red, blue and green lines are trajectories for UAV 1, 2 and 3, respectively. Dashed lines are from a trail with factorized (or independent) state costs ($w_{11} = w_{22} = 0.75$, $w_{33} = 1$, $w_{12} = w_{21} = 0$), and solid (transparent) lines are from trails with joint state costs ($w_{11} = w_{22} = 0.75$, $w_{33} = 1$, $w_{12} = w_{21} = 1.5$). (b) Distance between UAV 1 and 2. Red dashed line and blue solid line are respectively the mean distances from tails with factorized state cost and joint state cost. Height of strip represents one standard deviation.}
\end{figure}
\begin{figure}[H]
	\centering
	\vspace{-0.5em}\includegraphics[width=0.90\textwidth]{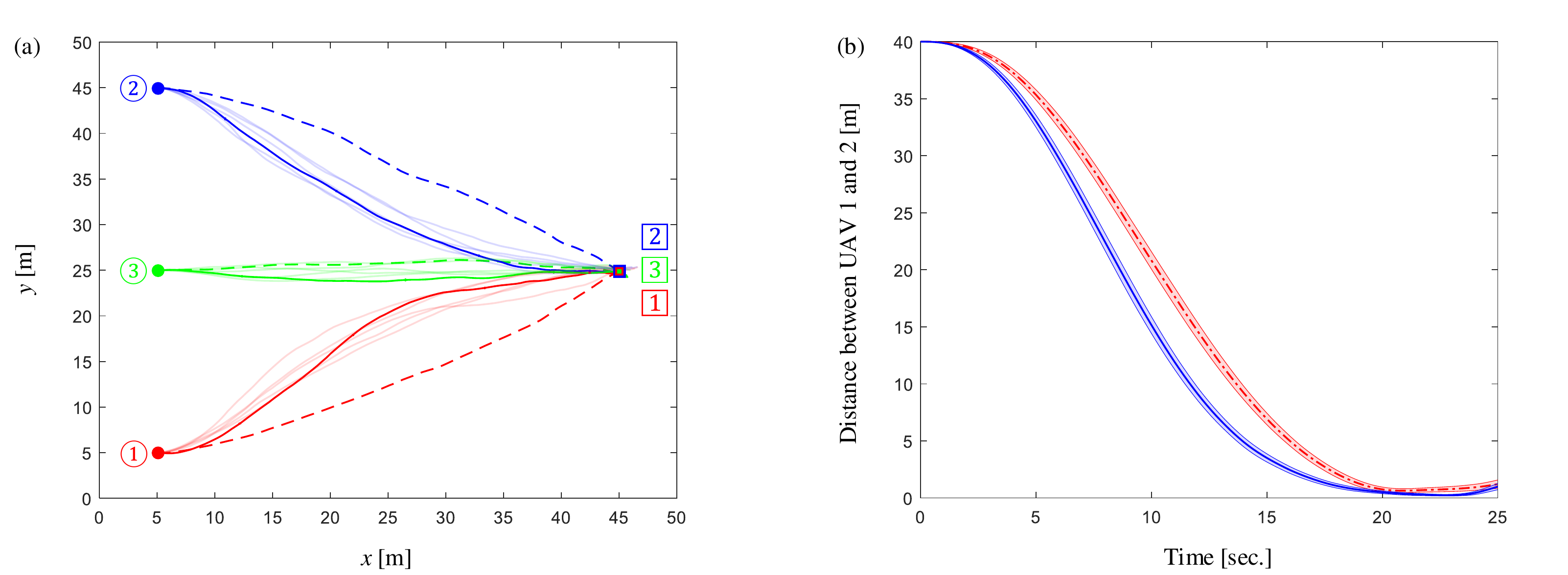}\vspace{-0em}\label{fig7}
	\caption{UAV trajectories and distance between UAV 1 and 2 from 50 trails subject to the REPS-based distributed LSOC algorithm. (a)~Trajectories of UAVs. Red, blue and green lines are trajectories for UAV 1, 2 and 3, respectively. Dashed lines are from a trail with factorized state costs ($w_{11} = w_{22} = w_{33} = 0.1$, $w_{12} = w_{21} = 0$), and solid (transparent) lines are from trails with joint state costs ($w_{11} = w_{22} = w_{33} = 0.1$, $w_{12} = w_{21} = 0.2$). (b) Red dashed line and blue solid line are respectively the mean distances from tails with factorized state cost and joint state cost. Height of strip represents one standard deviation.}
\end{figure}

We then consider a cluttered flight environment as the discrete-time example. Suppose three UAVs forming the network in~\hyperref[fig2]{Figure~2} and with initial states $x_1^{0} = (45, 5, 0.35, \pi)^\top$, $x_2^{0} = (5, 5, 0.65,  0)^\top$, $x_3^0 = (45, 5, 0.5, \pi)^\top$ and exit states $x_1^{t_f} = x_2^{t_f} = (45, 45, 0, \pi/2)^\top$, $x_3^{t_f} = (5, 45, 0, \pi)^\top$. The exit time is $t_f = 30$, and the length of each control cycle is $0.2$. When sampling trajectory roll-outs $\mathcal{Y}_i$, the time interval from $t$ to $t_f$ is partitioned into $K = 18$ intervals of equal length. The size of $\mathcal{Y}_i$ is $400$ trajectory roll-outs when adopting random sampling estimator, and the other parameters are the same as in the preceding continuous-time example. Subject to the sampling-based distributed LSOC algorithm, UAV trajectories and the relative distance between UAVs 1 and 2 subject to both joint and independent state-related costs are presented in~\hyperref[fig8]{Figure~9}. Letting the update rate constraint be $\delta_i = 50$ and the size of trajectory set be $|\mathcal{Y}_i| = 400$ and $150$ for the initial and subsequent policy iterations in every control period, experimental results obtained from the REPS-based distributed LSOC algorithm are given in~\hyperref[fig9]{Figure~10}. \hyperref[fig8]{Figure~9} and \hyperref[fig9]{Figure~10} show that our continuous-time distributed LSOC algorithms can guide UAVs to their terminal states, avoid obstacles, and shorten the relative distance between UAVs~1 and 2. It is also worth noticing that since there exist more than one shortest path for UAV 2 in condition of factorized state cost (see the footnote in \hyperref[sec4_1]{Section 4.1}), the standard variations of distance in \hyperref[fig8]{Figure~9} and \hyperref[fig9]{Figure~10} are significantly larger than other cases. Lastly, we compare the sample-efficiency between the sampling-based and REPS-based distributed LSOC algorithms in preceding two continuous-time examples. \hyperref[fig10]{Figure~11} shows the value of immediate cost function $c_2(\bar{x}_2, \bar{u}_2)$ from subsystem $\bar{\mathcal{N}}_2$ versus the amount of trajectory roll-outs, and the maximum numbers of trajectory roll-outs on horizontal axes are determined by the REPS-based trails with minimum amounts of sample roll-outs. In \hyperref[fig10]{Figure~11}, we can tell that the REPS-based distributed LSOC algorithm is more sample-efficient than the sampling-based algorithm.

\vspace{4em}

\begin{figure}[htpb]
	\centering
	\includegraphics[width=0.90\textwidth]{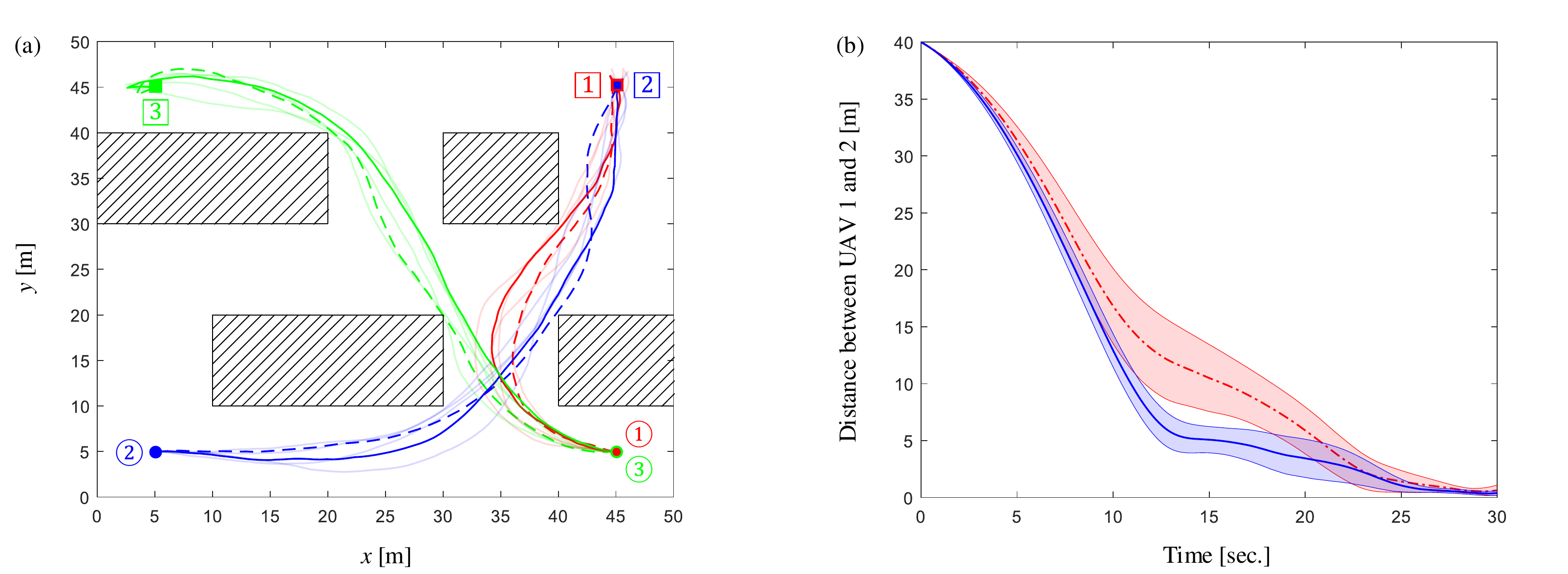}\label{fig8} 
	\caption{UAV trajectories and relative distance between UAV 1 and 2 from 100 trails based on random sampling estimator. (a)~Trajectories of UAVs. Red, blue and green lines are trajectories for UAV 1, 2 and 3, respectively. Dashed lines are from a trail with factorized (or independent) state costs ($w_{11} = w_{22} = w_{33} = 1$, $w_{12} = w_{21} = 0$), and solid (transparent) lines are from trails with joint state costs ($w_{11} = w_{22} = w_{33} = 1$, $w_{12} =1.5, w_{21} = 0.5$). (b) Distance between UAV 1 and 2. Red dashed line and blue solid line are respectively the mean distances from tails with independent state cost and joint state cost. Height of strip represents one standard deviation.}
\end{figure}

\clearpage

\begin{figure}[H]
	\centering
	\includegraphics[width=0.90\textwidth]{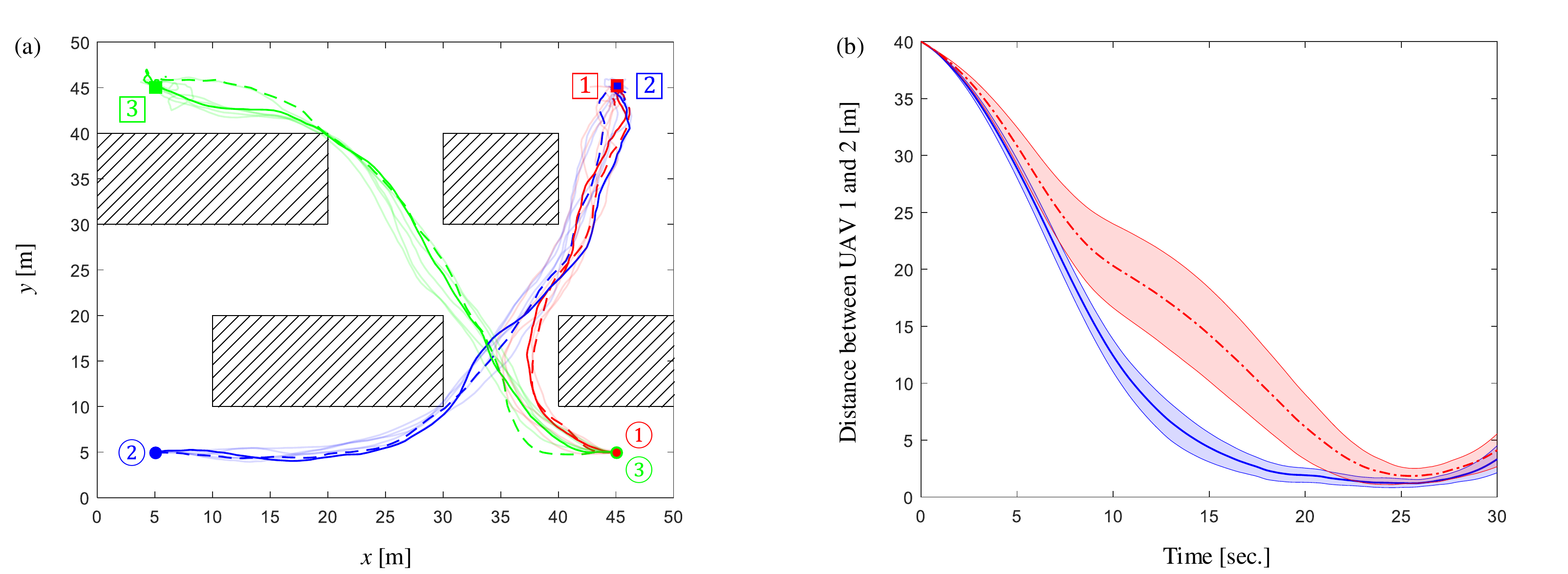}\label{fig9}\vspace{-0.5em}
	\caption{UAV trajectories and relative distance between UAV 1 and 2 from 100 trails based on REPS. (a)~Trajectories of UAVs. Red, blue and green lines are trajectories for UAV 1, 2 and 3, respectively. Dashed lines are from a trail with factorized state costs ($w_{11} = w_{22} = w_{33} = 0.18$, $w_{12} = w_{21} = 0$), and solid (transparent) lines are from trails with joint state costs ($w_{11} = w_{22} = w_{33} = 0.18$, $w_{12} =0.27, w_{21} = 0.1$). (b) Distance between UAV 1 and 2. Red dashed line and blue solid line are respectively the mean distances from tails with independent state cost and joint state cost. Height of strip is one standard deviation.}
\end{figure}

\begin{figure}[H]
	\centering
	\includegraphics[width=0.90\textwidth]{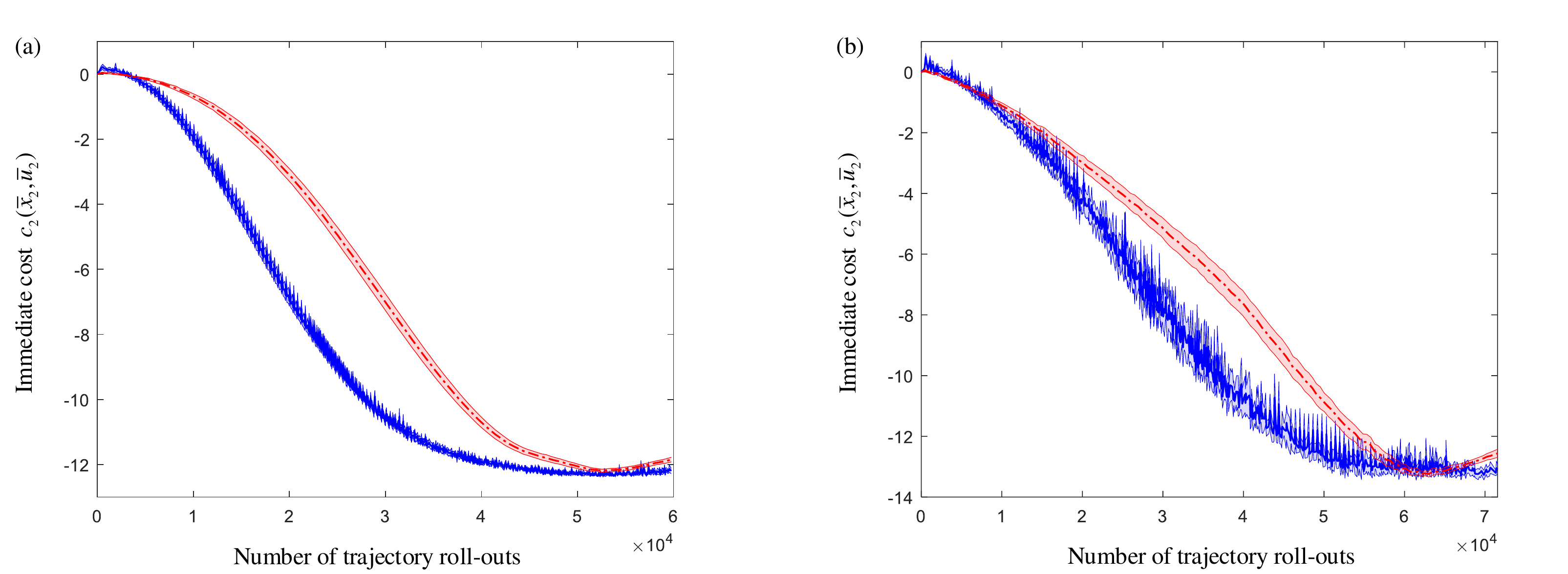}\label{fig10}\vspace{-0.5em}
	\caption{Sample-efficiency of continuous-time algorithms from 100 trails. (a) Immediate cost $c_2(\bar{x}_2, \bar{u}_2)$ in simple scenario without obstacle. Red dashed line is the mean of immediate cost subject to random sampling approach. Blue solid line is the mean of immediate cost subject to REPS algorithm. The height of strip is one standard deviation.	(b) Immediate cost $c_2(\bar{x}_2, \bar{u}_2)$ in complex scenario with obstacles. Interpretation is the same as (a).}
\end{figure}

To verify the effectiveness of distributed LSOC algorithms in larger MAS with more agents, we consider a line-shape network consisting of nine UAVs as shown in \hyperref[fig12]{Figure~12}. These nine UAVs $i = \{1, 2, \cdots, 9\}$ are initially distributed at $x_i^0 = (10, 100 - 10i, 0.5, 0)$ as shown in \hyperref[fig13]{Figure~13}, and they can be divided into three groups based on their exit states. UAV 1 to 6 share a same exit state $A$ at $x_{1:6}^{t_f} = (90, 65, 0, 0)$; the exit state $B$ of UAV 7 and 8 is at $x_{7:8}^{t_f} = (90, 25, 0, 0)$, and UAV 9 is expected to exit at state $C$, $x_9^{t_f} = (90, 10, 0, 0)$, where the exit time $t_f = 40 \textrm{ sec}$. As exhibited in \hyperref[fig12]{Figure~12}, UAVs from different groups are either loosely coupled through their terminal cost functions or mutually independent, where the latter scenario does not require any communication between agents. A state-related cost function $q_i(\bar{x}_i)$ in subsystem $\bar{\mathcal{N}}_i$ is designed to consider and optimize the distances between neighboring agents and towards the exit state of agent $i$: 
\begin{align*}
	q_i(\bar{x}_i) = w_{ii} \cdot (\|(x_i, y_i) - (x_i^{t_f}, y_i^{t_f})\|_2 - d_i^{\max})  &+ w_{i, i-1} \cdot (\|(x_i, y_i) - (x_{i-1}, y_{i-1})\| - d^{\max}_{i, i-1})  \\
	&+ w_{i, i+1} \cdot (\|(x_i, y_i) - (x_{i+1}, y_{i+1})\| - d^{\max}_{i, i+1}),
\end{align*}
where $w_{i,j}$ is the weight related to the distance between agent $i$ and $j$; $w_{i, j} = 0$ when $j = 0$ or $10$; $d_{i,j}^{\max}$ is the regularization term for numerical stability, which is assigned by the initial distance between agents $i$ and $j$ in this demonstration, and the remaining notations and parameters are the same as the assignments in~\eqref{UAV_RunningCost} and the first example in this subsection if not explicitly stated. Trajectories of UAV team subject to two distributed LSOC algorithms, \hyperref[alg2]{Algorithm~2} and \hyperref[alg3]{Algorithm~3} in \hyperref[appE]{Appendix~E}, are presented in \hyperref[fig13]{Figure~13}. For some network structures, such as line, loop, star and complete binary tree, in which the scale of every factorial subsystem is tractable, increasing the total number of agents in network will not dramatically boost the computational complexity on local agents thanks to the distributed LSOC framework proposed in this paper. Verification along with more simulation examples on the generalization of distributed LSOC controllers discussed in \hyperref[sec3.3]{Section~3.3} is supplemented in \cite{Song_arxiv_2020}.

\vspace{1em}

\begin{figure}[H]
	\centering 
	\includegraphics[width=0.68\textwidth]{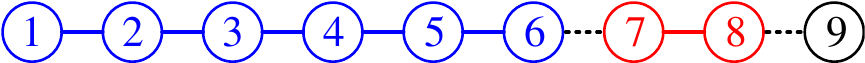}
	\caption{Communication network of a UAV team with nine agents. UAV 1 to 6, as well as UAV 7 and 8 are strongly coupled (represented by solid lines) through their immediate cost functions. UAV 6 and 7, as well as UAV 8 and 9 are loosely coupled (represented by dashed lines) through their terminal cost functions.}\label{fig12}
\end{figure}

\begin{figure}[H]
	\centering
	\includegraphics[width=0.90\textwidth]{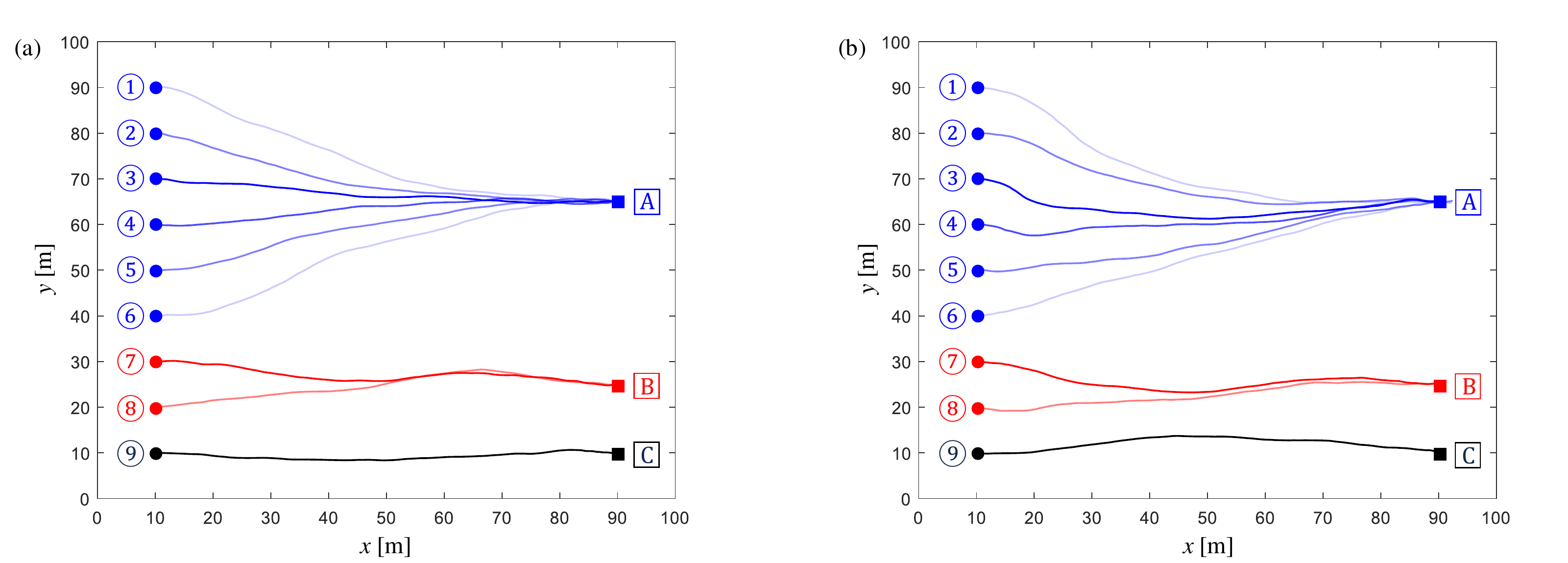}\label{fig13}
	\caption{UAV trajectories subject to two distributed LSOC algorithms. (a) Trajectories of UAV team controlled by the sampling-based distributed LSOC algorithm with $w_{1,2} = w_{2,3} = w_{3,4} = w_{6,5} = w_{7,8} = w_{8,7} = w_{99} =1$, $w_{ii} = w_{4, 3} = w_{4, 5} = w_{5, 4} = w_{5, 6} = 0.5$, and $w_{2, 1} = w_{3, 2} = w_{6, 7} = w_{7, 6} = w_{9,8} = 0$. (b) Trajectories of UAV team controlled by the REPS-based distributed LSOC algorithm with $w_{1,2} = w_{2,3} = w_{3,4} = w_{6,5} = w_{7,8} = w_{8,7} = w_{99} =0.2$, $w_{ii} = w_{4, 3} = w_{4, 5} = w_{5, 4} = w_{5, 6} = 0.1$, and $w_{2, 1} = w_{3, 2} = w_{6, 7} = w_{7, 6} = w_{9,8} = 0$.}
\end{figure}

Many interesting problems remain unsolved in the area of distributed (linearly solvable) stochastic optimal control and deserve further investigation. Most of existing papers, including this paper, assumed that the passive and controlled dynamics of different agents are mutually independent. However, when we consider some practical constraints, such as the collisions between different UAVs, the passive and controlled dynamics of different agents are usually not mutually dependent. Meanwhile, while this paper considered the scenario when all states of local agents are fully observable, it will be interesting to study the MAS of partially observable agents with hidden states. Lastly, distributed LSOC algorithms subject random network, communication delay, infinite horizon or discounted cost are also worth our attention.

\section{Conclusion}\label{sec5}
Discrete-time and continuous-time distributed LSOC algorithms for networked MASs have been investigated in this paper. A distributed control framework based on factorial subsystems has been proposed, which allows to optimize the joint state or cost function between neighboring agents with local observation and tractable computational complexity. Under this distributed framework, the discrete-time multi-agent LSMDP problem was addressed by respectively solving the local systems of linear equations in each subsystem, and a parallel programming scheme was proposed to decentralize and expedite the computation. The optimal control action/policy for continuous-time multi-agent LSOC problem was formulated as a path integral, which was approximated by a distributed sampling method and a distributed REPS method, respectively. Numerical examples of coordinated UAV teams were presented to verify the effectiveness and advantages of these algorithms, and some open problems were given at the end of this paper.

\section*{Acknowledgments}
This work was supported in part by NSF-NRI, AFOSR, and ZJU-UIUC Institute Research Program. The authors would like to appreciate the constructive comments from Dr. Hunmin Lee and Gabriel Haberfeld. In this arXiv version, the authors would also like to thank the readers and staff on arXiv.org.

\section*{Conflict of Interest Statement}
The authors have agreed to publish this article, and we declare that there is no conflict of interests regarding the publication of this article.

\section*{Appendix A: Proof for Theorem 1}\label{appA}

\noindent \textit{Proof for \hyperref[thm1]{Theorem 1}}: Substituting the joint running cost function~\eqref{eq4} into the joint Bellman equation~\eqref{CompBellman}, and by the definitions of KL-divergence and exponential transformation~\eqref{ExpTrans}, we have
\begin{equation}\label{eqA1}
	\begin{split}
		V_i(\bar{x}_i) & = \min_{\bar{u}_i} \left\{  q_i(\bar{x}_i)    + \textrm{KL}(\bar{u}_i(\cdot | \bar{x}_i) \ \| \ \bar{p}_i(\cdot | \bar{x}_i)) + \mathbb{E}_{\bar{x}'_i \sim \bar{u}_i(\cdot | \bar{x}_i)}[V_i(\bar{x}'_i)]  \right\}\\
		&= \min_{\bar{u}_i} \left\{ q_i(\bar{x}_i)    + \mathbb{E}_{\bar{x}'_i \sim \bar{u}_i(\cdot | \bar{x}_i)} \left[\log \dfrac{\bar{u}_i(\bar{x}'_i | \bar{x}_i)}{\bar{p}_i(\bar{x}'_i | \bar{x}_i)}\right] + \mathbb{E}_{\bar{x}' \sim \bar{u}_i (\cdot | \bar{x}_i)}\left[\log \frac{1}{Z_i(\bar{x}'_i)}\right] \right\}\\
		& = \min_{\bar{u}_i} \left\{  q_i(\bar{x}_i) +  \mathbb{E}_{\bar{x}'_i \sim \bar{u}_i(\cdot | \bar{x}_i)} \left[\log \dfrac{\bar{u}_i(\bar{x}'_i|\bar{x}_i)}{\bar{p}_i(\bar{x}'_i|\bar{x}_i) Z_i(\bar{x}'_i)}\right]   \right\}. 
	\end{split}
\end{equation}
The optimal policy will be straightforward if we can rewrite the expectation on the RHS of~\eqref{eqA1} as a KL-divergence and exploit the minimum condition of KL-divergence. While the control mapping $\bar{u}_i(\bar{x}_i'|\bar{x}_i)$  in~\eqref{eqA1} is a probability distribution, the denominator $p(x'|x)Z(x')$ is not necessarily a probability distribution. Hence, we define the following normalized term
\begin{equation}\label{eqA2}
	\mathcal{W}_i(\bar{x}_i) = \sum_{\bar{x}'_i}\bar{p}_i(\bar{x}'_i|\bar{x}_i)Z_i(\bar{x}'_i). 
\end{equation}
Since $p(x'|x)Z(x')/\mathcal{W}_i(\bar{x}_i)$ is a well-defined probability distribution, we can rewrite the joint Bellman equation~\eqref{eqA1} as follows
\begin{equation}\label{eqA3}
	\begin{split}
		V_i(\bar{x}_i) & = \min_{\bar{u}_i} \left\{  q_i(\bar{x}_i)   +  \mathbb{E}_{\bar{x}'_i \sim \bar{u}_i(\cdot | \bar{x}_i)} \left[  \log \dfrac{\bar{u}_i(\bar{x}'_i | \bar{x}_i)}{\bar{p}_i(\bar{x}'_i | \bar{x}_i) Z_i(\bar{x}'_i) / \mathcal{W}_i(\bar{x}_i)}  - \log \mathcal{W}_i(\bar{x}_i)  \right]   \right\}\\
		& = \min_{\bar{u}_i} \left\{ q_i(\bar{x}_i)   - \log \mathcal{W}_i(\bar{x}_i) + \mathrm{KL}\left( \bar{u}_i(  \cdot  | \bar{x}_i) \  \Bigg\| \ \frac{\bar{p}_i(\cdot | \bar{x}_i)Z_i(\cdot)}{\mathcal{W}_i(\bar{x}_i)} \right) \right\},
	\end{split}
\end{equation}
where only the last term depends on the joint control action $\bar{u}_i( \cdot  | \bar{x}_i)$. According to the minimum condition of KL-divergence, the last term in~\eqref{eqA3} attains its absolute minimum at $0$ if and only if 
\begin{equation*}
	\bar{u}_i^*(\cdot | \bar{x}_i) = \frac{\bar{p}_i(\cdot | \bar{x}_i)Z_i(\cdot)}{\mathcal{W}_i(\bar{x}_i)},
\end{equation*}
which gives the optimal control action~\eqref{OptimalControl} in~\hyperref[thm1]{Theorem~1}. By substituting~\eqref{OptimalControl} into~\eqref{eqA3}, we can minimize the RHS of joint Bellman equation~\eqref{CompBellman} and remove the minimum operator:
\begin{equation}\label{eqA5}
	V_i(\bar{x}_i) = q_i(\bar{x}_i) - \log \mathcal{W}_i(\bar{x}_i).
\end{equation}
Exponentiating both sides of~\eqref{eqA5} and substituting~\eqref{ExpTrans} and~\eqref{eqA2} into the result, the Bellman equation~\eqref{eqA5} can be rewritten as a linear equation with respect to the desirability function
\begin{equation}
	\begin{split}
		Z_i(\bar{x}_i) = \exp[-V_i(\bar{x}_i)] = \exp[-q_i(\bar{x}_i)] \cdot \mathcal{W}_i(\bar{x}_i) =  \exp[-q_i(\bar{x}_i)] \cdot \sum_{\bar{x}'_i}\bar{p}_i(\bar{x}'_i|\bar{x}_i)Z_i(\bar{x}'_i),
	\end{split}
\end{equation}
which implies~\eqref{eq_prop1} in~\hyperref[thm1]{Theorem~1}. This completes the proof. \qed

\section*{Appendix B: Proof for Theorem 2}\label{appB}

Before we present the proof for \hyperref[thm2]{Theorem~2}, Feynman–Kac formula that builds an important relationship between parabolic PDEs and stochastic processes is introduced as \hyperref[lem1]{Lemma~4}.

\begin{lemma}\label{lem1}[Feynman–Kac formula]
	Consider the Kolmogorov backward equation (KBE) described as follows
	\begin{equation*}
		\partial_t Z(x, t) = \frac{q(x, t)}{\lambda} \cdot Z(x, t) - f(x, t)^\top \cdot \nabla_x Z(x, t) - \frac{1}{2}\textrm{tr}\left( B(x)\sigma \sigma^\top B(x)^\top \cdot \nabla_{xx}Z(x, t) \right),
	\end{equation*}
	where the terminal condition is given by $Z(x, t_f) = \exp[ -\phi(x({t_f})) / \lambda]$. Then the solution to this KBE can be written as a conditional expectation
	\begin{equation*}
		Z(x, t) = \mathbb{E}_{x, t}\left[ \exp\left(  -\frac{1}{\lambda} \phi(y(t_f))  \right) - \frac{1}{\lambda} \int_{t}^{t_f} q(y, \tau) \ d\tau \ \Big| \ y(t) = x \right],
	\end{equation*}
	under the probability measure that $y$ is an It\^{o} diffusion process driven by the equation $d y(\tau) = f(y,\tau) d\tau + B(y) \sigma \cdot  dw(\tau)$ with initial condition $y(t) = x$. \qed
\end{lemma}

\noindent We now show the proof for \hyperref[thm2]{Theorem 2}.

\vspace{0.5em}

\noindent \textit{Proof for \hyperref[thm2]{Theorem 2}}: First, we show that the joint optimality equation~\eqref{ContBellman} can be formulated into the joint stochastic HJB equation~\eqref{sto_HJB} that gives an analytic expression of optimal control action~\eqref{OptimalAction}. Substituting immediate cost function~\eqref{cont_cost} into optimality equation~\eqref{ContBellman} and letting $s$ be a time step between $t$ and $t_f$, optimality equation~\eqref{ContBellman} can be rewritten as
\begin{equation}\label{EqB1}
	\begin{split}
		V_i(\bar{x}_i, t) & = \min_{\bar{u}_i} \mathbb{E}^{\bar{u}_i}_{\bar{x}_i, t} \left[ \phi_i(\bar{x}_{i},t_f) + \int_{t}^{t_f}    q_i(\bar{x}_i, \tau) + \frac{1}{2}  \bar{u}_i(\bar{x}_i, \tau)^\top \bar{R}_i\bar{u}_i(\bar{x}_i, \tau) \ d\tau \right]\\
		& =  \min_{\bar{u}_i} \mathbb{E}_{\bar{x}_i,t}^{\bar{u}_i} \left[ V_i(\bar{x}_i, s) + \int_{t}^{s}  q_i(\bar{x}_i, \tau) + \frac{1}{2}  \bar{u}_i(\bar{x}_i, \tau)^\top \bar{R}_i\bar{u}_i(\bar{x}_i, \tau) \ d\tau \right].
	\end{split}
\end{equation}
With some rearrangements and dividing both sides of~\eqref{EqB1} by $s - t > 0$, we have
\begin{equation}\label{EqB2}
	0 = \min_{\bar{u}_i} \mathbb{E}_{\bar{x}_i,t}^{\bar{u}_i} \left[ \frac{V_i(\bar{x}_i, s) - V_i(\bar{x}_i,t)}{s - t} + \frac{1}{s - t}\int_{t}^{s} q_i(\bar{x}_i, \tau)  +  \frac{1}{2} \bar{u}_i(\bar{x}_i, \tau)^\top \bar{R}_i\bar{u}_i(\bar{x}_i, \tau) \ d\tau   \right].
\end{equation}
By letting $s \rightarrow t$, the optimality equation~\eqref{EqB2} becomes
\begin{equation}\label{EqB3}
	0 = \min_{\bar{u}_i} \mathbb{E}_{\bar{x}_i,t}^{\bar{u}_i} \left[ \frac{dV_i(\bar{x}_i, t)}{dt} + q_i(\bar{x}_i, t) +  \frac{1}{2} \bar{u}_i(\bar{x}_i, t)^\top\bar{R}_i\bar{u}_i(\bar{x}_i, t) \right].
\end{equation}
Applying It\^{o}'s formula~\cite{LeGall_2016}, the differential $dV_i(\bar{x}_i, t)$ in~\eqref{EqB3} can be expanded as
\begin{equation}\label{eq54}
	dV_i(\bar{x}_i, t) = \sum_{j \in \bar{\mathcal{N}}_i}   \sum_{m=1}^{M} \frac{\partial V_i(\bar{x}_i, t)}{\partial x_{j(m)}} dx_{j(m)} + \frac{\partial V_i(\bar{x}_i, t)}{\partial t} dt + \frac{1}{2} \sum_{j, k \in \bar{\mathcal{N}}_i} \sum_{m, n = 1}^{M} \frac{\partial^2 V_i(\bar{x}_i, t)}{\partial x_{j(m)}\partial x_{k(n)}}dx_{j(m)} dx_{k(n)}.
\end{equation}
For conciseness, we will omit the indices (or subscripts) of state components, $(m)$ and $(n)$, in the following derivations. Dividing both sides of~\eqref{eq54} by $dt$, taking the expectation over all trajectories that initialized at $(\bar{x}_i^t, t)$ and subject to control action $\bar{u}_i$, and substituting the joint dynamics~\eqref{eq13} into the result, we have
\begin{equation}\label{EqB4}
	\begin{split}
		\mathbb{E}_{\bar{x}_i,t}^{\bar{u}_i}\left[ \frac{dV_i(\bar{x}_i, t)}{dt} \right] = \sum_{j \in \bar{\mathcal{N}}_i}  [f_j(x_j,t) & + B_j(x_j) u_j(\bar{x}_i, t) ]^\top \cdot \nabla_{x_j} V_i(\bar{x}_i,t)   + \frac{\partial V_i(\bar{x}_i, t)}{\partial t}  \\
		&+ \frac{1}{2} \sum_{j \in \bar{\mathcal{N}}_i}  \textrm{tr} \left( B_j(x_j)\sigma_j \sigma_j^\top B_j(x_j)^\top \cdot \nabla_{x_jx_j}V_i(\bar{x}_i, t) \right),
	\end{split}
\end{equation}
where the identity $\mathbb{E}_{\bar{x}_i,t}^{\bar{u}_i}[dx_{j(m)}dx_{k(n)}] = (\sigma_j \sigma^\top_j)_{mm} \delta_{jk} \delta_{mn} dt$ derived from the property of standard Brownian motion  $\mathbb{E}_{\bar{x}_i,t}^{\bar{u}_i}[dw_{j(m)}dw_{k(n)}] = \delta_{jk} \delta_{mn} dt$ is invoked, and operators $\nabla_{x_j}$ and $\nabla_{x_jx_j}$ follow the same definitions in~\eqref{singleHJB}. Substituting~\eqref{EqB4} into~\eqref{EqB3}, the joint stochastic HJB equation~\eqref{sto_HJB} in \hyperref[thm2]{Theorem~2} is obtained
\begin{equation*}
	\begin{split}
		-\partial_t V_i(\bar{x}_i, t) = & \min_{\bar{u}_i} \mathbb{E}_{\bar{x}_i,t}^{\bar{u}_i}   \bigg[  \sum_{j \in \bar{\mathcal{N}}_i} [f_j(x_j,t)  + B_j(x_j) u_j(\bar{x}_i, t)]^\top \cdot \nabla_{x_j} V_i(\bar{x}_i,t) + q_i(\bar{x}_i, t)   \\
		&  +  \frac{1}{2} \bar{u}_i(\bar{x}_i, t)^\top\bar{R}_i\bar{u}_i(\bar{x}_i, t) + \frac{1}{2} \sum_{j \in \bar{\mathcal{N}}_i}  \textrm{tr} \left( B_j(x_j)\sigma_j \sigma_j^\top B_j(x_j)^\top \cdot \nabla_{x_jx_j}V_i(\bar{x}_i, t)  \right) \bigg],
	\end{split}
\end{equation*}
where the boundary condition is given by $V_i(\bar{x}_i, t_f) = \phi_i(\bar{x}_i)$. The joint optimal control action $\bar{u}_i^*(\bar{x}_i, t)$ can be obtained by setting the derivative of~\eqref{sto_HJB} with respect to $\bar{u}_i(\bar{x}_i, t)$ equal to zero. When the control weights $R_j$ of each agent $j \in \bar{\mathcal{N}}_i$ are coupled, \textit{i.e.} the joint control weight matrix $\bar{R}_i$ cannot be formulated as a block diagonal matrix, the joint optimal control action for subsystem $\mathcal{\bar{N}}_i$ is given as~\eqref{OptimalAction} in~\hyperref[thm2]{Theorem~2},  $\bar{u}^*_i(\bar{x}_i, t) = -\bar{R}_i^{-1}$$\bar{B}_i(\bar{x}_i)^\top \cdot \nabla_{\bar{x}_i}V_i(x_i, t)$, where $\nabla_{\bar{x}_i}$ denotes the gradient with respect to the joint state $\bar{x}_i$. However, it is more common in practice that the joint control weight matrix is given by $\bar{R}_i = \textrm{diag}\{R_i, R_{j \in \mathcal{N}_i} \}$ as in~\eqref{cont_cost}, and the joint control cost satisfies $\frac{1}{2}\bar{u}_i^\top \bar{R}_i \bar{u}_i = \sum_{j\in\mathcal{\bar{N}}_i}\frac{1}{2}u_j^\top R_ju_j$. Setting the derivatives of~\eqref{sto_HJB} with respect to ${u}_j(\bar{x}_i, t)$ equal to zero in the latter case gives the local optimal control action of agent $j \in \mathcal{\bar{N}}_i$
\begin{equation}\label{LocOptCtrl}
	u^*_{j}(\bar{x}_i, t) = - R_j^{-1} B_j(x_j)^\top \cdot \nabla_{x_j}V_i(\bar{x}_i, t).
\end{equation}
For conciseness of derivations and considering the formulation of~\eqref{sto_HJB}, we will mainly focus on the latter scenario, $\bar{R}_i = \textrm{diag}\{R_i, R_{j \in \mathcal{N}_i} \}$, in the remaining part of this proof.

In order to solve the stochastic HJB equation~\eqref{sto_HJB} and evaluate the optimal control action~\eqref{OptimalAction}, we consider to linearize~\eqref{sto_HJB} with the Cole-Hopf transformation~\eqref{ContTrans}, $V_i(\bar{x}_i, t) = \lambda_i  \log Z(\bar{x}_i, t)$. Subject to this transformation, the derivative and gradients in~\eqref{sto_HJB} satisfy
\begin{equation}\label{der1}
	\partial_t V_i(\bar{x}_i,t) = - \lambda_i \cdot  \frac{ \partial_{t} Z_i(\bar{x}_i,t)}{Z_i(\bar{x}_i,t)},
\end{equation}
\begin{equation}\label{der2}
	\nabla_{x_j}V_i(\bar{x}_i,t) = - \lambda_i \cdot  \frac{\nabla_{x_j} Z_i(\bar{x}_i,t)}{Z_i(\bar{x}_i,t)}, 
\end{equation}
\begin{equation}\label{der3}
	\nabla_{x_jx_j} V_i(\bar{x}_i,t) = - \lambda_i \cdot \left[ \frac{\nabla_{x_jx_j} Z_i(\bar{x}_i,t)}{Z_i(\bar{x}_i,t)} - \frac{\nabla_{x_j}Z_i(\bar{x}_i,t) \cdot \nabla_{x_j}Z_i(\bar{x}_i,t)^\top}{Z_i(\bar{x}_i,t)^2} \right].
\end{equation}
Equalities~\eqref{der2} and~\eqref{der3} will still hold when we replace the agent's state $x_j$ in $\nabla_{x_j}$ and $\nabla_{x_j x_j}$ by the joint state $\bar{x}_i$ and reformulate the joint HJB \eqref{sto_HJB} in a more compact form. Substituting local optimal control action~\eqref{LocOptCtrl}, and gradients~\eqref{der2} and \eqref{der3} into the stochastic HJB equation~\eqref{sto_HJB}, the corresponding terms of each agent $j \in \mathcal{\bar{N}}_i$ in~\eqref{sto_HJB} satisfy
\begin{align}\label{EqB5}
	[B_j(x_j)  u_j(\bar{x}_i, t)]^\top  \nabla_{x_j} &  V_i(\bar{x}_i, t)  +   \frac{1}{2}{u}_{j}(\bar{x}_i,t)^\top {R}_{j}  u_j(\bar{x}_i, t)  \nonumber \\
	=  & - \frac{1}{2} \nabla_{x_j}V_i(\bar{x}_i, t)^\top \cdot B_j(x_j)R_j^{-1} B_j(x_j)^\top \cdot \nabla_{x_j}  V_i(\bar{x}_i, t) \\
	=  & \ \frac{-\lambda_i^2}{2 \cdot Z_i(\bar{x}_i, t)^2} \cdot \nabla_{x_j} Z_i(\bar{x}_i,t)^\top \cdot B_j(x_j)R_j^{-1} B_j(x_j)^\top \cdot \nabla_{x_j} Z_i(\bar{x}_i,t), \nonumber
\end{align}
\begin{align}\label{EqB6}
	\frac{1}{2} \textrm{tr}   \big(  B_j(x_j)\sigma_j & \sigma_j^\top  B_j(x_j)^\top  \cdot \nabla_{x_jx_j} V_i(\bar{x}_i, t)  \big) \nonumber\\
	= & \ \frac{ -\lambda_i}{2 \cdot Z_i(\bar{x}_i,t)} \cdot \textrm{tr}\left( B_j(x_j)\sigma_j \sigma_j^\top B_j(x_j)^\top \cdot \nabla_{x_jx_j} Z_i(\bar{x}_i,t) \right) \\
	& + \frac{\lambda_i}{2 \cdot Z_i(\bar{x}_i,t)^2} \cdot \textrm{tr} \left( B_j(x_j)\sigma_j \sigma_j^\top B_j(x_j)^\top \cdot \nabla_{x_j}Z_i(\bar{x}_i,t) \cdot \nabla_{x_j}Z_i(\bar{x}_i,t)^\top  \right). \nonumber
\end{align}
By the properties of trace operator, the quadratic terms in~\eqref{EqB5} and~\eqref{EqB6} will be canceled if 
\begin{equation}\label{EqB7}
	\sigma_{j} \sigma^\top_{j} = \lambda_i R_j^{-1},
\end{equation}
\textit{i.e.} $R_j = (\sigma_{j} \sigma^\top_{j} / \lambda_i)^{-1}$, which is equivalent to the condition $\bar{\sigma}_i\bar{\sigma}_i^\top = \lambda_i \bar{R}^{-1}_i$ or $\bar{R}_i = (\bar{\sigma}_i\bar{\sigma}_i^\top / \lambda_i)^{-1}$ subject to joint dynamics~\eqref{eq13}. Substituting~\eqref{OptimalAction}, \eqref{der1}, \eqref{EqB5} and~\eqref{EqB6} into stochastic HJB equation~\eqref{sto_HJB}, we then remove the minimization operator and obtain the linearized PDE as~\eqref{Z_Function} in \hyperref[thm2]{Theorem~2}
\begin{equation*}
	\partial_{t} Z_i(\bar{x}_i, t) = \left[\frac{q_i(\bar{x}_i, t)}{\lambda_i} - \sum_{j\in\bar{\mathcal{N}}_i} f_j(x_j, t)  \nabla_{x_j} - \frac{1}{2}\sum_{j \in \bar{\mathcal{N}}_i} \textrm{tr}\left( B_j(x_j)\sigma_j \sigma_j^\top B_j(x_j)^\top  \nabla_{x_jx_j} \right)  \right]Z_i(\bar{x}_i,t),
\end{equation*}
where the boundary condition is given by $Z_i(\bar{x}_i, t_f) = \exp[- \phi_i(\bar{x}_i) / \lambda_i]$. Once the value of desirability function $Z_i(\bar{x}_i, t)$ in~\eqref{Z_Function} is solved, we can readily figure out the value function $V_i(\bar{x}_i, t)$ and joint optimal control action from \eqref{ContTrans} and~\eqref{OptimalAction}, respectively. Invoking the Feynman–Kac formula introduced in \hyperref[lem1]{Lemma~1}, a solution to~\eqref{Z_Function} can be formulated as~\eqref{Z_Solution} in \hyperref[thm2]{Theorem~2}
\begin{equation*}
	Z_i(\bar{x}_i, t) = \mathbb{E}_{\bar{x}_i,t}\left[ \exp\left( -\frac{1}{\lambda_i} \phi_i(\bar{y}^{t_f}_i) -\frac{1}{\lambda_i} \int_{t}^{t_f}q_i(\bar{y}_i, \tau) \ d\tau \right) \right],
\end{equation*}
where $\bar{y}(t)$ satisfies the uncontrolled dynamics $d \bar{y}_i(\tau) = \bar{f}_i(\bar{y}_i,\tau) d\tau + \bar{B}_i(\bar{y}_i)   \bar{\sigma}_i  \cdot d\bar{w}_i(\tau)$ with initial condition $\bar{y}_i(t) = \bar{x}_i(t)$. This completes the proof. \qed

\section*{Appendix C: Proof of Proposition 3}\label{appC}
\noindent \textit{Proof of \hyperref[prop3]{Proposition~3}}: First, we formulate the desirability function~\eqref{Z_Solution} as a path integral shown in~\eqref{Prop3E1}.  Partitioning the time interval from $t$ to $t_f$ into $K$ intervals of equal length $\varepsilon > 0$, $t = t_0 < t_1 < \cdots < t_K = t_f$, we can rewrite~\eqref{Z_Solution} as the following path integral
\begin{equation}\label{EqC1}
	\begin{split}
		Z_i(\bar{x}_i,t) = \ & \mathbb{E}_{\bar{x}_i,t}\left[ \exp\left( -\frac{1}{\lambda_i} \phi_i(\bar{y}^{t_f}_i) -\frac{1}{\lambda_i} \int_{t}^{t_f} q_i(\bar{y}_i, \tau) \ d\tau \right)  \right] \\
		= \ & \int d\bar{x}^{(1)}_i \cdots \int  \exp\left( -\frac{1}{\lambda_i} \phi_i(\bar{x}^{(K)}_i) \right) \cdot \prod_{k=0}^{K-1} Z_i(\bar{x}_i^{(k+1)}, t_{k+1}; \bar{x}_i^{(k)}, t_k) \ d\bar{x}_i^{(K)},
	\end{split}
\end{equation}
where the integral of variable $\bar{x}_i^{(k)}$ is over the set of all joint uncontrolled trajectories $\bar{x}_i(\tau)$ on time interval $[t_{k-1}, t_k)$ and with initial condition $\bar{x}_i^{(0)} = \bar{x}_i(t_0) = \bar{x}_i(t)$, which can be measured by agent $i$ at initial time $t$, and the function $Z_i(\bar{x}_i^{(k+1)}, t_{k+1}; \bar{x}_i^{(k)}, t_k)$ is implicitly defined by
\begin{equation}\label{EqC1.5}
	\begin{split}
		\int f(\bar{x}_{i}^{(k+1)}) \cdot & Z_i(\bar{x}_i^{(k+1)},  t_{k+1};  \bar{x}_i^{(k)}, t_k) \ d\bar{x}_i^{(k+1)} \\
		& = \mathbb{E}_{\bar{x}_i^{(k)}, t_k}\left[ f(\bar{x}_i^{(k+1)}) \cdot \exp\left( -\frac{1}{\lambda_i}\int_{t_i}^{t_{i+1}} q_i(\bar{y}_i, \tau) \ d\tau \right)  \Big| \ \bar{y}_i(t_k) = \bar{x}_i^{(k)} \right]
	\end{split}
\end{equation}
for arbitrary functions $f(\bar{x}_i^{(k+1)})$. Based on definition~\eqref{EqC1.5} and in the limit of infinitesimal $\varepsilon$, the function $Z_i(\bar{x}_i^{(k+1)}, t_{k+1} ; \bar{x}_i^{(k)}, t_k)$ can be approximated by 
\begin{equation}\label{EqC2}
	Z_i(\bar{x}_i^{(k+1)}, t_{k+1} ; \bar{x}_i^{(k)}, t_k) = p_i(\bar{x}_i^{(k+1)}, t_{k+1} | \bar{x}_i^{(k)}, t_k) \cdot \exp\left(-\frac{\varepsilon}{\lambda_i} \cdot q_i(\bar{x}_i^{(k)}, t_k)    \right),
\end{equation}
where $p_i(\bar{x}_i^{(k+1)}, t_{k+1} | \bar{x}_i^{(k)}, t_k)$ is the transition probability of uncontrolled dynamics from state-time pair $(\bar{x}_i^{(k)}, t_k)$ to $(\bar{x}_i^{(k+1)}, t_{k+1})$ and can be factorized as follows
\begin{equation}\label{EqC2_5}
	\begin{split}
		p_i(\bar{x}_i^{(k+1)}, t_{k+1} | \bar{x}_i^{(k)}, t_k) & =  p_i(\bar{x}_{i(n)}^{(k+1)}, t_{k+1} | \bar{x}_i^{(k)}, t_k) \cdot p_i(\bar{x}_{i(d)}^{(k+1)}, t_{k+1} | \bar{x}_{i}^{(k)}, t_k)\\
		&  = p_i(\bar{x}_{i(n)}^{(k+1)}, t_{k+1} | \bar{x}_{i(d)}^{(k)}, \bar{x}_{i(n)}^{(k)},  t_k) \cdot p_i(\bar{x}_{i(d)}^{(k+1)}, t_{k+1} | \bar{x}_{i(d)}^{(k)}, \bar{x}_{i(n)}^{(k)}, t_k)\\
		& \propto p_i(\bar{x}_{i(d)}^{(k+1)}, t_{k+1} | \bar{x}_{i}^{(k)}, t_k),\\		
	\end{split}
\end{equation}
where $p_i(\bar{x}_{i(n)}^{(k+1)}, t_{k+1} | \bar{x}_{i(d)}^{(k)}, \bar{x}_{i(n)}^{(k)},  t_k)$ is a Dirac delta function, since $\bar{x}_{i(n)}^{(k+1)}$ can be deterministically calculated from $\bar{x}_{i(n)}^{(k)}$ and $\bar{x}_{i(d)}^{(k)}$. Provided the directly actuated uncontrolled dynamics from~\eqref{Partitioned_dynamics}
\begin{equation*}
	\bar{x}_{i(d)}^{(k+1)} - \bar{x}_{i(d)}^{(k)} = \bar{f}_{i(d)}(\bar{x}^{(k)}_i, t_k) \varepsilon  + \bar{B}_{i(d)}(\bar{x}^{(k)}_i) \cdot \bar{\sigma}_i \bar{w}_i
\end{equation*}
with Brownian motion $\bar{w}_i \sim \mathcal{N}(0, \varepsilon {I}_{M})$, the directly actuated states $\bar{x}_{i(d)}^{(k)}$ and $\bar{x}_{i(d)}^{(k+1)}$ satisfy Gaussian distribution $\bar{x}_{i(d)}^{(k+1)} \sim \mathcal{N}(\bar{x}_{i(d)}^{(k)} + \bar{f}_{i(d)}(\bar{x}^{(k)}_i, t_k)\varepsilon, \Sigma^{(k)}_i)$ with covariance $\Sigma^{(k)}_i = \varepsilon \bar{B}_{i(d)}(\bar{x}^{(k)}_i)  \bar{\sigma}_i  \bar{\sigma}^\top_i \cdot \bar{B}_{i(d)}(\bar{x}^{(k)}_i)^\top$. When condition $\bar\sigma_{i} \bar\sigma^\top_{i} = \lambda_i \bar{R}_i^{-1}$ in~\hyperref[thm2]{Theorem~2} is fulfilled, the covariance is $\Sigma_i^{(k)} = \varepsilon \lambda_i \bar{B}_{i(d)}(\bar{x}^{(k)}_i)  \cdot \bar{R}_i^{-1} \bar{B}_{i(d)}(\bar{x}^{(k)}_i)^\top = \varepsilon H_i^{(k)}$ with $H_i^{(k)} =  \lambda_i \bar{B}_{i(d)}(\bar{x}^{(k)}_i)  \bar{R}_i^{-1} \cdot \allowbreak \bar{B}_{i(d)}(\bar{x}^{(k)}_i)^\top = \bar{B}_{i(d)}(\bar{x}^{(k)}_i)  \bar{\sigma}_i  \bar{\sigma}^\top  \bar{B}_{i(d)}(\bar{x}^{(k)}_i)^\top$. Hence, the transition probability in~\eqref{EqC2_5} satisfies
\begin{align}\label{EqC3}
	p_i(\bar{x}_{i(d)}^{(k+1)}, & t_{k+1}  | \bar{x}_i^{(k)}, t_k)   =   \frac{1}{[{\det(2\pi \Sigma_i^{(k)}) }]^{1/2}}  \exp \left( -\frac{1}{2} \left\| \bar{x}_{i(d)}^{(k+1)} - \bar{x}_{i(d)}^{(k)} - \bar{f}_{i(d)}(\bar{x}^{(k)}_i, t_k)\varepsilon \right\|^2_{\left(\Sigma_i^{(k)}\right)^{-1}}  \right) \nonumber  \\
	= &  \frac{1}{[{\det(2\pi \Sigma_i^{(k)}) }]^{1/2}} \cdot  \exp\left( - \frac{\varepsilon}{2} \left\|\frac{ \bar{x}_{i(d)}^{(k+1)} - \bar{x}_{i(d)}^{(k)}}{\varepsilon}  - \bar{f}_{i(d)}(\bar{x}^{(k)}_i, t_k) \right\|_{\left(H_i^{(k)}\right)^{-1}}^2\right).
\end{align}
Substituting~\eqref{EqC2}, \eqref{EqC2_5} and~\eqref{EqC3} into~\eqref{EqC1} and in the limit of infinitesimal $\varepsilon$, the desirability function $Z_i(\bar{x}_i, t)$ can then be rewritten as a path integral
\begin{equation}\label{EqC4}
	Z_i(\bar{x}_i, t) = \lim_{\varepsilon \downarrow 0}Z_i^{(\varepsilon)}(\bar{x}^{(0)}_i, t_0).
\end{equation}
Defining a path variable $\bar \ell_i = (\bar{x}^{(1)}_i, \cdots, \bar{x}^{(K)}_i)$, the discretized desirability function in~\eqref{EqC4} can be expressed as
\begin{align}\label{EqC4.5}
	Z^{(\varepsilon)}_i(\bar{x}^{(0)}_i, t_0) & = \int \exp\left( - S_i^{\varepsilon, \lambda_i}(\bar{x}_i^{(0)}, \bar{\ell}_i, t_0) - \frac{1}{2}\sum_{k=0}^{K-1} \log \det(2\pi \Sigma_i^{(k)})  \right) \ d\bar{\ell}_i \\
	& = \int \exp\left( -S_i^{\varepsilon, \lambda_i}(\bar{x}_i^{(0)}, \bar{\ell}_i, t_0) - \frac{1}{2}\sum_{k=0}^{K-1}\log \det (H_i^{(k)}) - \frac{K D |\mathcal{\bar N}_i|}{2} \log (2\pi \varepsilon) \right) \ d\bar{\ell}_i \nonumber\\
	& = \int \exp\left( -\tilde{S}_i^{\varepsilon, \lambda_i}(\bar{x}_i^{(0)}, \bar{\ell}_i, t_0) - \frac{K D |\mathcal{\bar N}_i|}{2}  \log (2\pi \varepsilon) \right) \ d\bar{\ell}_i, \nonumber
\end{align}
where $S_i^{\varepsilon, \lambda_i}(\bar{x}_i^{(0)}, \bar{\ell}_i, t_0)$ is the path value for a trajectory $(\bar{x}_i^{(0)}, \cdots, \bar{x}_i^{(K)})$ starting at space-time pair $(\bar{x}_i, t)$ or $(\bar{x}_i^{(0)}, t_0)$ and takes the form of
\begin{equation*}
	S^{\varepsilon, \lambda_i}_i(\bar{x}^{(0)}_i, \bar{\ell}_i, t_0) = \frac{\phi_i(\bar{x}^{(K)}_i)}{\lambda_i} + \varepsilon \sum_{k=0}^{K-1} \Bigg[ \frac{q_i(\bar{x}^{(k)}_i, t_k)}{\lambda_i}   + \frac{1}{2} \left\|  \frac{ \bar{x}_{i(d)}^{(k+1)} - \bar{x}_{i(d)}^{(k)}}{\varepsilon} -  \bar{f}_{i(d)}(\bar{x}^{(k)}_{i}, t_k) \right\|^2_{\left(H_i^{(k)}\right)^{-1}}  \Bigg];
\end{equation*}
$\tilde{S}_i^{\varepsilon, \lambda_i}(\bar{x}_i^{(0)}, \bar{\ell}_i, t_0)$ is the generalized path value and satisfies $\tilde{S}_i^{\varepsilon, \lambda_i}(\bar{x}_i^{(0)}, \bar{\ell}_i, t_0) = S_i^{\varepsilon, \lambda_i}(\bar{x}_i^{(0)}, \bar{\ell}_i, t_0)  \allowbreak + \frac{1}{2}\sum_{k=0}^{K-1} \allowbreak \log \det(H_i^{(k)})$, and the constant $KD|\bar{\mathcal{N}}_i| / 2 \cdot \log (2\pi \varepsilon)$ in~\eqref{EqC4.5} is related to the numerical stability, which demands a careful choice of $\varepsilon$ and a fine partition  over $[t, t_f)$. Identical to the expectation in~\eqref{Z_Solution} and the integral in~\eqref{EqC1}, the integral in~\eqref{EqC4.5} is subject to the set of all uncontrolled trajectories $\bar{x}_i(\tau)$ initialized at~$(\bar{x}_i, t)$ or $(\bar{x}_i^{(0)}, t_0)$. Summarizing the preceding deviations, \eqref{Prop3E1} and~\eqref{Prop3E2} in~\hyperref[prop3]{Proposition~3} can be restored.

Substituting gradient~\eqref{der2} and discretized desirability function~\eqref{EqC4} into the joint optimal control action~\eqref{OptimalAction}, we have
\begin{align}\label{EqC6}
	\bar{u}^*_{i}(\bar{x}_i, t) & = \lambda_i \bar{R}_i^{-1} \bar{B}_i(\bar{x}_i)^\top \frac{\nabla_{\bar{x}_i} Z_i(\bar{x}_i,t)}{Z_i(\bar{x}_i,t)} = \bar\sigma_i \bar\sigma_i^\top \bar{B}_i(\bar{x}_i)^\top \frac{\nabla_{\bar{x}_i} Z_i(\bar{x}_i,t)}{Z_i(\bar{x}_i,t)}  \\ 
	&= \lambda_i \bar{R}_i^{-1} \bar{B}_i(\bar{x}_i)^\top \cdot \lim_{\varepsilon \downarrow 0}  \frac{\nabla_{\bar{x}^{(0)}_i}\int \exp[-\tilde{S}_i^{\varepsilon, \lambda_i}(\bar{x}_i^{(0)}, \bar{\ell}_i, t_0)- K D |\mathcal{\bar N}_i| / 2 \cdot \log (2\pi \varepsilon) ] \ d\bar{\ell}_i }{\int \exp[-\tilde{S}_i^{\varepsilon, \lambda_i}(\bar{x}_i^{(0)}, \bar{\ell}_i, t_0)- K D |\mathcal{\bar N}_i| / 2 \cdot \log (2\pi \varepsilon)] \ d\bar{\ell}_i} \allowdisplaybreaks  \nonumber\\ 
	& \neweq{(a)} \lambda_i \bar{R}_i^{-1} \bar{B}_i(\bar{x}_i)^\top \cdot \lim_{\varepsilon \downarrow 0}   \frac{ \exp[- K D |\mathcal{\bar N}_i| / 2 \cdot \log(2\pi \varepsilon)] \cdot \nabla_{\bar{x}^{(0)}_i} \int \exp[ -\tilde{S}_i^{\varepsilon, \lambda_i}(\bar{x}_i^{(0)}, \bar{\ell}_i, t_0) ] \ d\bar{\ell}_i}{ \exp[- K D |\mathcal{\bar N}_i| / 2 \cdot \log (2\pi \varepsilon)] \cdot \int \exp[ -\tilde{S}_i^{\varepsilon, \lambda_i}(\bar{x}_i^{(0)}, \bar{\ell}_i, t_0) ] \ d\bar{\ell}_i}  \allowdisplaybreaks \nonumber\\
	& \neweq{(b)} \lambda_i \bar{R}_i^{-1} \bar{B}_i(\bar{x}_i)^\top \cdot \lim_{\varepsilon \downarrow 0}  \frac{\int  \exp[ -\tilde{S}_i^{\varepsilon, \lambda_i}(\bar{x}_i^{(0)}, \bar{\ell}_i, t_0) ] \cdot \nabla_{\bar{x}^{(0)}_i}[ -\tilde{S}_i^{\varepsilon, \lambda_i}(\bar{x}_i^{(0)}, \bar{\ell}_i, t_0) ] \ d\bar{\ell_i}}{\int \exp[ -\tilde{S}_i^{\varepsilon, \lambda_i}(\bar{x}_i^{(0)}, \bar{\ell}_i, t_0) ] \ d\bar{\ell}_i} \nonumber \\
	& \neweq{(c)} \lambda_i \bar{R}_i^{-1} \bar{B}_i(\bar{x}_i)^\top \cdot \lim_{\varepsilon \downarrow 0} \int \tilde{p}^*_i( \bar{\ell}_i | \bar{x}_i^{(0)}, t_0) \cdot \nabla_{\bar{x}^{(0)}_i}[ -\tilde{S}_i^{\varepsilon, \lambda_i}(\bar{x}_i^{(0)}, \bar{\ell}_i, t_0) ] \ d\bar{\ell}_i \nonumber \\
	& \neweq{(d)} \lambda_i \bar{R}_i^{-1} \bar{B}_{i(d)}(\bar{x}_i)^\top \cdot \lim_{\varepsilon \downarrow 0} \int \tilde{p}^*_i(\bar{\ell}_i | \bar{x}_i^{(0)}, t_0) \cdot \tilde{u}_i(\bar{x}_i^{(0)}, \bar{\ell}_i, t_0) \ d\bar{\ell}_i.  \nonumber
\end{align}
(a) follows from the fact that $\exp[- K D |\mathcal{\bar N}_i| / 2 \cdot \log (2\pi \varepsilon)]$ is independent from the path variables $(\bar{x}_i^{(0)}, \bar{\ell}_i)$; (b) employs the differentiation rule for exponential function and requires that the integrand $\exp[-\tilde{S}_i^{\varepsilon, \lambda_i}(\bar{x}_i^{(0)}, \allowbreak \bar{\ell}_i, t_0)]$ be continuously differentiable in $\varepsilon$ and along the trajectory $(\bar{x}_i^{(0)}, \bar{\ell}_i)$; (c) follows from the optimal path distribution $\tilde{p}^*_i(\bar{\ell}_i | \bar{x}_i^{(0)}, t_0)$ that satisfies
\begin{equation*}
	\tilde{p}^*_i(\bar{\ell}_i | \bar{x}_i^{(0)}, t_0) = \frac{\exp[ -\tilde{S}_i^{\varepsilon, \lambda_i}(\bar{x}_i^{(0)}, \bar{\ell}_i, t_0) ]}{\int \exp[ -\tilde{S}_i^{\varepsilon, \lambda_i}(\bar{x}_i^{(0)}, \bar{\ell}_i, t_0) ] \ d\bar{\ell}_i};
\end{equation*}
(d) follows the partitions $\bar{B}_i(\bar{x}_i) = [0, \bar{B}_{i(d)}(\bar{x}_i)^\top]^\top$ and $- \nabla_{\bar{x}^{(0)}_i}\ \tilde{S}_i^{\varepsilon, \lambda_i}(\bar{x}_i^{(0)}, \bar{\ell}_i, t_0) = [ -\nabla_{\bar{x}^{(0)}_{i(n)}} \allowbreak \tilde{S}_i^{\varepsilon, \lambda_i}(\bar{x}_i^{(0)}, \allowbreak \bar{\ell}_i, t_0)^\top,  - \nabla_{\bar{x}^{(0)}_{i(d)}}  \tilde{S}_i^{\varepsilon, \lambda_i}  (\bar{x}_i^{(0)}, \bar{\ell}_i, t_0)^\top ]^\top$, and the initial control variable $\tilde{u}_i(\bar{x}_i^{(0)}, \bar{\ell}_i, t_0)$ is determined by
\begin{align}\label{EqC7}
	&\nabla_{\bar{x}^{(0)}_{i(d)}}   \tilde{S}_i^{\varepsilon, \lambda_i}(\bar{x}_i^{(0)}, \bar{\ell}_i, t_0)   
	=  \nabla_{\bar{x}^{(0)}_{i(d)}} \bigg[  \frac{\phi_i(\bar{x}^{(K)}_i)}{\lambda_i} + \frac{\varepsilon}{\lambda_i}\sum_{k=0}^{K-1}q_i(\bar{x}^{(k)}_i, t_k) + \frac{1}{2}\sum_{k=0}^{K-1} \log \det(H_i^{(k)})    \\
	& \hspace{191pt}  + \frac{\varepsilon}{2} \sum_{k=0}^{K-1} \left\|  \frac{ \bar{x}_{i(d)}^{(k+1)} - \bar{x}_{i(d)}^{(k)}}{\varepsilon}-  \bar{f}_{i(d)}(\bar{x}^{(k)}_i, t_k) \right\|^2_{\left(H_i^{(k)}\right)^{-1}}\bigg]. \nonumber
\end{align}
In the following, we calculate the gradients in~\eqref{EqC7}. Since the terminal cost $\phi_i(\bar{x}^{(K)}_i)$ usually is a constant, the first gradient in~\eqref{EqC7} is zero, \textit{i.e.} $\nabla_{\bar{x}^{(0)}_{i(d)}} [\phi_i(\bar{x}_i^{(K)}) / \lambda_i] = 0$. When the immediate cost $q_i(\bar{x}^{(0)}_i, t_0)$ is a function of state $\bar{x}^{(0)}_{i(d)}$, the second gradient in~\eqref{EqC7} can be computed as follows
\begin{equation}\label{EqC7_5}
	\nabla_{\bar{x}^{(0)}_{i(d)}} \frac{\varepsilon}{\lambda_i}\sum_{k=0}^{K-1}q_i(\bar{x}^{(k)}_i, t_k) = \frac{\varepsilon}{\lambda_i} \nabla_{\bar{x}^{(0)}_{i(d)}} q_i(\bar{x}^{(0)}_i, t_0);
\end{equation}
when $q_i(\bar{x}^{(0)}_i, t_0)$ is not related to the value of $\bar{x}^{(0)}_i$, \textit{i.e.} a constant or an indicator function, the second gradient is then zero. The third gradient in~\eqref{EqC7} follows
\begin{equation}\label{EqC8}
	\nabla_{\bar{x}^{(0)}_{i(d)}} \frac{1}{2} \sum_{k=0}^{K-1} \log \det(H_i^{(k)})  = \frac{1}{2} \nabla_{\bar{x}^{(0)}_{i(d)}}  \log \det(H_i^{(0)}).
\end{equation}
Letting $\alpha_i^{(k)} = (\bar{x}_{i(d)}^{(k+1)} - \bar{x}_{i(d)}^{(k)}) / \varepsilon - \bar{f}_{i(d)}(\bar{x}_i^{(k)}, t_k)$ and $\beta_i^{(k)} = (H_i^{(k)})^{-1}\alpha_i^{(k)}$, the gradient of the fourth term in~\eqref{EqC7} satisfies
\begin{align}\label{EqC9}
	&\nabla_{\bar{x}_{i(d)}^{(0)}} \frac{\varepsilon}{2} \sum_{k=0}^{K-1} \left\|  \alpha_i^{(k)} \right\|^2_{\left(H_i^{(k)}\right)^{-1}}  = \frac{\varepsilon}{2} \cdot \nabla_{\bar{x}^{(0)}_{i(d)}} (\alpha_i^{(0)})^\top  \beta_i^{(0)}  \nonumber \\ 
	&  = \frac{\varepsilon}{2} \left[  \left( \nabla_{\bar{x}^{(0)}_{i(d)}}  \alpha_i^{(0)} \right) \beta_i^{(0)} + \left( \nabla_{\bar{x}_{i(d)}^{(0)}} \beta_i^{(0)} \right) \alpha_i^{(0)} \right]   \\ 
	&  =  -\frac{1}{2} \beta_i^{(0)} -  \frac{\varepsilon}{2} \left[  \left( \nabla_{\bar{x}^{(0)}_{i(d)}}\bar{f}_{i(d)}(\bar{x}_i^{(0)}, t_0)  \right)  \beta_i^{(0)} - \alpha_i^{(0)}  \nabla_{\bar{x}^{(0)}_{i(d)}} \beta_i^{(0)} \right]   \nonumber \\
	&  =- \left(H_i^{(0)}\right)^{-1} \alpha_i^{(0)}  - \varepsilon   \left[ \nabla_{\bar{x}^{(0)}_{i(d)}} \bar{f}_{i(d)}(\bar{x}_i^{(0)}) \right] \left(  H_i^{(0)}\right)^{-1} \alpha_i^{(0)}  + \frac{\varepsilon}{2} \left( \alpha_i^{(0)} \right)^\top \left[ \nabla_{\bar{x}^{(0)}_{i(d)}}  \left(  H_i^{(0)}   \right)^{-1}  \right]    \alpha_i^{(0)}. \nonumber
\end{align}
Detailed interpretations on the calculation of \eqref{EqC9} can be found in~\cite{Theodorou_JMLR_2010, Theodorou_2011}. Meanwhile, after substituting~\eqref{EqC9} into~\eqref{EqC6}, one can verify that the integrals in~\eqref{EqC6} satisfy
\begin{gather}
	\label{App2_Eq12}
	\int \tilde{p}^*_i(\bar{\ell}_i | \bar{x}_i^{(0)}, t_0) \cdot \varepsilon \left[ \nabla_{\bar{x}^{(0)}_{i(d)}} \bar{f}_{i(d)}(\bar{x}_i^{(0)}) \right] \left(H_i^{(0)}\right)^{-1}  \alpha_i^{(0)} d\ell = 0,\\
	\label{App2_Eq13}
	\int \tilde{p}^*_i(\bar{\ell}_i | \bar{x}_i^{(0)}, t_0) \cdot \frac{\varepsilon}{2} \left( \alpha_i^{(0)} \right)^\top \left[ \nabla_{\bar{x}^{(0)}_{i(d)}}  \left(  H_i^{(0)}   \right)^{-1}  \right]    \alpha_i^{(0)} d\ell = - \frac{1}{2} \nabla_{\bar{x}^{(0)}_{i(d)}} \log \det( H_i^{(0)} ).
\end{gather}
Substituting~(\ref{EqC7}-\ref{App2_Eq13}) into~\eqref{EqC6}, we obtain the initial control variable $\tilde{u}_i(\bar{x}_i^{(0)}, \bar{\ell}_i, t_0)$ as follows
\begin{equation*}
	\tilde{u}_i(\bar{x}_i^{(0)}, \bar{\ell}_i, t_0)  = -\frac{\varepsilon}{\lambda_i}\nabla_{\bar{x}^{(0)}_{i(d)}}q_i(\bar{x}^{(0)}_i, t_0) + \left(H_i^{(0)}\right)^{-1} \left(\frac{\bar{x}_{i(d)}^{(1)} - \bar{x}_{i(d)}^{(0)}}{\varepsilon} - \bar{f}_{i(d)}(\bar{x}_i^{(0)}, t_0)  \right).
\end{equation*}
This completes the proof. \qed

\section*{Appendix D: Relative Entropy Policy Search}\label{appD}
The local REPS algorithm in subsystem $\bar{\mathcal{N}}_i$ alternates between two steps, learning the optimal path distribution and updating the parameterized policy, till the convergence of the algorithm. First, we consider the learning step realized by the optimization problem~\eqref{LearningStep}. Since we want to minimize the relative entropy between the current approximate path distribution $\tilde p_i(\bar{x}_i^{(0)}, \bar{\ell}_i)$ and the optimal path distribution $\tilde{p}_i^*(\bar{x}_i^{(0)}, \bar{\ell}_i)$, the objective function of the learning step follows
\begin{equation*}
	\begin{split}
		&\arg \min_{\tilde p_i} \textrm{KL}( \tilde p_i(\bar{x}_i^{(0)}, \bar{\ell}_i) \ \| \ \tilde{p}_i^*(\bar{x}_i^{(0)}, \bar{\ell}_i) )\\
		\neweq{(a)} & \arg \max_{\tilde p_i} \int \tilde{p}_i(\bar{x}_i^{(0)}, \bar{\ell}_i) \left[ \log \tilde{p}_i^*(\bar{\ell}_i | \bar{x}_i^{(0)}) + \log \mu(\bar{x}_i^{(0)}) - \log \tilde{p}_i(\bar{x}_i^{(0)}, \bar{\ell}_i) \right] \ d\bar{x}_i^{(0)} d\bar{\ell}_i\\
		\neweq{(b)} & \arg \max_{\tilde{p}_i} \int \tilde{p}_i(\bar{x}_i^{(0)}, \bar{\ell}_i) \left[   -\tilde{S}_i^{\varepsilon, \lambda_i}(\bar{x}_i^{(0)}, \bar{\ell}_i, t_0) - \log \tilde{p}_i(\bar{x}_i^{(0)}, \bar{\ell}_i)  \right] \ d\bar{x}_i^{(0)} d\bar{\ell}_i .
	\end{split}
\end{equation*}
(a) transforms the minimization problem to a maximization problem and adopts the identity $\tilde{p}_i^*(\bar{x}^{(0)}_i, \bar{\ell}_i) = \tilde{p}_i^*(\bar{\ell}_i | \bar{x}_i^{(0)}) \cdot \mu_i(\bar{x}^{(0)}_i)$, where time arguments are generally omitted in this appendix for brevity; and (b) employs identity~\eqref{OptPathDist} and omits the terms that are independent from the path variable $\bar{\ell}_i$, since these terms have no influence on the optimization problem. In order to construct the information loss from old distribution and avoid overly greedy policy updates~\cite{Peters_CAI_2010, Gomez_KDD_2014}, we restrict the update rate with constraint
\begin{equation*}
	\int \tilde{p}_i(\bar{x}_i^{(0)}, \bar{\ell}_i) \log\frac{\tilde{p}_i(\bar{x}_i^{(0)}, \bar{\ell}_i)}{\tilde{q}_i(\bar{x}_i^{(0)}, \bar{\ell}_i)} \ d\bar{x}_i^{(0)} d\bar{\ell}_i   \leq \delta,
\end{equation*}
where $\delta > 0$ can be used as a trade-off between exploration and exploitation, and LHS is the relative entropy between the current approximate path distribution $\tilde{p}_i(\bar{x}_i^{(0)}, \bar{\ell}_i)$ and the old approximate path distribution $\tilde{q}_i(\bar{x}_i^{(0)}, \bar{\ell}_i)$. Meanwhile, the marginal distribution $\tilde{p}_i(\bar{x}_i^{(0)}) = \int \tilde{p}_i(\bar{x}_i^{(0)}, \bar{\ell}_i) \ d\bar{\ell}_i$ needs to match the initial distribution $\mu_i(\bar{x}^{(0)}_i)$, which is known to the designer. However, this condition could generate an infinite number of constraints in optimization problem~\eqref{LearningStep} and is too restrictive for practice~\cite{Kupcsik_CAI_2013, Gomez_KDD_2014, Sutton_2018}. Hence, we relax this condition by only considering to match the state feature averages of initial state
\begin{equation*}
	\int \tilde{p}_i(\bar{x}_i^{(0)}, \bar{\ell}_i) \cdot \psi_i(\bar{x}_i^{(0)}) \ d\bar{x}_i^{(0)} d\bar{\ell}_i = \int  \mu_i(\bar{x}_i^{(0)}) \cdot \psi_i(\bar{x}_i^{(0)}) \ d\bar{x}_i^{(0)} = \hat{\psi}^{(0)}_i,
\end{equation*}
where $\psi_i(\bar{x}_i^{(0)})$ is a feature vector of initial state, and $\hat{\psi}^{(0)}_i$ is the expectation of the state feature vector subject to initial distribution $\mu_i(\bar{x}^{(0)}_i)$. In general, $\psi(\bar{x}_i^{(0)})$ can be a vector made up with linear and quadratic terms of initial states, $\bar{x}_{i(m)}^{(0)}$ and $\bar{x}_{i(m)}^{(0)}\bar{x}_{i(n)}^{(0)}$, such that the mean and the covariance of marginal distribution $\tilde{p}_i(\bar{x}_i^{(0)})$ match those of initial distribution $\mu_i(\bar{x}^{(0)}_i)$. Lastly, we consider the following normalization constraint 
\begin{equation}\label{NormConstraint}
	\int \tilde{p}_i(\bar{x}_i^{(0)}, \bar{\ell}_i) \  d\bar{x}_i^{(0)} d\bar{\ell}_i = 1,
\end{equation}
which ensures that $\tilde{p}_i(\bar{x}_i^{(0)}, \bar{\ell}_i)$ defines a probability distribution.

The optimization problem~\eqref{LearningStep} can be solved analytically by the method of Lagrange multipliers. Defining the Lagrange multipliers $\kappa > 0$, $\eta \in \mathbb{R}$ and vector $\theta$, the Lagrangian is
\begin{align}\label{Lagrangian}
	\mathcal{L} = \eta + \kappa \delta +  \theta^\top \hat{\psi}_i^{(0)} +  \int \tilde{p}_i(\bar{x}_i^{(0)}, \bar{\ell}_i)\Bigg[  & -\tilde{S}_i^{\varepsilon, \lambda_i}(\bar{x}_i^{(0)}, \bar{\ell}_i)  - \log \tilde{p}_i(\bar{x}_i^{(0)}, \bar{\ell}_i) -\eta  \\
	&  -    \theta^\top \psi_{i}(\bar{x}_i^{(0)})  - \kappa \log\frac{\tilde{p}_i(\bar{x}_i^{(0)}, \bar{\ell}_i)}{\tilde{q}_i(\bar{x}_i^{(0)}, \bar{\ell}_i)} \Bigg] \ d\bar{x}_i^{(0)} d\bar{\ell}_i. \nonumber
\end{align}
We can maximize the Lagrangian $\mathcal{L}$ and derive the maximizer $\tilde{p}_i(\bar{x}_i^{(0)}, \bar{\ell}_i)$ in~\eqref{LearningStep} by letting $\partial \mathcal{L} / \partial \tilde{p}_i(\bar{x}_i^{(0)}, \bar{\ell}_i) = 0$. This condition will hold for arbitrary initial distributions $\mu_i(\bar{x}_i^{(0)})$ if and only if the derivative of the integrand in~\eqref{Lagrangian} is identically equal to zero, \textit{i.e.}
\begin{equation*}
	-\tilde{S}_i^{\varepsilon, \lambda_i}(\bar{x}_i^{(0)}, \bar{\ell}_i, t_0) - (1+\kappa)\left[1 + \log \tilde{p}_i(\bar{x}_i^{(0)}, \bar{\ell}_i)\right] -\eta - \theta^\top  \psi_{i}(\bar{x}_i^{(0)}) + \kappa\log \tilde{q}_i(\bar{x}_i^{(0)} \bar{\ell}_i) = 0,
\end{equation*}
from which we can find the maximizer $\tilde{p}_i(\bar{x}_i^{(0)}, \bar{\ell}_i)$ as shown in~\eqref{distribution}. In order to evaluate $\tilde{p}_i(\bar{x}_i^{(0)}, \bar{\ell}_i)$ in~\eqref{distribution}, we then determine the values of dual variables $\kappa, \eta$ and $\theta$ by solving the dual problem. Substituting~\eqref{distribution} into the normalization constraint~\eqref{NormConstraint}, we have the identity
\begin{equation}\label{NormCond}
	\exp\left(\frac{1+\kappa + \eta}{1 + \kappa} \right) = \int \tilde{q}_i(\bar{x}_i^{(0)}, \bar{\ell}_i)^{\frac{\kappa}{1 + \kappa}}  \cdot \exp\left( - \frac{\tilde{S}_i^{\varepsilon, \lambda_i}(\bar{x}_i^{(0)}, \bar{\ell}_i)  + \theta^\top \psi_i(\bar{x}_i^{(0)})}{1 + \kappa} \right) \ d\bar{x}_i^{(0)} d\bar{\ell}_i,
\end{equation}
which can be used to determine $\eta$ provided the values of $\kappa$ and $\theta$. To figure out the values of $\kappa$ and $\theta$, we solve the dual problem~\eqref{DualProb}, where the objective function $g(\kappa, \theta)$ in~\eqref{DualFun} is obtained by substituting~\eqref{distribution} and~\eqref{NormCond} into~\eqref{Lagrangian}
\begin{equation*}
	\begin{split}
		g(\kappa,& \theta)  \neweq{(a)} \eta + \kappa \delta +  \theta^\top \hat{\psi}_i^{(0)} + 1 + \kappa\\
		& = \kappa \delta + \theta^\top \hat{\psi}_i^{(0)} +  (1+\kappa)  \frac{1  +\kappa + \eta}{1 + \kappa}\\
		& \neweq{(b)} \kappa \delta + \theta^\top \hat{\psi}_i^{(0)} +  (1+\kappa)  \log \int \tilde{q}_i(\bar{x}_i^{(0)}, \bar{\ell}_i)^{\frac{\kappa}{1 + \kappa}}  \cdot \exp\left( - \frac{\tilde{S}_i^{\varepsilon, \lambda_i}(\bar{x}_i^{(0)}, \bar{\ell}_i)  + \theta^\top \psi_i(\bar{x}_i^{(0)})}{1 + \kappa} \right) \ d\bar{x}_i^{(0)} d\bar{\ell}_i.
	\end{split}
\end{equation*}
(a) substitutes~\eqref{distribution} into~\eqref{Lagrangian}, and (b) substitutes~\eqref{NormCond} into the result. By applying the Monte Carlo method and using the data set $\mathcal{Y}_i = \{ (\bar{x}_i^{(0)}, \bar{\ell}^{[y]}_i)  \}_{y = 1, \cdots, Y}$, the objective function~\eqref{DualFun} can be approximated from sample trajectories by 
\begin{equation*}
	g(\kappa, \theta) = \kappa \delta +  \theta^\top \hat\psi_i^{(0)} +  (1 + \kappa) \log \Bigg[ \frac{1}{Y} \sum_{y=1}^{Y} \tilde{q}_i(\bar{x}_i^{(0)}, \bar{\ell}^{[y]}_i)^{\frac{\kappa}{1 + \kappa}} 
	\exp\left( - \frac{\tilde{S}_i^{\varepsilon, \lambda_i}(\bar{x}_i^{(0)}, \bar{\ell}^{[y]}_i)  + \psi_{i}^\top(\bar{x}_i^{(0)}) \cdot \theta}{1+\kappa} \right)   \Bigg],
\end{equation*}
where $\tilde{S}_i^{\varepsilon, \lambda_i}(\bar{x}_i^{(0)}, \bar{\ell}^{[y]}_i)$ and $\psi_{i}(\bar{x}_i^{(0)})$ are respectively the generalized path value and state feature vector of sample trajectory $(\bar{x}_i^{(0)}, \bar{\ell}^{[y]}_i)$, and the old path probability $\tilde{q}_i(\bar{x}_i^{(0)}, \bar{\ell}^{[y]}_i)$ can be evaluated with the current or initial policy by
\begin{equation}\label{OldDist}
	\tilde{q}_i(\bar{x}_i^{(0)}, \bar{\ell}^{[y]}_i) = \mu_i(\bar{x}_i^{(0)}) \cdot \prod_{k = 0}^{K-1} \int  p_i(\bar{x}_i^{(k+1)}, t_{k+1} | \bar{x}_i^{(k)}, \bar{u}_i^{(k)}, t_k) \cdot \pi^{(k)}_i(\bar{u}_i^{(k)} | \bar{x}_i^{(k)}) \ d\bar{u}_i^{(k)},
\end{equation}
where the state variables $\bar{x}_i^{(0)}$, $\bar{x}_i^{(k)}$ and $\bar{x}_i^{(k+1)}$ in~\eqref{OldDist} are from sample trajectory $(\bar{x}_i^{(0)}, \bar{\ell}_i^{[y]})$; the control policy $\pi^{(k)}_i(\bar{u}_i^{(k)} | \bar{x}_i^{(k)})$ is given either as an initialization or an optimization result from updating step~\eqref{UpdatingStep}, and the controlled transition probability $p_i(\bar{x}_i^{(k+1)}, t_{k+1} | \bar{x}_i^{(k)}, \bar{u}_i^{(k)}, t_k)$ with $\bar{x}_i^{(k+1)} \sim \mathcal{N}(\bar{x}_i^{(k)} + \bar{f}_i(\bar{x}^{(k)}_i, t_k)\varepsilon + \bar{B}_i(\bar{x}_i^{(k)}) \bar{u}_i(\bar{x}_i^{(k)}, t_k) \varepsilon, \Sigma^{(k)}_i)$ can be obtained by following the similar steps when deriving the uncontrolled transition probability in~\eqref{EqC3}. When policy $\pi_i^{(k)}$ is Gaussian,~\eqref{OldDist} can be analytically evaluated.

In the policy updating step, we can find the optimal parameters $\chi_i^{*(k)}$ for Gaussian policy by minimizing the relative entropy between the joint distribution $\tilde{p}_i(\bar{x}_i^{(0)}, \bar{\ell}_i)$ from learning step~\eqref{LearningStep} and joint distribution $\tilde{p}_i^{\pi}(\bar{x}_i^{(0)}, \bar{\ell}_i)$ generated by parametric policy $\pi_i^{(k)}(\bar{u}_i^{(k)} | \bar{x}_i^{(k)}, \chi_i^{(k)} ) \allowbreak \sim \mathcal{N}(\bar{u}_i^{(k)}|\hat{a}_i^{(k)}\bar{x}_i^{(k)} + \hat{b}_i^{(k)}, \hat{\Sigma}_i^{(k)})$. To determine the policy parameters $\chi_i^{(k)}= (\hat{a}_i^{(k)}, \hat{b}_i^{(k)}, \hat\Sigma_i^{(k)})$ at time $t_k$, we need to solve the following optimization problem, which is also a weighted maximum likelihood problem
\begin{align}\label{UpdatingStep}
	\chi_i^{*(k)} & = \arg {\textstyle \min_{\chi_i^{(k)}} } \ \textrm{KL} (\tilde{p}_i(\bar{x}_i^{(0)}, \bar{\ell}_i) \ \| \ \tilde{p}_i^{\pi}(\bar{x}_i^{(0)}, \bar{\ell}_i))   \nonumber \\
	& \neweq{(a)}  \arg {\textstyle \max_{\chi_i^{(k)}} } \  \int   \tilde{p}_i(\bar{x}_i^{(0)}, \bar{\ell}_i) \cdot \log\frac{\tilde{p}_i^{\pi}(\bar{x}_{i}^{(k+1)} | \bar{x}_i^{(k)})}{\tilde{p}_i(\bar{x}_i^{(k+1)} | \bar{x}_i^{(k)})}  \ d\bar{x}_i^{(0)} d\bar{\ell}_i    \allowdisplaybreaks  \nonumber \\
	& \newapprox{(b)}  \arg {\textstyle \max_{\chi_i^{(k)}} } \ \int  \tilde{p}_i(\bar{x}_i^{(0)}, \bar{\ell}_i) \cdot \log {\pi}_i^{(k)}(\bar{u}_i^{*(k)} | \bar{x}_i^{(k)}, \chi_i^{(k)}) \ d\bar{x}_i^{(0)} d\bar{\ell}_i   \allowdisplaybreaks \\
	& \neweq{(c)}  \arg {\textstyle \max_{\chi_i^{(k)}} } \ \sum_{y=1}^{Y} \frac{\tilde{p}_i(\bar{x}_i^{(0)}, \bar{\ell}^{[y]}_i)}{\tilde{q}_i(\bar{x}_i^{(0)}, \bar{\ell}^{[y]}_i)} \cdot \log {\pi}_i^{(k)}(\bar{u}_i^{*(k)} | \bar{x}_i^{(k)}, \chi_i^{(k)})  \allowdisplaybreaks \nonumber \\
	& \neweq{(d)} \arg {\textstyle \max_{\chi_i^{(k)}} } \sum_{y=1}^{Y} d_i^{[y]} \cdot  \log {\pi}_i^{(k)}(\bar{u}_i^{*(k)} | \bar{x}_i^{(k)}, \chi_i^{(k)}) \allowdisplaybreaks \nonumber.
\end{align}
(a) converts the minimization problem to a maximization problem and replaces the joint distributions by the products of step-wise transition distributions; (b) employs the assumption that the distribution of controlled transition equals the product of passive transition distribution and control policy distribution~\cite{Gomez_KDD_2014}, \textit{i.e.} $\tilde{p}_i^{\pi}(\bar{x}_{i}^{(k+1)} | \bar{x}_i^{(k)}) = \tilde{p}_i(\bar{x}_i^{(k+1)} | \bar{x}_i^{(k)}) \cdot {\pi}_i^{(k)}(\bar{u}_i^{*(k)} | \bar{x}_i^{(k)}, \chi_i^{(k)})$, and the control action $\bar{u}_i^{*(k)} = [\bar{B}_{i(d)}(\bar{x}_i^{(k)})^\top  \bar{B}_{i(d)}(\bar{x}_i^{(k)})]^{-1} \allowbreak \bar{B}_{i(d)}(\bar{x}_{i}^{(k)})^\top  [\bar{x}_{i(d)}^{(k+1)} - \bar{x}_{i(d)}^{(k)} - \varepsilon \bar{f}_{i(d)}(\bar{x}_i^{(k)}, t_k) ] / \varepsilon  $ from control affine system dynamics maximizes the likelihood function; (c) approximates the integral by using sample trajectories from $\tilde{q}_i(\bar{x}_i^{(0)}, \bar{\ell}^{[y]}_i)$; (d) substitutes~\eqref{distribution}, and the weight of likelihood function is
\begin{equation*}
	d_i^{[y]} = \tilde{q}_i(\bar{x}_i^{(0)}, \bar{\ell}^{[y]}_i)^{\frac{-1}{1 + \kappa}} \cdot \exp\left( - \frac{\tilde{S}_i^{\varepsilon, \lambda_i}(\bar{x}_i^{(0)}, \bar{\ell}^{[y]}_i)  + \psi_{i}^\top(\bar{x}_i^{(0)}) \cdot \theta}{1+\kappa} \right).
\end{equation*}
Constant terms are omitted in steps (a) to (d).

\section*{Appendix E: Multi-Agent LSOC Algorithms}\label{appE}

\hyperref[alg1]{Algorithm~1} gives the procedures of distributed LSMDP algorithm introduced in \hyperref[Sec3_1]{Section~3.1}.

\begin{breakablealgorithm} \small
	\caption{Distributed LSMDP based on factorial subsystems}
	\label{alg1}
	\begin{algorithmic}[1]
		\renewcommand{\algorithmicrequire}{\textbf{Input}}
		\REQUIRE \parbox[t]{\dimexpr\linewidth-\algorithmicindent}{agent set $\mathcal{D}$, communication network $\mathcal{G}$, initial time $t_0$, exit time $t_f$, initial states $x_i^{t_0}$, exit states $x_i^{t_f}$, joint state-related costs $q_i(\bar{x}_i)$, exit costs $\phi(x_t^{t_f})$, weights on exit costs $w_j^i$, and error bound $\epsilon$.\strut}
		
		\renewcommand{\algorithmicrequire}{\textbf{Initialize}} 
		\REQUIRE factorial subsystems $\mathcal{\bar{N}}_{i \in\mathcal{D}}$, and joint exit costs $\phi_i(\bar{x}_i)$.
		
		\renewcommand{\algorithmicrequire}{\textbf{Planning}:}
		\REQUIRE \

		\FOR{$i \in \mathcal{D} = \{1, \cdots, N\}$} 
		
		\STATEx /*Calculate joint desirability $Z_i(\cdot)$ for subsystem $\mathcal{\bar{N}}_i$*/
		
		\STATE \parbox[t]{\dimexpr\linewidth-\algorithmicindent}{Compute coefficients $\Theta$, $\Omega$ and $Z_{\mathcal{B}}$ in~\eqref{Z_update} for subsystem $\bar{\mathcal{N}}_i$. For distributed planning~\eqref{dist_alge}, partition the coefficients $[I - \Theta, \Omega Z_{\mathcal{B}}]_j$ and calculate the projection matrices $P_j$ for $j\in \mathcal{\bar{N}}_i$. \strut}
		
		\WHILE{ $\| Z_{\mathcal{I}}^{(n+1)} - Z_{\mathcal{I}}^{(n)} \| > \epsilon$}
		
		\STATE \parbox[t]{\dimexpr\linewidth-\algorithmicindent}{Update desirability $Z_\mathcal{I}$ with~\eqref{Z_update}. For distributed planning, exchange local solutions $Z^{(n)}_{\mathcal{I}, j}$ with neighboring agents $k \in \mathcal{N}_j \cap \bar{\mathcal{N}}_i$ and update desirability $Z_\mathcal{I}^{(n+1)}$ with~\eqref{dist_alge}.\strut}
		
		\ENDWHILE
		
		\STATEx /*Calculate control distribution for agent $i$*/
		
		\STATE Compute the joint optimal control distribution $\bar{u}_i^*(\cdot | \bar{x}_i)$ of subsystem $\mathcal{\bar{N}}_i$ by~\eqref{OptimalControl}.
		\STATE Derive the local optimal control distribution $u_i^*(\cdot | \bar{x}_i)$ for agent $i$ by marginalizing $\bar{u}_i^*(\cdot | \bar{x}_i)$.
		\ENDFOR
		
		\renewcommand{\algorithmicrequire}{\textbf{Execution}:}
		\REQUIRE

		\WHILE {$t < t_f$ or $x \notin \mathcal{B}$}
		
		\FOR{$i \in \mathcal{D} = \{1, \cdots, N\}$} 
		
		\STATE Measure joint state $\bar{x}_i(t)$ by collecting state information from neighboring agents $j \in \mathcal{N}_i$.
		
		\STATE Sample control action or posterior state $x'_i$ from $u_i^*(\cdot | \bar{x}_i)$.
		
		\ENDFOR
		
		\ENDWHILE
	\end{algorithmic}     
\end{breakablealgorithm}

\noindent \hyperref[alg2]{Algorithm~2} illustrates the procedures of sampling-based distributed LSOC algorithm introduced in \hyperref[Sec3_2]{Section 3.2}.

\begin{breakablealgorithm} \small
	\caption{Distributed LSOC based on sampling estimator}
	\label{alg2}
	\begin{algorithmic}[1]
		\renewcommand{\algorithmicrequire}{\textbf{Input}}
		\REQUIRE \parbox[t]{\dimexpr\linewidth-\algorithmicindent}{agent set $\mathcal{D}$, communication network $\mathcal{G}$, initial time $t_0$, exit time $t_f$, initial states $x_i^{t_0}$, exit states $x_i^{t_f}$, joint state-related costs $q_i(\bar{x}_i)$, control weight matrices $\bar{R}_i$, exit costs $\phi(x_t^{t_f})$, and weights on exit costs~$w_j^i$\strut}
		
		\renewcommand{\algorithmicrequire}{\textbf{Initialize}} 
		\REQUIRE factorial subsystems $\mathcal{\bar{N}}_{i \in \mathcal{D}}$, and joint exit costs $\phi_i(\bar{x}_i)$.
		
		\renewcommand{\algorithmicrequire}{\textbf{Planning \& Execution}:}
		\REQUIRE \
		
		\WHILE {$t < t_f$ or $x \notin \mathcal{B}$}
		
		
		\FOR{$i \in \mathcal{D} = \{1, \cdots, N\}$} 
		
		\STATE Measure joint state $\bar{x}_i(t)$ by collecting state information from neighboring agents $j \in \mathcal{N}_i$.
		
		\STATE \parbox[t]{\dimexpr\linewidth-\algorithmicindent}{Generate uncontrolled trajectory set $\mathcal{Y}_i$ by sampling or collecting data from neighboring agents. \strut}
		
		\STATE \parbox[t]{\dimexpr\linewidth-\algorithmicindent}{Evaluate generalized path value $\tilde{S}_i^{\varepsilon,\lambda_i}(\bar{x}_i^{(0)}, \bar{\ell}^{[y]}_i, t_0)$ and initial control $\tilde{u}_i(\bar{x}_i^{(0)}, \bar{\ell}^{[y]}_i, t_0)$ of each sample trajectory $(\bar{x}_i^{(0)}, \bar{\ell}_i^{[y]})$ in $\mathcal{Y}_i$ by~\eqref{Prop3E2} and~\eqref{inictrl}. \strut}
		
		\STATE \parbox[t]{\dimexpr\linewidth-\algorithmicindent}{Approximate the optimal path distribution $\tilde{p}^*_i(\bar{\ell}_i^{[y]} | \bar{x}_i^{(0)}, t_0)$  and joint optimal control action $\bar{u}^*_i(\bar{x}_i, t)$ by~\eqref{MC_Estimator}, \eqref{MC_Estimato2} or other sampling techniques. \strut}
		
		\STATE Select and execute local control action $u_i^*(\bar{x}_i, t)$ from joint optimal control action $\bar{u}^*_i(\bar{x}_i, t)$.
		
		\ENDFOR
		
		\ENDWHILE
		
	\end{algorithmic}     
\end{breakablealgorithm}

\noindent \hyperref[alg3]{Algorithm~3} illustrates the procedures of REPS-based distributed LSOC algorithm introduced in \hyperref[Sec3_2]{Section 3.2}.

\begin{breakablealgorithm} \small
	\caption{Distributed LSOC based on REPS}
	\label{alg3}
	\begin{algorithmic}[1]
		\renewcommand{\algorithmicrequire}{\textbf{Input}}
		\REQUIRE \parbox[t]{\dimexpr\linewidth-\algorithmicindent}{agent set $\mathcal{D}$, communication network $\mathcal{G}$, initial time $t_0$, exit time $t_f$, initial states $x_i^{t_0}$, exit states $x_i^{t_f}$, joint state-related costs $q_i(\bar{x}_i)$, control weight matrices $\bar{R}_i$, exit costs $\phi(x_t^{t_f})$, weights on exit costs~$w_j^i$, and initial policy $\pi^{(k)}_i(\bar{u}^{(k)}_i | \bar{x}_i^{(k)}, \chi^{(k)}_i)$. \strut}
		
		\renewcommand{\algorithmicrequire}{\textbf{Initialize}} 
		\REQUIRE factorial subsystems $\mathcal{\bar{N}}_{i \in \mathcal{D}}$, and joint exit costs $\phi_i(\bar{x}_i)$.
		
		\renewcommand{\algorithmicrequire}{\textbf{Planning \& Execution}:}
		\REQUIRE \
		
		\WHILE {$t < t_f$ or $x \notin \mathcal{B}$}
		
		\FOR{$i \in \mathcal{D} = \{1, \cdots, N\}$} 
		
		\STATE Measure joint state $\bar{x}_i(t)$ by collecting state information from neighboring agents $j \in \mathcal{N}_i$.
		
		\REPEAT	
		
		\STATE \parbox[t]{\dimexpr\linewidth-\algorithmicindent}{Generate trajectory set $\mathcal{Y}_i$ by sampling with (initial) policy $\pi^{(k)}_i(u^{(k)}_i | \bar{x}^{(k)}_i, \chi^{(k)}_i)$ or collecting data from neighboring agents $j \in \mathcal{N}_i$. \strut}
		
		\STATE Solve dual variables $\kappa, \theta$ and $\eta$ from dual problem~\eqref{DualProb} and condition~\eqref{NormCond}.
		
		\STATE Compute path distribution $\tilde{q}(\bar{x}_i^{(0)}, \bar{\ell}_i^{[y]})$ of each trajectory in $\mathcal{Y}_i$ by \eqref{OldDist}. 
		
		\STATE Update parameter $\chi^{(k)}_i$ by solving weighted maximum likelihood problem \eqref{UpdatingStep}. 
		
		\UNTIL{convergence of parametric policy $\pi^{(k)}_i(\bar{u}^{(k)}_i | \bar{x}_i^{(k)}, \chi^{(k)}_i)$.}
		
		\STATE  Marginalize joint optimal control policy $\pi^{(0)}_i(\bar{u}^{(0)}_i | \bar{x}_i^{(0)}, \chi^{(0)}_i)$ by~\eqref{Marginalize_Policy}. 
		
		\STATE \parbox[t]{\dimexpr\linewidth-\algorithmicindent}{Sample and execute local control action $u^*_i(\bar{x}_i, t)$ from local optimal control policy $\pi_i^{*(0)}(u^{*(0)}_i|\bar{x}_i^{(0)}, \chi_i^{*(0)})$.\strut}
		
		\ENDFOR
		
		\ENDWHILE
	\end{algorithmic}     
\end{breakablealgorithm}


	\bibliographystyle{IEEEtran}
	\bibliography{myref}

\end{document}